%% file: Main.tex
\documentclass[a4paper]{article}

\usepackage[english]{babel}
\usepackage[utf8]{inputenc}
\usepackage{amsmath}
\usepackage{graphicx}
\usepackage[square,numbers]{natbib}
\usepackage[colorinlistoftodos]{todonotes}
\usepackage[utf8]{inputenc} 
\usepackage[T1]{fontenc}    
\usepackage[colorlinks = true, citecolor = blue]{hyperref}       
\usepackage{url}            
\usepackage{booktabs}       
\usepackage{amsfonts}       
\usepackage{nicefrac}       
\usepackage{microtype}      
\usepackage{lipsum}
\usepackage{fancyhdr}       
\usepackage{graphicx}       
\graphicspath{{media/}}     
\usepackage{fullpage}
\usepackage{multicol}

\usepackage[utf8]{inputenc} 
\usepackage[T1]{fontenc}  

\usepackage{url}  
\usepackage[export]{adjustbox}
\usepackage[normalem]{ulem}

\usepackage{enumitem}

\usepackage{booktabs} 
\usepackage{subcaption} 
\usepackage{amsfonts}       
\usepackage{nicefrac}        
\usepackage{microtype}      
\usepackage{xcolor}  
\usepackage{algpseudocode}
\usepackage{algorithm,algpseudocode}
\usepackage{amsmath,amsthm,amsfonts,amssymb,amscd}
\usepackage{graphicx}
\usepackage{placeins}
\usepackage{float}
\usepackage[normalem]{ulem}
\usepackage{wrapfig}
\newcommand{\indep}{\perp \!\!\! \perp}
\newcommand{\notindep}{\not\!\perp\!\!\!\perp}
\usepackage{amsmath,bm}
\usepackage{enumitem}

\theoremstyle{plain}
\newtheorem{theorem}{Theorem}[section]

\newtheorem{fact}{Fact}
\newtheorem{claim}{Claim}
\newtheorem{proposition}[theorem]{Proposition}
\newtheorem{lemma}[theorem]{Lemma}

\theoremstyle{definition}
\newtheorem{definition}[theorem]{Definition}
\newtheorem{assumption}[theorem]{Assumption}
\theoremstyle{remark}
\newtheorem{remark}[theorem]{Remark}

\newcommand{\mbf}{\mathbf}

\newcommand{\combinedgraph}{G^{\textsc{j}}}

\newcommand{\nontrivialblocks}{\mathcal{B}^{\textsc{nt}}}
\newcommand{\vobs}{V^{\textsc{o}}}

\newcommand{\ancobs}{A_{\textrm{obs}}}
\newcommand{\anchid}{A_{\textrm{hid}}}
\newcommand{\tripletCollection}{\mathfrak{V}}
\newcommand{\tripletCollectionObs}{\mathfrak{V}_{\textrm{obs}}}
\newcommand{\tripletCollectionHid}{\mathfrak{V}_{\textrm{hid}}}

\newcommand{\vtrue}{V}

\newcommand{\gtrue}{G}

\newcommand{\vcombined}{V^{\textsc{j}}}
\newcommand{\ecombined}{E^{\textsc{j}}}
\newcommand{\gcombined}{G^{\textsc{j}}}
\newcommand{\anccombined}{A^{\textsc{j}}}

\newcommand{\mainalgo}{\textsf{NoMAD}}
\newcommand{\glasso}{\textsf{GLASSO}}

\newcommand{\bctree}{\mathcal{T}_{\textsc{bc}}}

 \newcommand{\Palgo}{\mathcal{P}_{\rm algo}}
 
\newcommand{\Aalgo}{A_{\rm algo}}
 \newcommand{\Ealgo}{E_{\rm algo}}
\newcommand{\ASTalgo}{\mathcal{T}_{\rm algo}}

\newcommand{\Dextend}{\mathcal{D}_{\rm ext}}

\usepackage{authblk}

\title{Robust Model Selection of Gaussian Graphical Models}

\date{}

\author[1]{\small Abrar Zahin}
\author[2]{\small Rajasekhar Anguluri \thanks{The work was primarily done while R.A. was at Arizona State University.}}
\author[1]{\small Lalitha Sankar}
\author[1]{\small Oliver Kosut}
\author[1]{\small Gautam Dasarathy}

\affil[1]{\footnotesize Department of Electrical, Computer, \& Energy Engineering, Arizona State University}
\affil[2]{\footnotesize Department of Computer Science \& Electrical Engineering, University of Maryland, Baltimore County}

\date{}

\begin{document}
\maketitle

\begin{abstract}

  In Gaussian graphical model selection, noise-corrupted samples present significant challenges. It is known that even minimal amounts of noise can obscure the underlying structure, leading to fundamental identifiability issues. A recent line of work addressing this ``robust model selection'' problem narrows its focus to tree-structured graphical models. Even within this specific class of models, exact structure recovery is shown to be impossible. However, several algorithms have been developed that are known to provably recover the underlying tree-structure up to an (unavoidable) equivalence class.

 In this paper, we extend these results beyond tree-structured graphs. We first characterize the equivalence class up to which general graphs can be recovered in the presence of noise. Despite the inherent ambiguity (which we prove is unavoidable), the structure that can be recovered reveals local clustering information and  global connectivity patterns in the underlying model. Such information is useful in a range of real-world problems, including power grids, social networks, protein-protein interactions, and neural structures. We then propose an algorithm which provably recovers the underlying graph up to the identified ambiguity.  We further provide finite sample guarantees in the high-dimensional regime for our algorithm and validate our results through numerical simulations. 

\end{abstract}

\input{Intro}

\input{RelatedWork}

\input{Prelim_and_Prob_Statement}

\input{Prelim_for_Algorithm}
\input{Algorithm}
\input{Theory_Population}
\input{Theory_FiniteSample}
\input{Simulation.tex}
\input{Conclusion}

\bibliography{references}

\newpage

\appendix

\input{Appendix-Algo}

\input{Appendix-Theory}

\input{Appendix_smpl_cmplxty}
\input{Appendix-Unidentifiablity}

\end{document}

%% file: Intro.tex
\section{Introduction}
\label{sec:intro}

Probabilistic graphical models have emerged as a powerful and flexible formalism for expressing and leveraging relationships among entities in large interacting systems \citep{lauritzen1996graphical}. They have found  application in a range of areas including signal processing \citep{kim2013single, ott2006neurodynamics, murphy2013loopy}, power systems \citep{Angu2022, deka2020graphical, deka2015one},  (phylo)genomics \citep{zuo2017incorporating, dasarathy2014data, dasarathy2022stochastic}, and neuroscience \citep{bullmore2011brain, vinci2019graph}. Gaussian graphical models are an important subclass of graphical models and are the main focus of this paper; our techniques do apply more broadly, as discussed in Section~\ref{sec:discussion}.

  In several applications, we do not know the underlying graph structure, and the goal is to learn this from data --- a problem dubbed graphical model selection. This is important because the graph structure provides a succinct representation of the complex multivariate distribution and can reveal important relationships among the underlying variables. See, e.g., \cite{drton2017structure, maathuis2018handbook} and references therein for more on this problem. Here, we focus on a relatively new but important task where samples from the underlying distribution are corrupted by independent noise with unknown variances. This occurs in a wide variety of applications where sensor data or experimental measurements suffer from statistical uncertainty or measurement noise. In these situations, we refer to the task of graph structure learning as \emph{robust model selection}.

  This problem was recently considered by \cite{katiyar2019robust} and a line of follow-up work \citep{katiyar2020robust, casanellas2021robust, tandon2021sga} who show that unfortunately the conditional independence structure of the underlying distribution can be completely lost in general under such corruption; see Section~\ref{subsec: robust model selection} for more on this. The authors show that even in the often tractable case of tree-structured graphical models, one can only identify the structure up to an equivalent class. In fact, the assumption that the underlying uncorrupted graphical model has a tree structure is critical to the techniques of this line of work. As \citet{casanellas2021robust} astutely observe, when the random vector associated with the true underlying graph is corrupted with independent but non-identical additive noise, the robust estimation problem reduces to a latent tree structure learning problem. We improve on the algorithmic and theoretical results of this line of work by considering the robust model selection problem for general graphs. Our main contributions are summarized below.

\begin{itemize}
\setlength\itemsep{0em}
\setlength\topsep{0em}
\item We establish a fundamental identifiability result for general graphs in the robust Gaussian graphical model selection problem. This confirms that the identifiability problem is exacerbated if one considers more general graphs. More importantly, this also generalizes the identifiability results from earlier lines of work and identifies an equivalence class up to which one may hope to recover the underlying structure. 
\item We devise a novel algorithm, called \mainalgo{} (for Noisy Model selection based on Ancestor Discovery), that tackles the robust model selection problem for {\em general} graphs extending the results of \citep{katiyar2019robust, katiyar2020robust, casanellas2021robust, tandon2021sga}. Our algorithm is based on a novel ``ancestor discovery'' procedure (see Section~\ref{subsubsec:algo}) that we expect to be of independent interest. It is worth observing that the tree-based algorithms previously proposed fail, often catastrophically, when there are loops in the underlying graph.  
\item We show that \mainalgo{} provably recovers the underlying graph up to a small equivalence class and establish sample complexity results for partial structure recovery in the high-dimensional regime. 
\item We also show the efficacy of our algorithm through experiments on synthetic and realistic network structures. 
\end{itemize}

%% file: RelatedWork.tex
\subsection{Related Work}
\label{sec:rel-work}
Several lines of research have tackled the problem of robust estimation of high-dimensional graphical models under corruption. This includes graphical modeling with missing data, outliers, or bounded noise, see, for instance,  \cite{loh2011high, chen2013robust, wang2014robust, nguyen2022distributionally} and references therein. For the missing data problem, several other algorithms have been proposed for estimating mean values and covariance matrices from the incomplete dataset available to the learner \cite{rja1987statistical,schneider2001analysis,lounici2014high}. \citet{zheng2022graphical} considered a variant of the missing value problem where instead of missing values, the measurements are irregular; that is, different vertex pairs have vastly different sample sizes. \citet{vinci2019graph, chang2023nonparanormal, dasarathy2019gaussian} explored the situations where one is only able to obtain samples from subsets of variables, possibly missing joint observations from several pairs. \citet{sun2012robust} and \citet{yang2015robust} proposed algorithms for handling the outliers. There is another line of work that treats this problem using the error-in-variables lens (see the books and papers \cite{hwang1986multiplicative, carroll1995measurement, iturria1999polynomial, xu2007covariate} and references therein). For the problem of model selection from bounded noisy measurements, see \cite{wang2014robust,ollerer2015robust, loh2018high, chen2015robust}.

 However, these papers do not consider the setting of unknown additive noise and the corresponding implications on the conditional independence structure. Recently, \cite{nikolakakis2019learning} considered recovering forest-structured graphical models assuming that noise distribution across all vertices is identical. In contrast, our setting allows for unknown and non-identical noise. The robust model selection problem, as considered here, had not been adequately addressed even for the tree-structured graphical models until the recent work by \cite{katiyar2019robust, katiyar2020robust} who showed that the structure recovery in the presence of unknown noise is possible only up to an equivalence class. These studies also proposed algorithms to recover the correct equivalence class from noisy samples. Using information-theoretic methods, \cite{tandon2021sga} improved the sample complexity result of \cite{katiyar2020robust, nikolakakis2019learning} and provided a more statistically robust algorithm for partial tree recovery. Finally, \cite{zhang2021robustifying} studied the structure recovery problem under noise when the nodes of the GGM are vector-valued random variables. However, these results are limited to tree-structured graphical models. The results of this paper significantly extend this line of work, and are applicable to general graphs.

%% file: Prelim_and_Prob_Statement.tex
\section{Preliminaries and Problem Statement}
\label{sec:prelim}
\textbf{Graph theory.} Let $G = \left(V, E\right)$ be an undirected graph on vertex set $V$ (with cardinality $p$) and edge set $E \subset {V \choose 2}$. For a vertex $v\in V$, let $N_v\triangleq\{ u \in V: \{u,v\} \in E\}$ be the neighborhood of the vertex $v\in V$ and {\em degree} ${\rm deg}(v)$ be the size of $N_v$. A vertex $v$ is said to be {\em leaf} if ${\rm deg}(v)=1$.  A \emph{subgraph} of $G$ is any graph whose vertices and edges are subsets of those of $G$. For $V' \subseteq V$ the \emph{induced subgraph} $G(V')$ has the vertex set $V'$ and the edge set $E' = \lbrace \{u,v\} \in E  : u,v \in V'\rbrace$. A {\em path} between the vertices $u, v$  is a sequence of distinct vertices $v_1=u,v_2, \ldots, v_k=v$ such that $\{v_{i},v_{i+1}\} \in E$, for $ 1\leq i < k$. We let $\mathcal{P}_{uv}$ denote the set of all paths between $u$ and $v$. If $\mathcal{P}_{uv}$ is not empty, we say $u$ and $v$ are connected. The graph $G$ is connected if every pair of vertices in $G$ is connected. A set $S \subseteq V$ separates two disjoint subsets $A, B \subseteq V$ if any path from $A$ to $B$ contains a vertex in $S$. We denote this separation as $A \indep B \mid S$. The resemblance of this notation to that of the conditional independence of random variables in the graphical model will be made clear later. \\

\noindent\textbf{Gaussian graphical models.} 
Let $\mathbf{X} = \left(X_1, X_2, \ldots, X_p\right) \in \mathbb{R}^p$ be a zero-mean Gaussian random vector with a covariance matrix $\Sigma \in \mathbb{R}^{p\times p}$. Compactly, $\mathbf{X}\sim \mathcal{N}(\mathbf{0}, \Sigma)$, where $\mathbf{0}$ is the $p$-dimensional vector of all zeros. Let $G = ([p], E)$ be a graph on the vertex set $[p] \triangleq \{1,2,\ldots, p\}$ representing the coordinates of $\mathbf{X}$. Let $K \triangleq \Sigma^{-1}$ is called the \emph{precision matrix} of $\mathbf{X}$. The distribution of $\mathbf{X}$ is said to be a \emph{Gaussian graphical model} (or equivalently, Markov) with respect to $G$ if $K_{ij} = 0$ for all $\{i,j\} \notin E$ \footnote{Hence, the sparsity pattern of $K$ is represented by the edge set of $G$}.  In other words, for any $\{i, j\}\notin E$, $X_i$ and $X_j$ are conditionally independent given all the other coordinates of $\mathbf{X}$ (see 
\cite{lauritzen1996graphical} for more details). In the sequel,  we will use a generic set $V$ to denote the vertex set of our graph with the understanding that every element in $V$ is uniquely mapped to a coordinate of the corresponding random vector $\mathbf{X}$. For a vertex $v\in V$, with a slight abuse of notation, we write $X_v$ to denote the corresponding coordinate of $\mathbf{X}$. Similarly, we write $\Sigma_{uv}$ to mean the covariance between $X_u$ and $X_v$.  

\subsection{The Robust Model Selection Problem}
\label{subsec: robust model selection}

In this paper, we consider a variant of the model selection problem, which we refer to as \emph{robust model selection}. Formally, let $\mathbf{X}\sim \mathcal{N}(\mathbf{0},
 \Sigma)$ be a Gaussian graphical model with respect to an unknown graph $\gtrue{}\!=\!(V, E)$. In the robust model selection problem, the goal is to estimate the edge set $E$ (or equivalently the sparsity pattern of $K = \Sigma^{-1}$) when one only has access to noisy samples of $\mathbf{X}$. That is, we suppose we have access to the corrupted version $Y$ of the underlying random vector $X$ such that $\mathbf{Y} = \mathbf{X} + \mathbf{Z}$, where the noise  $\mathbf{Z}\sim \mathcal{N}(\mathbf{0}, D)$ is independent of $\mathbf{X}$ and $D$ is assumed to be diagonal with possibly distinct and even zero entries. In other words, the noise is assumed to be independent and heteroscedastic while potentially allowing for some coordinates of $\mathbf{X}$ to be observed uncorrupted. Observe that $\mathbf{Y}\sim \mathcal{N}(\mathbf{0}, \Sigma^o)$, where $\Sigma^o \triangleq \Sigma + D$. Indeed, $D$ is assumed to be unknown. 
 
 Unfortunately, such corruption can completely obliterate the conditional independence structure of $\mathbf{X}$. For instance, suppose that $D=e_je_j^T$, where $e_j$ is a vector of zeros except in the $j^{th}$ entry where it is one. By the Sherman-Morrison identity \citep[see e.g.,][]{horn2012matrix}, we have $(\Sigma^o)^{-1}=K-cK e_je_j^\top K$ for some $c\geq 0$. The term $K e_je_j^\top K $ can be dense in general and, hence, can fully distort the sparsity of $K$ (the conditional independence structure of $\mbf{X}$). 

  In view of the above example, the robust model selection problem  appears intractable. As outlined in Section~\ref{sec:intro}, recent studies show that this problem is partially tractable for tree-structured graphs. 
Notably, as \cite{casanellas2021robust} astutely observes, one can reduce the problem of robust structure estimation of trees to the problem of learning the structure of latent tree graphical models, which enjoy several efficient algorithms and rich theoretical results (see, e.g., ~\cite{choi2011learning, erdos1999few,semple2003phylogenetics,dasarathy2014data}). To see this, suppose that $\mathbf{X}$ is Markov according to a graph $G = (V = [p], E)$ that is tree-structured. Let the \emph{joint graph} $\gcombined{}$ denote the graph obtained by creating a copy of each node in $G$ and linking the copies to their counterparts in $G$. Formally, we define $\gcombined = (\vcombined{}, \ecombined{)}$, where  $\vcombined{}\triangleq V \cup \{1^e, 2^e, \ldots, p^e\}$ and $\ecombined{}\triangleq E \cup \{\{i, i^e\}: i\in [p]\}$. In subsequent sections, for any vertex subset $B \subseteq 2^V$, we let $B^e$ denote the vertices associated with the noisy observations from $B$ and call it the  \emph{noisy counterpart} of $B$.

  Notice that with this definition, if we associate the coordinates of $\mathbf{Y}$ to the newly added leaf vertices, the concatenated random vector $[\mathbf{X}; \mathbf{Y}]$, obtained by stacking $\mathbf{X}$ on top of $\mathbf{Y}$, is Markov according to $\gcombined{}$. \cite{casanellas2021robust} then uses the fact that when given samples of $\mathbf{Y}$, one can reconstruct a reduced latent tree representation of $\gcombined{}$, which in turn can be used to infer an equivalence class of trees that contains the true tree $G$. Indeed, the equivalence class thus obtained is the same one identified by \cite{katiyar2019robust,katiyar2020robust}. We next state our general identifiability result after introducing some more graph-theoretic concepts.

\subsection{An Identifiability Result}
\label{sec:identifiability}

   A connected graph $G$ is said to be {\em biconnected} if at least 2 vertices need to be removed to disconnect the graph. A subgraph $H$ of $G$ is said to be a {\em biconnected component} if it is a maximal biconnected subgraph of $G$. That is, $H$ is not a strict subgraph of any other biconnected subgraph of $G$. The {vertex set} of such a biconnected component will be referred to as a \emph{block}. A block is  \emph{non-trivial} if it has more than two vertices. For example, in Fig.~\ref {fig:true_graph (zero_corrupted)}, the vertex set $\{1, 2, 3, 4\}$ and $B_1 \cup \{6,8\}$ are non-trivial blocks where $B_1$ is an arbitrary set of vertices such that the subgraph on $B_1 \cup \{6,8\}$ is a biconnected component, whereas, the set $\{10,8\}$ is a trivial block. In what follows, we will often be interested in the vertices of such non-trivial blocks and toward this we write $\nontrivialblocks{}$ to denote the set of all vertex sets of non-trivial blocks in $G$. From these definitions, it follows that trees (which are cycle free) do not have any non-trivial blocks. It also follows that two blocks can share at most one vertex; we refer to such shared vertices as  \emph{cut} vertices. In Fig.~\ref {fig:true_graph (zero_corrupted)}, the vertices $4 \mbox{ and } 10$ are cut vertices. The vertices in $\nontrivialblocks{}$ which are not cut vertices are referred to as \emph{non-cut} vertices. In Fig.~\ref {fig:true_graph (zero_corrupted)}, the vertex $1$ is a non-cut vertex.

  With these definitions, we now introduce a novel representation for a graph $G$ that will be crucial to stating our results. This representation is a tree-structured graph $\bctree{(G)}$, which we call the {\em articulated set tree} of $G$, whose vertices correspond to (a) non-trivial blocks in $G$, and (b) vertices in $G$ that are not a member of any non-trivial blocks.  Vertices in this tree-structured representation are connected by edges if the corresponding sets in $G$ either share a vertex or are connected by a single edge. The vertices in the original graph that are responsible for the edges in the  representation are called articulation points\footnote{Articulation (points) vertices are cut vertices that separate non-trivial blocks from the rest of the graph.}. We formally define this below. For example, for $G$ in Fig.~\ref {fig:true_graph (zero_corrupted)}, the articulated set tree is illustrated in Fig.~\ref{fig:corresponding_set_tree}, where the sets $\{1,2,3,4\}$ and $\{17,18,19\}$ are associated to vertices in the articulated set tree representation, and $4,6,7,14,$ and $17$ are examples of articulation points. 

    \begin{definition}[Articulated Set Tree]
      \label{def:art-set-graph}
      For an undirected graph $G = (V,E)$, the {\bf articulated set tree}  $\bctree{(G)}$ is a tuple $(\mathcal{P}, \mathcal{E}, \mathcal{A})$ where (a) the set $\mathcal{P} = \left\{ B: B\in \nontrivialblocks{} \right\} \cup \left\{\{v\}: v\in V \setminus \cup_{B\in \mathcal{B}^{\textsc{nt}}}B \right\}$; (b) an edge $\{P, P'\}\in \mathcal{E}$ if and only if (i) vertices $v,v' \in V$ are such that $v \in P, v' \in P'$, and $\{v,v'\} \in E$ or (ii) there exists a vertex $v \in V$ such that $v \in P \cap P'$, and (c) the articulation function $\mathcal{A}: \mathcal{E}\to V\times V$ returns the articulation points of each edge. 
\end{definition}

  Notice that the articulated set tree (AST), as the name suggests, is indeed a tree. Otherwise, by definition, the set of non-trivial blocks $\nontrivialblocks{}$ would be incorrect (we show this formally in Lemma~\ref{lemma:AST_is_a_Tree} in the Appendix). Readers familiar with graph theory may have observed that the AST representation is quite similar to the block-cut tree representation~\citep[see e.g.,][]{harary2018graph, biggs1986graph}, but unlike a block-cut tree, the subgraph associated with any non-trivial block does not matter in the articulated set tree. \\

  We will now define the equivalence class of graphs up to which robust recovery is possible. Let $L(G)$ denote the set of all leaves in $G$ (i.e., all vertices of degree one). A subset $R\subset L(G)$ is said to be \emph{remote} if no two elements of $R$ share a common neighbor. Let $\mathcal{R}$ be the set of all remote subsets of $L(G)$. For each $R\in \mathcal{R}$, define a graph $G_R$ on $V$ by exchanging each vertex in $R$ with its (unique) neighbor.

    \begin{figure*}
    \begin{subfigure}[b]{0.325\textwidth}
  \centering
    \includegraphics[width=3.8cm, height=2.1 cm]{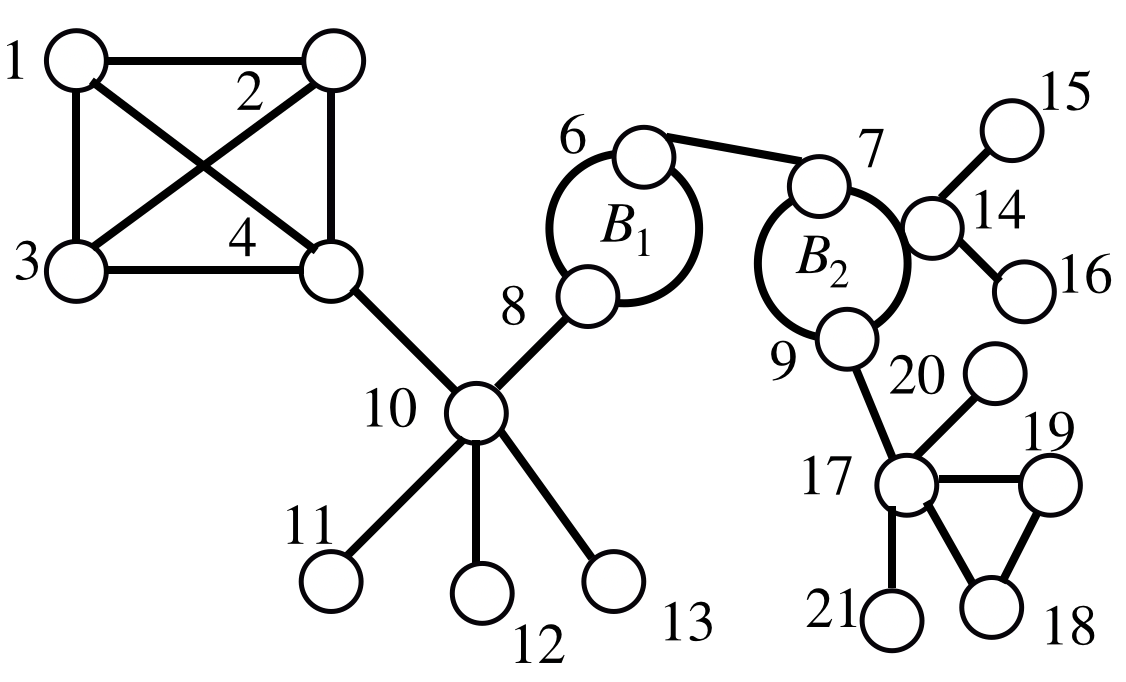}
  \caption{$\gtrue{}$}
    \label{fig:true_graph (zero_corrupted)}
  \end{subfigure}
  \begin{subfigure}[b]{0.325\textwidth}
  \centering
    \includegraphics[width=3.9cm, height=2.5 cm]{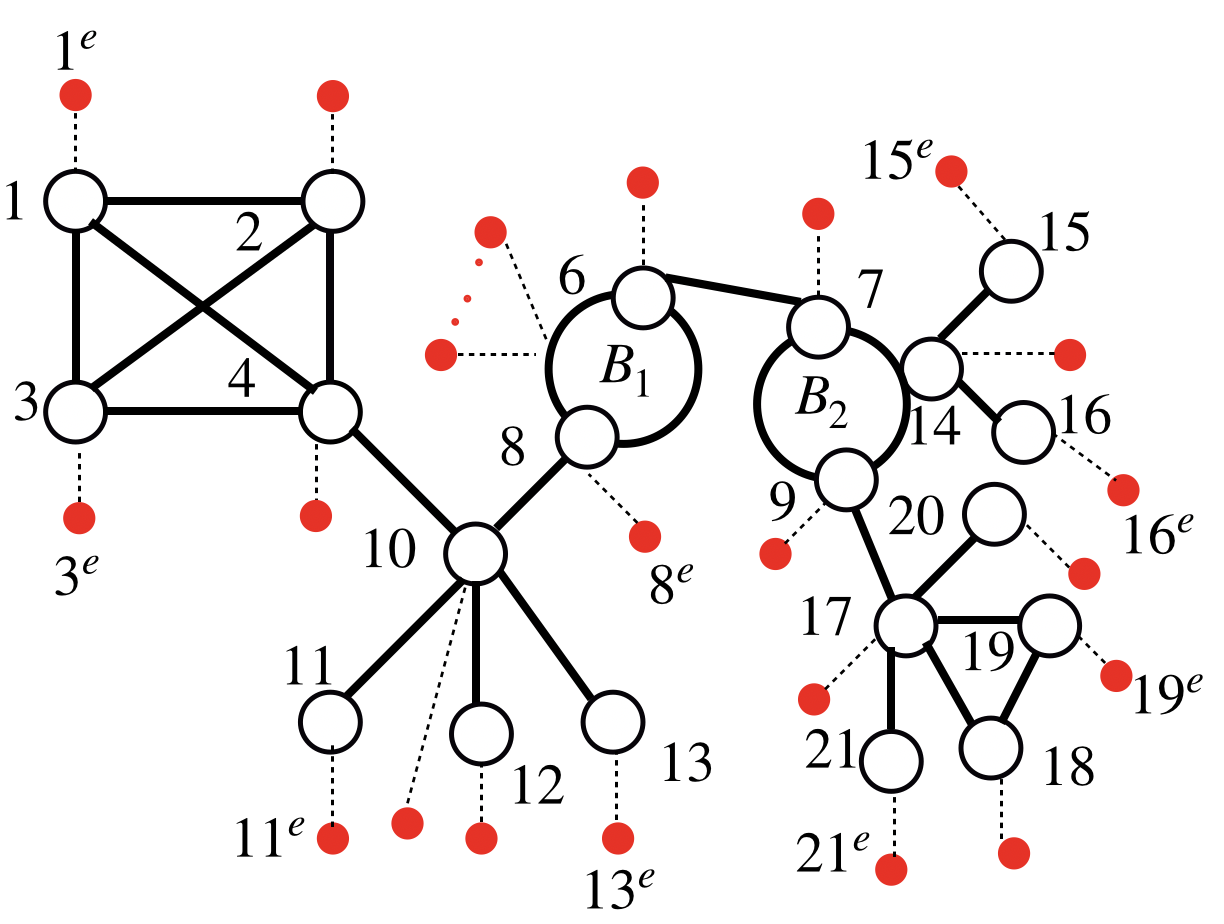}
\caption{$\gcombined{}$}
    \label{fig:parttially_crrptd}
  \end{subfigure}
  \begin{subfigure}[b]{0.325\textwidth}
  \centering
    \includegraphics[width=3.4 cm, height=2.5 cm]{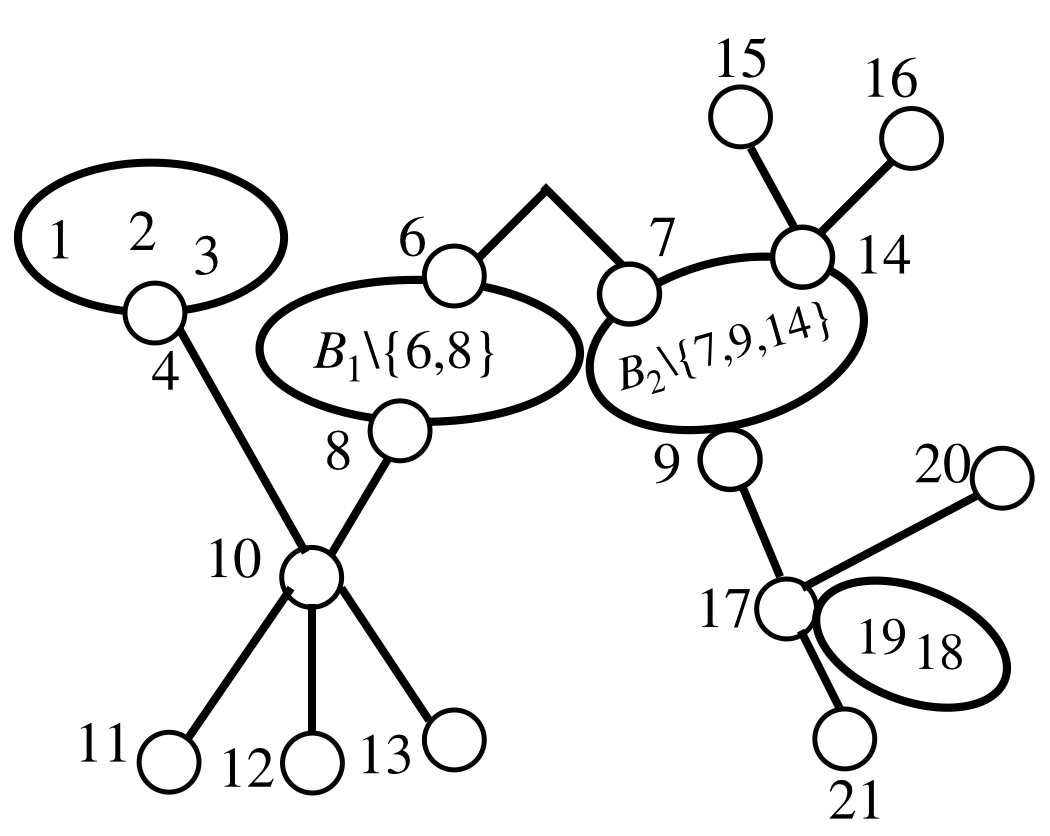}
  \caption{$\bctree{(\gtrue{})}$}
    \label{fig:corresponding_set_tree}
  \end{subfigure}

  \caption{ (a) a true underlying graph where both $B_1 \cup \{6,8\}$ ($B_2 \cup \{7,9\}$) are non-trivial blocks where $B_1$ ($B_2$) is an arbitrary set of vertices such that the subgraph on $B_1 \cup \{6,8\}$ ($B_2 \cup \{7,9\}$) is a biconnected component, (b) joint graph $\gcombined{}$; noisy vertices associated with the non-trivial blocks containing $B_1$ and $B_2$, and some other vertices are not numbered to reduce the clutter, and (c) the articulated set tree $\bctree{(\gtrue{})}$.}
    \label{fig:intro_figure_Robust_Problem}
\end{figure*}

\begin{definition}[Equivalence Relation, $\sim$]
\label{def:equiv_rel}
Two graphs $G, H$ are said to be equivalent if and only if $\exists R\in \mathcal{R}\mbox{ such that } \bctree{(G_R)} = \bctree{(H)}$. Symbolically, we write as $H \sim G$. 
\end{definition}

   We let $[G]$ denote the equivalence class of $G$ with respect to $\sim$. It is not hard to verify that Definition \ref{def:equiv_rel} is a valid equivalence relation. Furthermore, it can be readily checked that this notion of equivalence subsumes the ones defined for trees in \cite{katiyar2019robust, casanellas2021robust}. Fig.~\ref{fig:three_graphs_from_same_EC} illustrates three graphs from the same equivalence class. Notice that $G_1$ can be constructed from $G$ by: (i) exchanging the labels between  the leaf vertices $\{13, 17, 15\}$ with their corresponding neighbors $\{10, 21, 14\}$;  (ii) adding an edge $\{2,4\}$ inside a non-trivial block. Similarly, $G_2$ can be constructed from $G$ by: (i) exchanging the labels between the leaf vertices $\{11, 20\}$ with their corresponding neighbors $\{10,21\}$; (ii) removing an edge $\{1,4\}$ from a non-trivial block\footnote{Notice that the sets $\{13, 17, 15\}$ and $\{11, 20\}$ are remote according to Definition~\ref{def:equiv_rel}.}. Therefore, for any two graphs in the equivalence class, the non-cut vertices of any non-trivial block remain unchanged, whereas, the edges in the non-trivial block can be arbitrarily changed; the labels of the leaves can be swapped with their neighbor. In the following we will show that our identifiability result also complements the equivalence class. Finally, notice that any graph whose AST is equal to $\bctree{(G)}$ is in $[G]$. Hence, with a slight abuse of notation, we also let $\bctree{(G)}$ to denote $[G]$.

 \begin{figure*}
  \begin{subfigure}[b]{0.325\textwidth}
  \centering
    \includegraphics[width=4cm, height=2.2 cm]{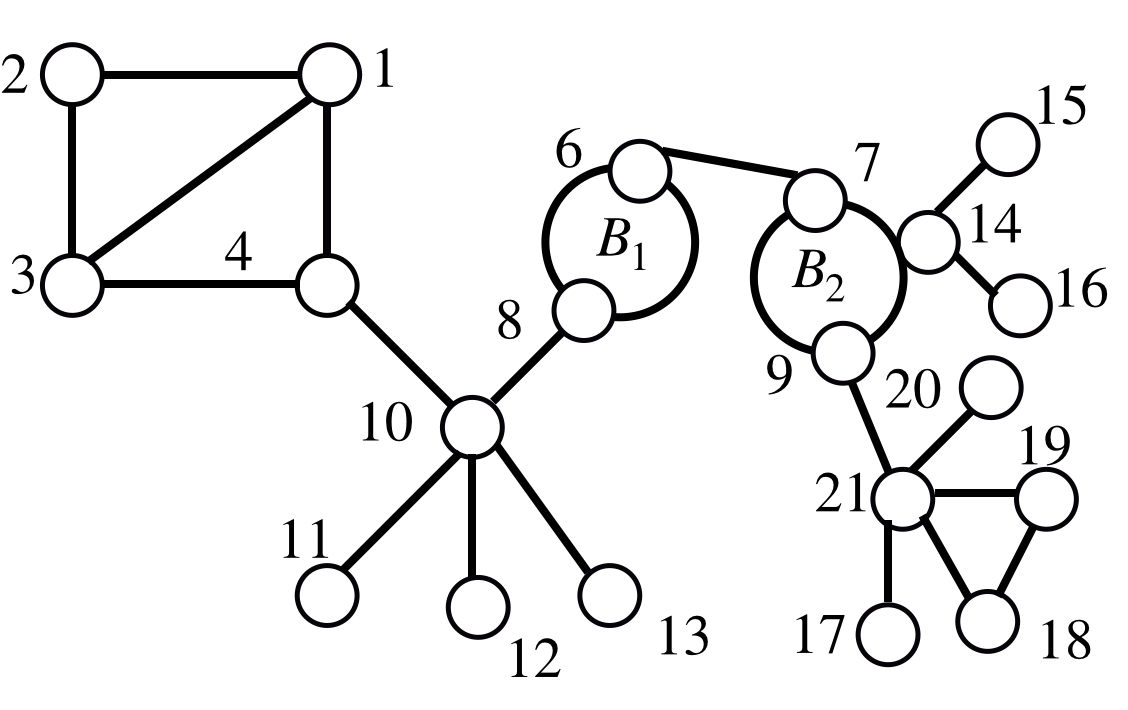}
  \caption{$G$}
    \label{fig:G_EC_1}
  \end{subfigure}
  \begin{subfigure}[b]{0.325\textwidth}
  \centering
    \includegraphics[width=4cm, height=2.2 cm]{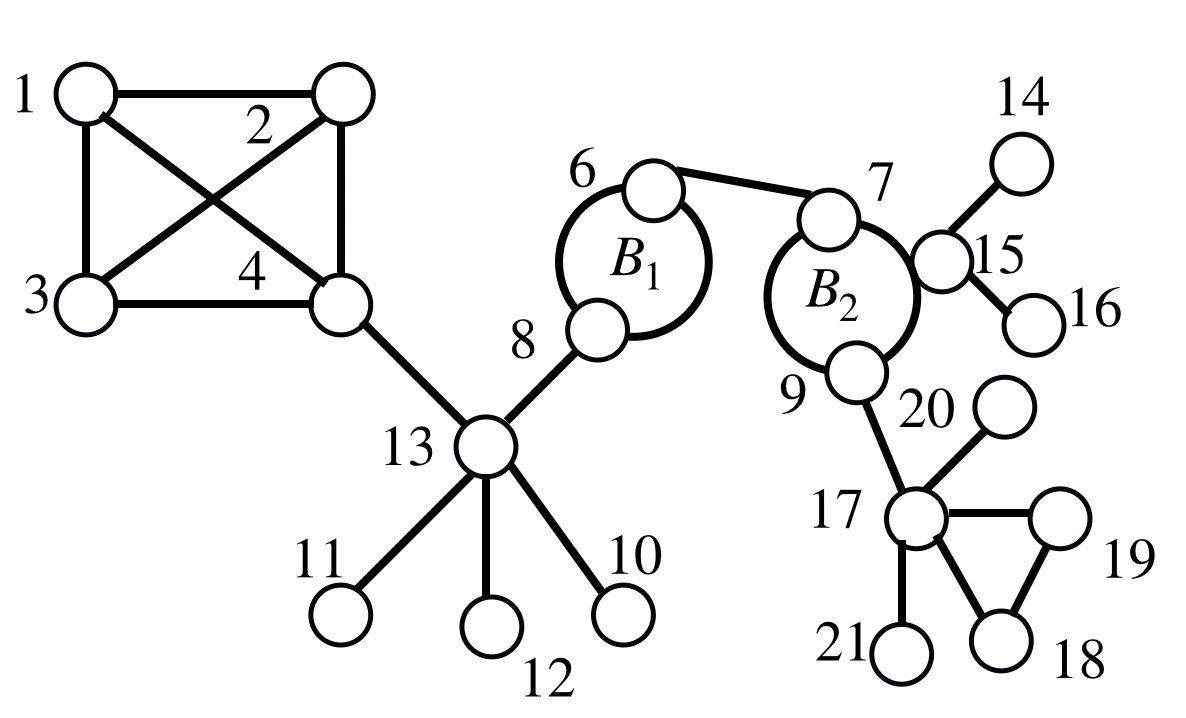}
  \caption{$G_1$}
    \label{fig:G_EC_2}
  \end{subfigure}
      \begin{subfigure}[b]{0.325\textwidth}
  \centering 
    \includegraphics[width=4cm, height=2.2 cm]{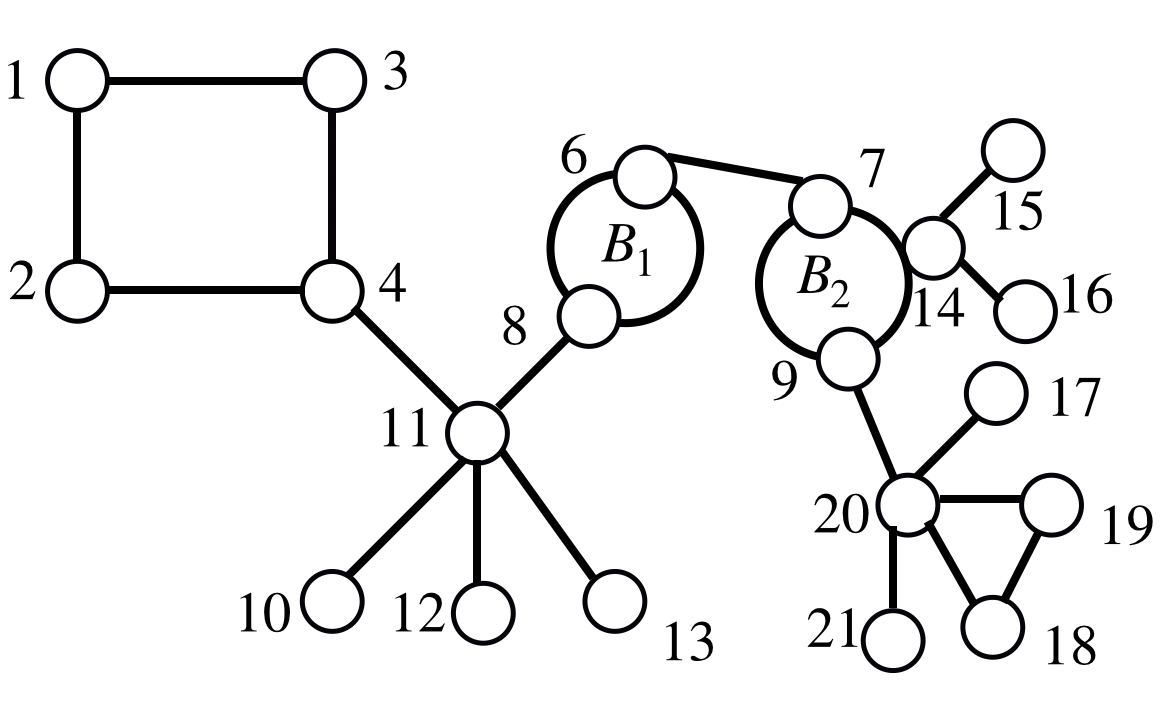}
  \caption{$G_2$}
    \label{fig:G_EC_3}
  \end{subfigure}

  \caption{ An illustration of three graphs from the same equivalence class of $G$ in Fig.~\ref{fig:true_graph (zero_corrupted)}. For all three graphs, $B_1 \cup \{6,8\}$ ($B_2 \cup \{7,9\}$) can have any induced subgraph as long as subgraphs on $B_1 \cup \{6,8\}$ ($B_2 \cup \{7,9\}$) is a biconnected component.}
    \label{fig:three_graphs_from_same_EC}
\end{figure*}

  We now state our identifiability result. This result establishes that there exists at least one (and most likely several) graphs whose true covariance matrices, under the noise model of Subsection~\ref{subsec: robust model selection}, will result in the same observed covariance matrix $\Sigma^o$.

\setcounter{theorem}{0}

\begin{theorem}[Identifiability] \label{thm: identifiability}
Fix a covariance matrix $\Sigma^*$ whose conditional independence structure is given by the graph $G$. Fix $D$, where $D=\textrm{diag} (D_{11},\ldots,D_{pp})\geq 0$. Let $\Sigma^{o}=\Sigma^*+D$. Then, there exists at least one $H\in [G]$ such that $\Sigma^{o}$ can be written as $\Sigma^{q}+D^{q}$, where $D^{q}$ is a diagonal matrix, and the sparsity pattern of $(\Sigma^{q})^{-1}$ is described by $H$.

\end{theorem} 

  This theorem is proved in Section~\ref{sec:identifiblty--proof}. Notice that this indentifiability result shows it is impossible to uniquely recover $G$ as there are other confounding graphs whose noisy observations would be indistinguishable from those of $G$. However, this theorem does not rule out the possibility of recovering the equivalence class to which $G$ belongs since all the confounding examples are confined to $[G]$. In Section~\ref{sub-sec:algo}, we devise an algorithm that precisely does this. Before we conclude this section, we introduce the notion of partial structure recovery and discuss how recovering the equivalence class still reveals useful information about the true graph. \\

 {\bf Partial Structure Recovery of $G$.}
Given the noise model described in Subsection~\ref{subsec: robust model selection} we know that any graph is only identifiable upto the equivalence relation in Definition~\ref{def:equiv_rel}. This is not only the best one can do (by Theorem~\ref{thm: identifiability}), but also preserves useful \textit{partial} structure of the graph. In particular, such a partial structure recovery is able to identify the (non-cut) constituents of the non-trivial blocks and the set of leaf vertices (and  neighbors thereof).  As the following examples illustrate, we conclude this subsection by arguing that \textbf{even such partially recovered graphs} are instrumental in several application domains.

\begin{enumerate}
    \item \textbf{Electrical distribution networks} usually have radial (or globally tree-like) network structures. Recently, several parts of such networks have become increasingly locally interconnected due to the adoption of technologies like roof-top solar panels and battery power storage that enable more flexible power flow. Nonetheless, practitioners critically rely on (learning) the global structure for most operations and maintenance tasks, such as state estimation, power flow, and cybersecurity.

    \item \textbf{Neuronal networks.} Network models are commonly used to describe structural and functional connectivity in the brain \cite{sporns2018graph}. An important property of networks that model the brain is their modular structure; modules or communities  correspond to clusters of nodes that are densely connected (and are presumably functionally related). Such modular structure has long been regarded as a hallmark of many complex systems; see \cite{herbert1962architecture} for more details. Module detection, that is, understanding which nodes belong to which modules can yield important insights into how networks function and also uncover a network's latent community structure.
\end{enumerate}

%% file: Prelim_for_Algorithm.tex
\subsection{The Robust Model Selection Algorithm}
\label{sub-sec:algo}

We now present an algorithm that can recover the partial structure of a graph $G$ for which only noisy samples are available based on the setup from Subsection~\ref{subsec: robust model selection}. Before we describe our algorithm, we introduce a few more concepts that will play a key role. We start with a well-known fact about the factorization of pairwise correlations for a faithful\footnote{The global Markov property for GGMs ensures that graph separation implies conditional independence. The reverse implication need not to hold. However, for a faithful GGM, the reverse implication does hold.} Gaussian graphical model.
\begin{fact}[see e.g., \cite{soh2014testing}]
\label{fact:sepration and pairwise correlation factorization (faithful_graph)}
For a faithful Gaussian graphical model, $X_i \indep X_k \vert X_j$ if and only if $\rho_{ik}=\rho_{ij}\times \rho_{jk}$.\vspace{-2mm}
\end{fact}

We now define information distances.

\setcounter{theorem}{2}
\begin{definition}[Information distances]
\label{def:information_distances}
 For $(X_1,\ldots,X_p)\sim \mathcal{N}(\mathbf{0},\Sigma)$, the information distance between $X_i$ and $X_j$ is defined by $d_{ij}\triangleq-\log \vert \rho_{ij} \vert \geq 0,$ where $\rho_{ij}$ is the \textit{pairwise correlation coefficient}  between $X_i$ and $X_j$; that is
    $\rho_{ij} = \Sigma_{ij} / \sqrt{\Sigma_{ii} \Sigma_{jj}}.$
\end{definition}

 For tree-structured Gaussian graphical models, the strength of the correlation dictates the information (or graphical) distance between vertices $i$ and $j$. Higher the strength, the smaller is the distance, and vice versa. In fact, the information distance defined this way is an additive metric on the vertices of the tree. Although the graphs we consider are not necessarily trees, we still refer to this quantity as a {distance} throughout the paper for convenience. We now define the notion of ancestors for a triplet of vertices using the notion of minimal mutual separator.

\begin{definition}[Minimal mutual separators, Star triplets, Ancestors]
\label{def:minimal-mut-sep}
Fix a triplet of vertices $U \in {V \choose 3}$. A vertex set $S \subseteq V$ is called a mutual separator of $U$ if $S$ separates each pair $i, j \in U$; that is, every path $\pi \in \mathcal{P}_{ij}$ contains at least one element of $S$. The set $S$ is called a {\em minimal mutual separator} of the triplet $U$ if no proper subset of $S$ is a mutual separator of $U$. We let $S_{\rm min}(U)$ denote the set of all minimal mutual separators of $U$. $U$ is said to be a {\em star triplet} if $\left| S_{\rm min}(U) \right| = 1$ and the separator in $S_{\rm min}(U)$ is a singleton. The unique vertex that mutually separates $U$ is called the {\em ancestor} of $U$. 
\end{definition}

 For instance, for the graph in Fig. \ref{fig:graph_with_non_unique_mms}, the set $\{2,4,7,8,3,5\}$ is a mutual separator for the triple $\{1,6,9\}$. Further notice that minimal mutual separator set may not be unique for a triplet: here, the sets $\{2,3,7\}$ and $\{4,5,8\}$ are minimal mutual separators of the triple $\{1,6,9\}$. In Fig.~\ref{fig:true_graph (zero_corrupted)}, for the triplet $\{1,11,12\}$, minimal mutual separator $S_{\rm min}\left(\{1,11,12\}\right)= \{10\}$. Further, notice that an ancestor of $U$ can be one of the elements of $U$. In Fig.~\ref{fig:true_graph (zero_corrupted)}, the vertex $\{10\}$ is the ancestor of triplet $\{4,10,11\} \triangleq U$. Notice that for triplet $U$, the distance between $10$ and the ancestor is zero. \\

 \begin{wrapfigure}{r}{0.18\textwidth}
\includegraphics[width=1\linewidth]{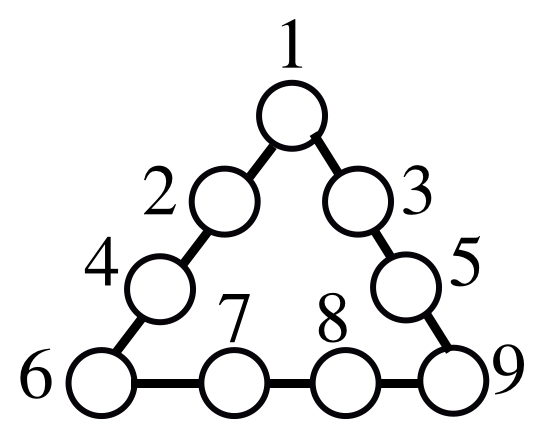} 
\caption{Graph with multiple minimal mutual separators.}
\label{fig:graph_with_non_unique_mms}
\end{wrapfigure}

 For a star triplet  $\{i,j,k\}$ with ancestor $r$,  it is clear that the following holds true based on the relationship between graph separation and conditional independence: $X_i \indep X_j \mid X_r$, $X_i \indep X_k \mid X_r$, and $X_j \indep X_k \mid X_r$. As a consequence, from Fact \ref{fact:sepration and pairwise correlation factorization (faithful_graph)} and Definition \ref{def:information_distances}, the pairwise distances $d_{ij}, d_{ik}, \mbox{ and } d_{jk}$ satisfy the following equations:  $d_{ij}=d_{ir}+d_{rj}$, $d_{ik}=d_{ir}+d_{rk}$, and $d_{jk}=d_{jr}+d_{rk}$. Some straightforward algebra results in the following identities that allows us to compute the distance between each vertex in $\{i,j,k\}$ and the ancestor vertex $r$. In particular, for any ordering $\{x,y,z\}$ of the set $\{i,j,k\}$ notice that the following is true:

\begin{align}
    \label{eq:ancestor_distance_triplet}
    d_{xr} = 0.5\times(d_{xy} + d_{xz} - d_{yz})
\end{align}

 For a triplet $U=\{i,j,k\}$, we will let $d_{i}^{U} \triangleq \frac{1}{2}(d_{ij}+d_{ik}-d_{jk})$. If $U$ is a star triplet, then $d_i^U$ reveals the distance between $i$ and the ancestor of $U$. However, we do not restrict this definition to star triplets alone. When $U$ is not a star triplet, $d_i^U$ is some arbitrary (operationally non-significant) number; in fact, for non-star triplets this quantity may even be negative. Notice that all vertex triplets in a tree are star triplets. Hence, for a tree-structured graphical model, we can choose any arbitrary triplet  and if we can find a vertex $r$ for which $d_{i}^U = d_{ir}, i\in U$, then we can identify the ancestor of $U$. If such a vertex does not exist, we can deduce the existence of a latent ancestor. Therefore, iterating through all possible triplets, one can recover the true (latent) tree structure underlying the observed variables. In fact, several algorithms in the literature use similar techniques to learn trees (see e.g., \cite{saitou1987neighbor, krishnamurthy2012robust, choi2011learning, dasarathy2014data}.   

%% file: Algorithm.tex
\subsection{The \mainalgo{} Algorithm}
\label{subsubsec:algo}

  \begin{figure*}
     \centering
     \begin{subfigure}[b]{0.23\textwidth}
         \centering
         \includegraphics[width=\textwidth]{graph_with_noisy_observation.png}
         \caption{}
         \label{fig:par-Corrupt-JointGraph}
     \end{subfigure}
     \hfill
     \begin{subfigure}[b]{0.23\textwidth}
         \centering
         \includegraphics[width=\textwidth]{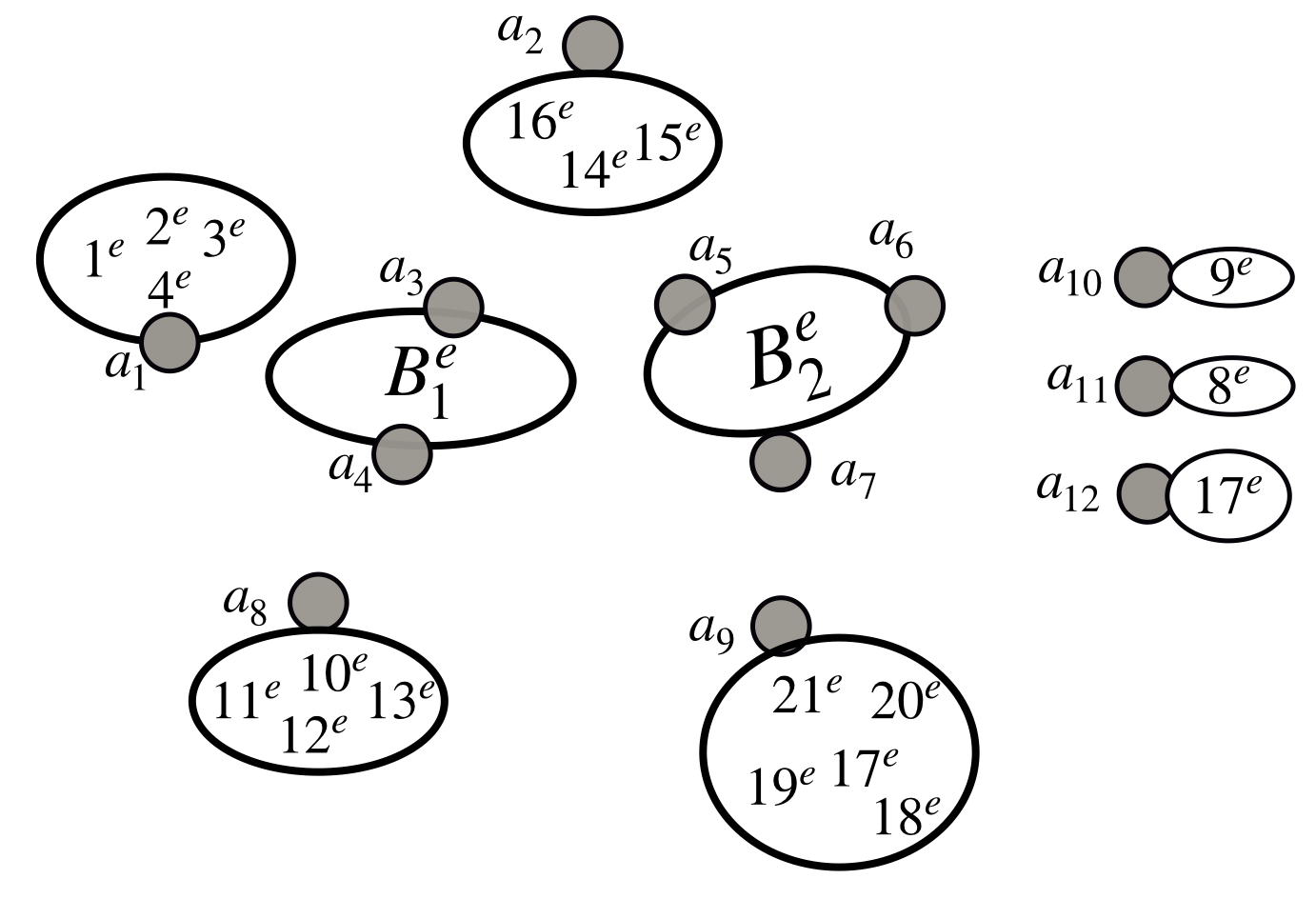}
         \caption{}
         \label{fig:clusters}
     \end{subfigure}
     \hfill
     \begin{subfigure}[b]{0.23\textwidth}
         \centering
         \includegraphics[width=\textwidth]{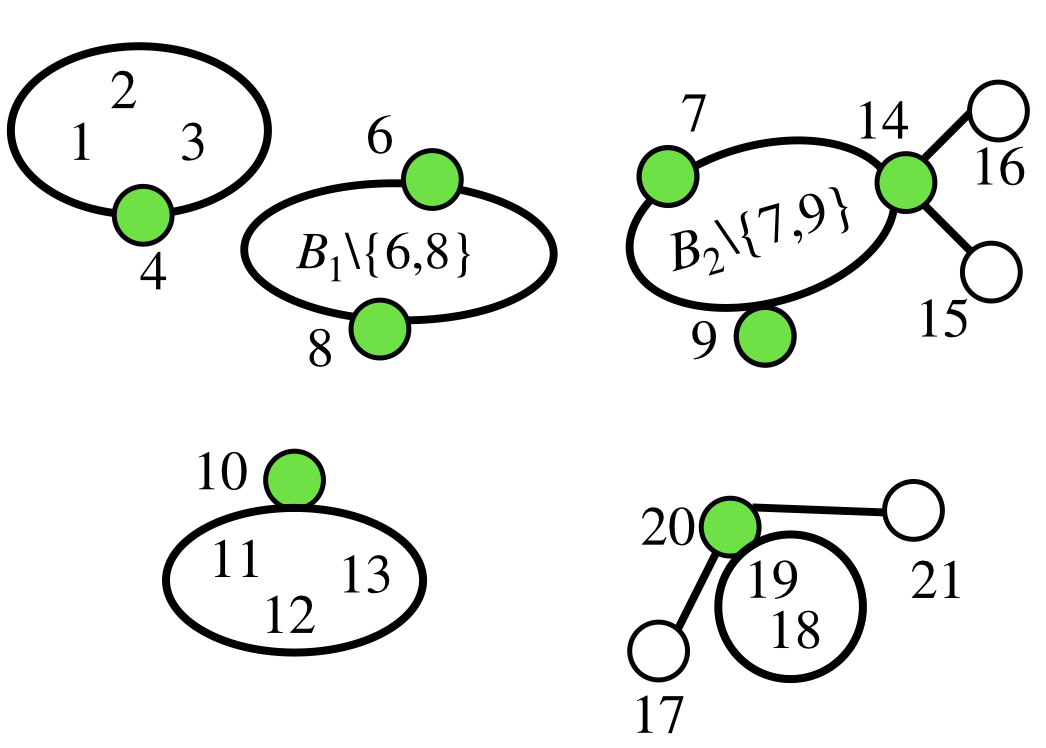}
         \caption{}
         \label{fig:learn-partition-and-artcltn-points}
     \end{subfigure}
          \begin{subfigure}[b]{0.23\textwidth}
         \centering
         \includegraphics[width=\textwidth]{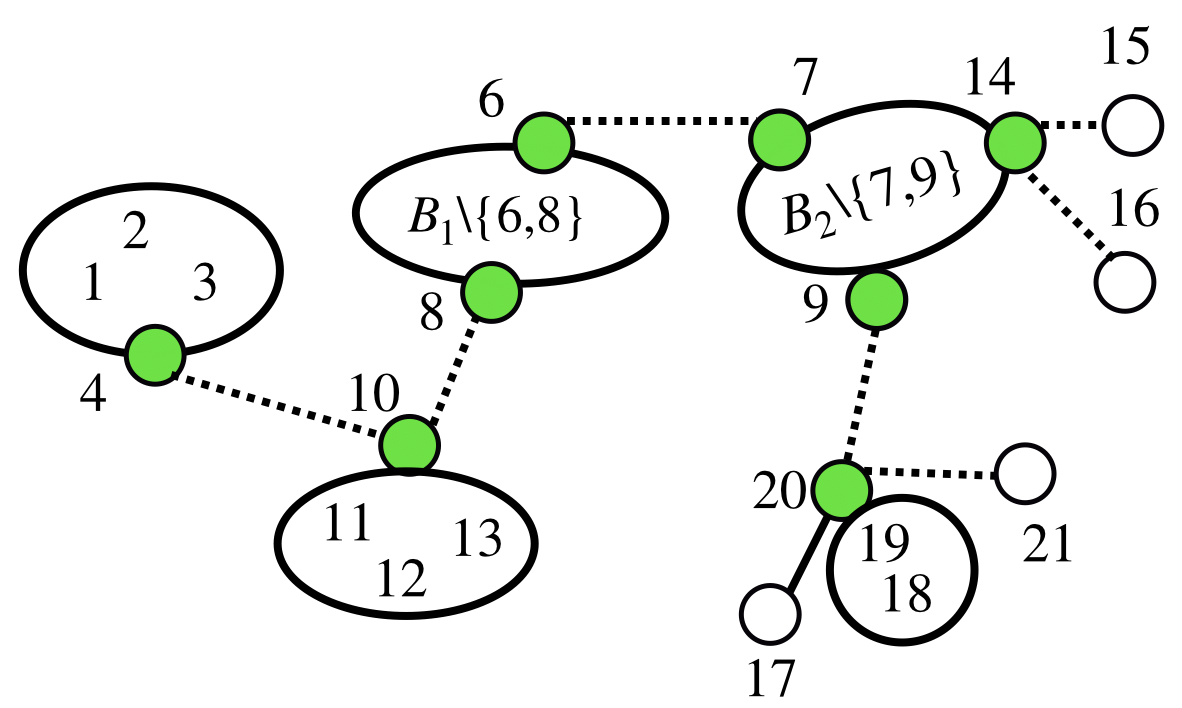}
         \caption{}
         \label{fig:learn-edge-set}
     \end{subfigure}
        \caption{(a) The joint graph $\gcombined{}$, (b) the leaf clusters and internal clusters of $\gcombined{}$; $B_1^e$ ($B_2^e$) denote the set of noisy vertices associated with the vertices in $B_1$ ($B_2$); Also, recall that clusters are a set of vertices; grey vertices are identified but unlabeled articulation points associated with the clusters, (c) non-trivial blocks and trivial blocks along with the identified and labeled articulation points, and (d) the edges between different articulation points.}
        \label{fig:illustrations_of_procedures}
\end{figure*}

In this section, we introduce our algorithm \mainalgo{}, for \underline{No}isy \underline{M}odel selection based on \underline{A}ncestor \underline{D}iscovery, for robust model selection. We describe \mainalgo{}
  in the population setting (i.e., in the infinite sample limit) for clarity of presentation;  the modifications required for the finite sample setting are discussed in Section~\ref{sec:sample-complexity}. In the population setting, \mainalgo{} takes as input the (exact) pairwise information distances $d_{ij}$, for all $i,j$ in the observed vertex set $\vobs{}$, and returns an articulated set tree $\ASTalgo{} \triangleq  (\Palgo{}, \Aalgo{}, \Ealgo{})$ (see Definition~\ref{def:art-set-graph}). A high level overview of the algorithm is given in Algorithm~\ref{algo:gen-frmwrk}. Its operation may be divided into  two main steps: 
(a) {learning $\Palgo{}$ and $\Aalgo{}$}; and (b) {learning $\Ealgo{}$}. These steps are summarized in the following. A formal algorithmic listing and a full description can be found in the Appendix.

 \textbf{(a) Learning $\Palgo{}$ and $\Aalgo{}$.} Inspired by the aforementioned  ancestor based tree reconstruction algorithms,  \mainalgo{} first identifies the ancestors\footnote{We say that a vertex $r$ is an ancestor if there is a triple $U$ such that $r$ is an ancestor of $U$.} in $\gcombined{}$, and learns the pairwise distances between them. Notice that finding the ancestors in $\gcombined{}$ is challenging for the following reasons: (a) since $\gcombined{}$ is not a tree, some vertex triplets are not star triplets
(e.g., $\{1^e,3^e,11^e\}$ in Fig.~\ref{fig:parttially_crrptd}), and (b) a subset of vertices (which may include ancestors) in $\gcombined{}$ are unobserved or latent. Hence, we can not guarantee the identification of a star triplet following the procedure for trees. \mainalgo{} instead uses a novel procedure that compares {\em two triplets} of vertices to identify the ancestors; we call this the \textsc{TIA} (Test Identical Ancestor) test which is defined as follows:

  \begin{algorithm}
	\begin{algorithmic}[1]
  	\caption{\mainalgo{}}
 	\label{algo:gen-frmwrk}
 		\State \textbf{Input:} Pairwise distances $\mathcal{D} = \{d_{ij}\}_{i,j \in \vobs{}}$.
 		\State \textbf{Output:} $\ASTalgo{} \triangleq  (\Palgo{}, \Aalgo{}, \Ealgo{})$. 
   	\State {\color{cyan}\textsc{IdentifyAncestors}}(Subroutine{\color{cyan}~\ref{app:algo--Ident_Anc_Extnd_Dist}}). \textbf{Accepts:} $\mathcal{D}$; \textbf{Returns:} (I) A set $\ancobs{}$ ($\anchid{}$ resp.) of observed (hidden resp.) ancestors, and the corresponding collection of vertex triplets $\tripletCollectionObs{}$ ($\tripletCollectionHid{}$ resp.), and (II) The set of pairwise distances $\{d_{ij}\}$ for each pair $i,j \in \vobs{} \cup \anchid{}$

  \State {\color{cyan}\textsc{LearnClusters}}  (Subroutine{\color{cyan}~\ref{app:algo--Learn_clusters}}). \textbf{Accepts}:  $\ancobs{}, \anchid{}$, and $\mathcal{D}$; \textbf{Returns:}  A collection of (I) leaf clusters $\mathcal{L}$, and 
    (II)   internal clusters $\mathcal{I}$.

 	\State {\color{cyan} \textsc{VertexSet-AST}}(Subroutine{\color{cyan} ~\ref{app:algo--PaLE}}).\textbf{Accepts:} $\mathcal{L}$ and $\mathcal{I}$; \textbf{Returns:} 
 	    (I) the vertex set $\Palgo{}$ (II) the articulation points $\Aalgo{}$, and (III) a \textit{subset} of the edge set $\Ealgo{}$.

 		 	\State {\color{cyan} \textsc{EdgeSet-AST}}(Subroutine{\color{cyan}~\ref{app:algo--EAST}}). \textbf{Accepts:} $\Palgo{}$ and $\Aalgo{}$; \textbf{Returns:} $\Ealgo{}$.
\end{algorithmic}
 \end{algorithm}

\begin{definition}[\textsc{TIA} Test]
\label{def:TIA-test}
The \underline{T}est \underline{I}dentical \underline{A}ncestor (\textsc{TIA}) test accepts a triplet pair $U,W \in {\vobs{} \choose 3}$, and returns \textsc{True} if and only if for all $x \in U$, there exists at least one pair $y,z \in W$ such that $d_x^U+d_y^W=d_{xy}$ and $d_x^U+d_z^W=d_{xz}$.
\end{definition}
In words, in order for a pair of triplets $U, W$ to share an ancestor in $\gcombined{}$, each vertex in one triplet (say, $U$) needs to be separated from at least a pair in $W$ by the (shared) ancestor in $\gcombined{}$. We now describe the fist step of the \mainalgo{} in three following sub-steps: \\

 \underline{Identifying the ancestors in $\gcombined{}$.} In the first sub-step, \mainalgo{} (a) uses the TIA test to create the set $\mathfrak{V} = \{\mathcal{V}\subset {V\choose 3}: \mbox{each $U\in \mathcal{V}$ has the same ancestor
}\}$, (b) it then assigns each triplet collection $\mathcal{V} \in \mathfrak{V}$ to either $\tripletCollectionObs{}$ or $\tripletCollectionHid{}$; the former is the collection of vertex triples whose ancestor is observed and the latter has ancestors that are hidden, and (c) identifies the observed ancestors and enrolls them into a set of observed ancestors $\ancobs{}$. Furthermore, for each collection $\mathcal{V}_i \in \tripletCollectionHid{}$, \mainalgo{} introduces a hidden vertex, and enrolls it in $\anchid{}$ such that, $\vert \anchid{} \vert$ equals to the number of hidden ancestors in $\gcombined{}$. For example, consider the joint graph $\gcombined{}$ in 
Fig.~\ref{fig:par-Corrupt-JointGraph}. In $\gcombined{}$, the vertex $\{4\}$ is the observed ancestor of the pair $\{1^e,4^e,10\}$ and $\{3^e,4^e,8^e\}$, and the vertex $\{8\}$ is the hidden ancestor of the pair $\{3^e,8^e,9^e\}$, $\{1^e,8^e,7^e\}$. Complete pseudocode for this step appears in Subroutine~\ref{app:algo--Ident_Anc_Extnd_Dist} in Appendix. \\

 \underline{Extending the distance set $\{d_{ij}\}_{i,j \in \vobs{}}$.}  In the next sub-step, using pairwise distances $\{d_{ij}\}_{i,j \in \vobs{}}$, and $\anchid{}$, \mainalgo{} learns the following distances: (a) $d_{ij}$ for all $i,j \in \anchid{}$, and (b) $d_{ij}$ for all $i,j \in \anchid{} \cup \vobs{}$. For learning (a), notice from the last sub-step that each $a_i \in \anchid{}$ is assigned to a collection of triplets $\mathcal{V}_i \in \tripletCollectionHid{}$. In order to compute the distance between two hidden ancestors (say $a_p,a_q \in \anchid{}$), the \mainalgo{} chooses two triplets $V_p\in \mathcal{V}_p$ and $V_q\in \mathcal{V}_q$, and computes the set $\Delta_{pq}$ as follows: $\Delta_{pq}= \left\{d_{xy} - (d_{x}^{V_p} + d_{y}^{V_q}):  x\in V_p, y \in V_q\right\}.$ Then, the most frequent element in $\Delta_{pq}$, i.e., ${\rm mode}(\Delta_{pq})$,  is declared as $d_{pq}$. We show in Appendix that \mainalgo{} not only correctly learns the distance set $\{d_{ij}\}_{i,j \in \anchid{}}$, but also learns $d_{ij}$ for any $i \in \anchid{}$ and $j \in \vobs{}$. A pseudocode for this step appears in Subroutine~\ref{app:algo--Ident_Anc_Extnd_Dist} in Appendix. \\

  \underline{Learning $\Palgo{}$ and $\Aalgo{}$.} 
 In the final sub-step, \mainalgo{} learns the clusters of vertices in $\vobs{}\setminus \ancobs{}$ using the separation test in Fact~\ref{fact:sepration and pairwise correlation factorization (faithful_graph)} which eventually lead to finding $\Palgo{}$ and $\Aalgo{}$. Specifically, \mainalgo{} learns (a) all the leaf clusters, each of which is a set of vertices that are separated from the rest of the graph by a single ancestor, and (b) all the internal clusters, each of which is a set of vertices that are separated from the rest of the graph by multiple ancestors. For example, in Fig.~\ref{fig:clusters}, the set $\{17^e,18^e,19^e,20^e,21^e\}$ is a leaf cluster--- separated from the rest of the graph by the (hidden) ancestor $a_9$. The set $B_2^e$ is an internal cluster--- separated from the rest of the graph by the set of ancestors $\{a_5,a_6,a_7\}$. Next, \mainalgo{} uses the clusters to learn $\Palgo{}$ and $\Aalgo{}$ for $\ASTalgo{}$ by applying the TIA test on each cluster to identify the non-cut vertices and potential cut vertices in it. For example, for the leaf cluster $\{17^e,18^e,19^e,20^e,21^e\}$, $18$ and $19$ are non-cut vertices, and a vertex from $17, 20, \mbox{ and } 21$ may be declared as a cut-vertex arbitrarily (see Fig.~\ref{fig:learn-partition-and-artcltn-points}). A pseudocode for this step appears as Subroutine~\ref{app:algo--Learn_clusters} and Subroutine~\ref{app:algo--PaLE} in Appendix. \\

 \textbf{(b) Learning $\Ealgo{}$.}
 In this step, \mainalgo{} learns the edge set $\Ealgo{}$. Notice from Definition \ref{def:art-set-graph} that any two elements of $\Palgo{}$ are connected with each other through their respective articulation points. Hence, in order to learn $\Ealgo{}$, \mainalgo{} needs to learn the neighborhood of each articulation point in $\gtrue{}$. To this end,  \mainalgo{} first learns this neighborhood for each articulation point in $\gtrue{}$ using Fact~\ref{fact:sepration and pairwise correlation factorization (faithful_graph)}. Then, in the next step,  \mainalgo{} creates an edge between two elements of $\Palgo{}$ if the articulation points from each element are neighbors in $\gtrue{}$ (dotted lines in Fig.~\ref{fig:learn-edge-set}). A pseudocode for this step appears as Subroutine~\ref{app:algo--EAST} in Appendix.

%% file: Theory_Population.tex
 \newcommand{\noncut}{\textrm{non-cut}} 
 \section{Performance Analysis of \mainalgo{} in the Population Setting} 
 \label{sec:theory}
  In this section, we show the correctness of \mainalgo{} in returning the equivalence class of a graph $G$ while having access only to the noisy samples according to the problem setup in Section~\ref{subsec: robust model selection}. We now make an assumption that will be crucial to show the correctness of \mainalgo{}. This is similar to the faithfulness assumption common in the graphical modeling literature  \cite{choi2011learning,kalisch2007estimating, uhler2013geometry}, and like the latter, it rules out ``spurious cancellations''. To that end, let $\mathcal{V}_{\textrm{star}} \subseteq {\vtrue \choose 3}$ be the set of all star triplets in $\gtrue{}$ (see Definition~\ref{def:minimal-mut-sep}). Let $\mathcal{V}_{\textrm{sep}} \subseteq {\vtrue{} \choose 3}$ be the set of triplets $V$ such that one of the vertices in $V$ separates the other two vertices. 
\setcounter{theorem}{0}
\begin{assumption}[Ancestor faithfulness]
\label{assump:anc_faithfulness}
Let $U,W \in {\vtrue \choose 3} \setminus \mathcal{V}_{\rm star} \cup \mathcal{V}_{\rm sep}$. Then,
(i) there are no vertices $x \in U$ and $a \in W$ that satisfy $d_x^U + d_a^W =d_{xa}$, and (ii) there does not exist any vertex $r \in \vtrue$ and $x \in U$ for which the distance $d_{xr}$ satisfies relation in \eqref{eq:ancestor_distance_triplet}. 
\end{assumption}
Notice that Assumption~\ref{assump:anc_faithfulness} is only violated when there are explicit constraints on the corresponding covariance values which, like the faithfulness assumption, only happens on a set of measure zero. We next state the main result of our paper in the population settings.

\setcounter{theorem}{1}

  \begin{theorem}
  \label{thm--main_thm_pop}
Consider a covariance matrix $\Sigma^*$ whose conditional independence structure is given by the graph $G$, and the model satisfies Assumption~\ref{assump:anc_faithfulness}. Suppose that according to the problem setup in Section~\ref{subsec: robust model selection}, we are given pairwise distance $d_{ij}$ of a vertex pair $(i,j)$ in the observed vertex set $\vobs{}$, that is, $d_{ij} \triangleq - {\rm log} |\rho_{ij}|$ where $\rho_{ij} \triangleq \Sigma^o_{ij} / \sqrt{\Sigma^o_{ii} \Sigma^o_{jj}}$. Then, given the pairwise distance set $\{d_{ij}\}_{i,j \in \vobs{}}$ as inputs, \mainalgo{} outputs the equivalence class $[\gtrue{}]$.
 \end{theorem}

 {\bf Proof Outline.} In order to show that $\mainalgo{}$ correctly learns the equivalence class, it suffices to show that it can correctly deduce the articulated set tree $\ASTalgo{}$. Given this, and the equivalence relation from Definition~\ref{def:equiv_rel}, the entire equivalence class can be readily generated. Our strategy will be to show that \mainalgo{} learns $\ASTalgo{}$ correctly by showing that it learns (a) the vertex set $\Palgo{}$, (b) the articulation points $\Aalgo{}$, and (c) the edge set $\Ealgo{}$ correctly. We will now establish (a) and (b).  

\begin{itemize}
    \item From the description of the algorithm in Section~\ref{sub-sec:algo}, it is clear that \mainalgo{} succeeds in finding the ancestors, which is the first step, provided the TIA tests succeed. Indeed, in the first stage of our proof, we establish Lemma~\ref{lemma:two_dstar_triplets_with_identical_ancestor} in Appendix which shows that the \textsc{TIA} test passes with two triplets if and only if they share a common ancestor in $\gcombined{}$. 
    \item Next, we show in Lemma~\ref{lemma:learning_pairwise_dist_betn_landmarks} and Claim~\ref{claim:collect_all_vertex_land_dist} in Appendix that \mainalgo{} correctly learns the distances  $d_{ij}$ for all vertices $i,j$ that either in the set of observed vertices $\vobs{}$ or in the set of hidden ancestors $\anchid{}$. Proposition~\ref{prop:correctness of Pale} establishes the correctness of \mainalgo{} in learning the leaf clusters and internal clusters, and in learning $\Palgo{}$ and $\Aalgo{}$. The proof correctness of this step crucially depends on identifying the non-cut vertices in $\gtrue{}$ from different clusters which is proved in Lemma~\ref{lemma:cut_test}.
\end{itemize}
Finally, we outline the correctness of \mainalgo{} in learning the edge set $\Ealgo{}$. Recall that \mainalgo{} learns the neighbor articulation points of each articulation point. Proposition~\ref{prop:correctness of East} in Appendix shows that \mainalgo{} correctly achieves this task.  Using the neighbors of different articulation points, \mainalgo{} correctly learns the edges between different elements in $\Palgo{}$.  

%% file: Theory_FiniteSample.tex
\section{Performance Analysis of { \textsc{NoMAD}} in Finite Sample Setting}
\label{sec:sample-complexity}
In describing \mainalgo{} (cf. Section~\ref{subsubsec:algo}) and in the analysis in the population setting (cf. Section~\ref{sec:theory}), we temporarily assumed that we have access to the actual distances $d_{ij}$ for the sake of exposition. However, in practice, these distances need to be estimated from samples $\{\mathbf{Y}_1,\ldots,\mathbf{Y}_n\}$. In what follows, we show that \mainalgo{}, with high probability, correctly outputs the equivalence class $[G]$ even if we replace $d_{ij}$ with the estimate $\widehat{d}_{ij}=\widehat{\Sigma^o}_{ij}/\sqrt{\widehat{\Sigma^o}_{ii} \widehat{\Sigma^o}_{jj}}$, where $\widehat{\Sigma^o}_{ij}$ is the $(i,j)$-th element in $\frac{1}{n}\sum_{i=1}^N{\mbf{Y}}_i\mbf{Y}_i^T$. We recall from Section~\ref{subsubsec:algo} that the subroutines in \mainalgo{} depend on the  \textsc{TIA} test, which relies on the distances $d_{ij}$. Thus, we establish the correctness of \mainalgo{} in the finite sample setting by showing that the empirical \textsc{TIA} (TIA with estimated distances) correctly identifies the ancestors in $G^J$ with high probability. We begin with the following assumptions.

\setcounter{theorem}{1}

\begin{assumption}
\label{assump:strong_faith_assump}
[$\gamma$-Strong Faithfulness Assumption]
\label{assump:strong-faith}
For any vertex triplet $i,j,k \in {\vobs{} \choose 3}$, if $i \notindep j \vert k$, then $\left \vert d_{ij}-d_{ik}-d_{jk} \right \vert > \gamma$.
\end{assumption}

$\gamma$-Strong Faithfulness assumption is a standard assumption used in the literature (\cite{kalisch2007estimating,uhler2013geometry}). 0-Strong-Faithfulness is just the usual
Faithfulness assumption discussed in Section~\ref{sub-sec:algo}. This motivates our next assumption which strengthens our requirement on ``spurious cancellations'' involving ancestors.
 
\begin{assumption}[Strong Ancestor consistency]
\label{assump:strong_anc_consistency}
For any triplet pair $U,W \in {\vtrue \choose 3} \setminus \mathcal{V}_{\rm star} \cup \mathcal{V}_{\rm sep}$ and any vertex pair $(x,a) \in U \times W$, there exists a constant $\zeta>0$, such that $\left \vert d_x^U + d_a^W -d_{xa}\right \vert > \zeta$.
\end{assumption}
Assumption~\ref{assump:strong_anc_consistency} is in direct analogy with  Assumption~\ref{assump:strong_faith_assump}.  As we show in Lemma~\ref{lemma:two_dstar_triplets_with_identical_ancestor}, for any pair $U,W \in {\vtrue \choose 3} \setminus \mathcal{V}_{\rm star} \cup \mathcal{V}_{\rm sep}$ (i.e., any pair that would fail the TIA test), there exists at least one triplet $\{x,a,b\}$ where $x \in U \mbox{ and } a,b \in W$ such that $d_{xa}-d_x^U - d_a^W \neq 0$ and $d_{xb}-d_x^U - d_b^W \neq 0$. This observation motivates us to replace the exact equality testing in the \textsc{TIA} test in Definition~\ref{def:TIA-test} with the following hypothesis test against zero: $\max\left\{\left|\widehat{d}_{xa}-\widehat{d}_x^U - \widehat{d}_a^W \right|,  \left|\widehat{d}_{xb}-\widehat{d}_x^U - \widehat{d}_b^W \right|\right\}\leq \xi$. We set $\xi < \frac{\zeta}{2}$. \\

Furthermore, in order to learn the distance between two hidden ancestors in $\anchid{}$, the ${\rm mode}$ test introduced in Subsection~\ref{subsubsec:algo} needs to be replaced with a finite sample version; we call this the $\epsilon_d$-mode test and this is formalized in Appendix~\ref{app:smp-complexity}. Finally, for any triplet $(i,j,k) \in {\vobs{} \choose 3}$, in order to check whether $i \indep j \vert k$, the  test in Fact~\ref{fact:sepration and pairwise correlation factorization (faithful_graph)} needs to be replaced as follows: $\vert \widehat{d}_{ij}- \widehat{d}_{ik}-\widehat{d}_{jk} \vert < \frac{\epsilon_d}{6}$. We now introduce two new notations to state our main result. Let $\rho_{\rm min} (p) = {\rm min}_{i,j \in {p \choose 2}}|\rho_{ij}|$ and $\kappa (p)=\log((16+\left(\rho_\mathrm{min}(p)\right)^2\epsilon_d^2)/(16-\left(\rho_\mathrm{min}(p)\right)^2\epsilon_d^2))$, where $\epsilon_d= {\rm min} (\frac{\xi}{14}, \gamma)$, where $\gamma$ is from Assumption~\ref{assump:strong_faith_assump}.

\setcounter{theorem}{2}

\begin{theorem}
\label{main--thm:samp-complexity}
Suppose the underlying graph $\gtrue{} $ of a faithful GGM satisfies Assumptions \ref{assump:strong_faith_assump}-\ref{assump:strong_anc_consistency}. Fix any $\tau \in (0,1]$. Then, there exists a constant $C>0$ such that if the number of samples $n$ satisfies $n > C \left(\frac{1}{\kappa (p)}\right) {\rm max} \left( {\rm log} \left(\frac{p^2}{\tau}\right), {\rm log } \left(\frac{1}{\kappa (p)}\right)\right)$, then with probability at least $1- \tau$, \mainalgo{} accepting $\widehat{d}_{ij}$ outputs the equivalence class $[G]$.
\end{theorem}

\setcounter{theorem}{0}
\begin{remark}
Theorem~\ref{main--thm:samp-complexity} indicates that the sample complexity of the  \mainalgo{} is dependent on the absolute minimum and maximum pairwise correlations $\rho_{\rm min}$ and $\rho_{\rm max}$, the number of vertices $p$, and the magnitude of the quantity from Assumption \ref{assump:strong_anc_consistency}. Specifically, in regimes of interest, we see that the sample complexity scales as a logarithm in the number of vertices $p$ and inversely in $\rho^2_{\rm min} (p)$ and $\epsilon_d^2$, thus allowing for robust model selection in the high-dimensional regime. Further note that our sample requirement can have an exponential dependence on $p$ based on the degree of $G$. However, we expect that this dependency can be improved; see Section~\ref{sec:discussion} for more on this avenue for future work.
\end{remark}

%% file: Simulation.tex
\section{Experiments}

\newcommand{\trueprecision}{\Theta^*}
\newcommand{\glassoprecision}{\Theta_{\rm glasso}}

We perform experiments on both a synthetic graph and an IEEE 33 bus system, which is a graphical representation of established IEEE-33 bus benchmark distribution system \citep{baran1989network}, to assess the validity of our theoretical results and to demonstrate the performance of \mainalgo{}. In particular, our experiments demonstrate how an unmodified graphical modeling algorithm (in this case  \glasso{} ~\cite{friedman2008sparse}) compares to \mainalgo{}. 
The synthetic graph $G_{\rm{syn}}$ we consider and graph associated with the IEEE-33 bus system are given in Fig.~\ref{fig:synth_graph1} and Fig.~\ref{fig:synth_graph2}, respectively. 
As a first step, we need some new performance metrics to make a fair comparison in light of the unidentifiability of the underlying graph structure from noisy data. In the following we will introduce some sets, and show that if an algorithm can recover all these sets, then the output graph (from that algorithm) is in equivalence class.

    \begin{figure*}
    \begin{subfigure}[b]{0.42\textwidth}
  \centering
    \includegraphics[width=3.3cm, height=2 cm]{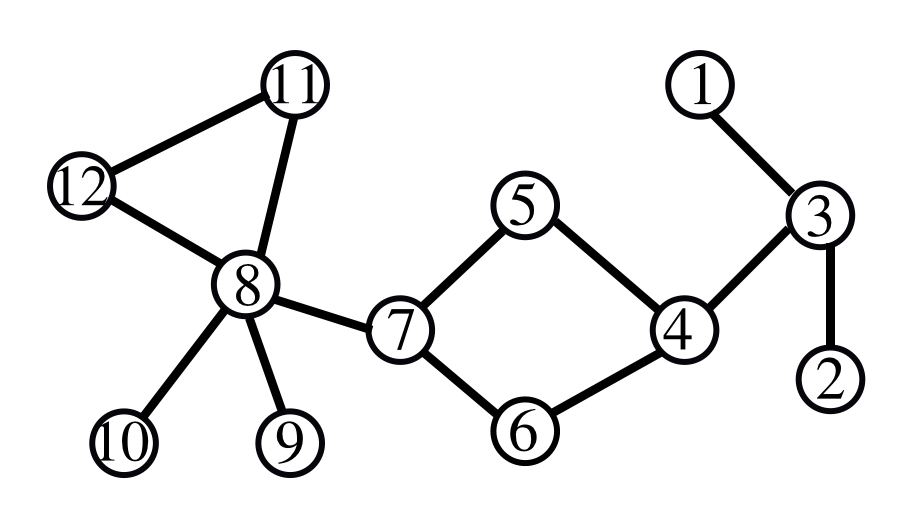}
  \caption{Synthetic graph, $G_{\rm{syn}}$}
    \label{fig:synth_graph1}
  \end{subfigure}
  \begin{subfigure}[b]{0.52\textwidth}
  \centering
    \includegraphics[width=3.6cm, height=2.8cm]{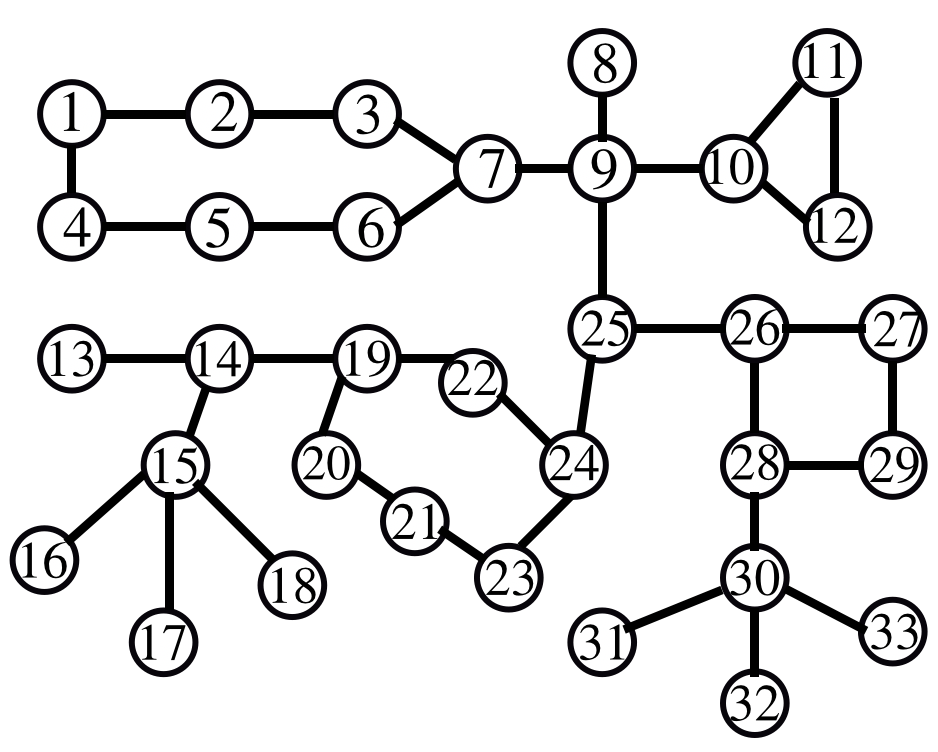}
\caption{IEEE 33 Bus Graph}
    \label{fig:synth_graph2}
  \end{subfigure}

  \caption{ Synthetic graph and IEEE-33 Bus system considered for our simulation}
    \label{fig:syn_graphs}
\end{figure*}

\begin{enumerate}
    \item {\bf Families.} For a vertex $i$, define its {\em family} $F_i$ as $\{v: {\rm deg}(v)=1 \mbox{ and } \{v,i\}\in E(G)\} \cup \{i\}$. Notice that the subgraph associated with each family is a tree. For synthetic graph $G_{\rm{syn}}$, the sets $\{1,2,3\}$, $\{8,9,10\}$ are some of the families in $G_{\rm{syn}}$; For IEEE 33 Bus Graph $\{15,16,17,18\}$, $\{30,31,32,33\}$ are some of the families. Let $\mathcal{F}= \bigcup_{i \in \vtrue{}}F_i$.

    \item {\bf Non-Cut Vertices.} For graph $G$, let $B_{\noncut}$ be the set of all non-cut vertices in a non-trivial block $B$. Define $\mathcal{B}_{\noncut} \triangleq \bigcup_{B \in \mathcal{B}^{\textsc{NT}}}B_{\noncut}$, where $\mathcal{B}^{\textsc{NT}}$ is the set of all non-trivial blocks. Recall from Section~\ref{sec:prelim} that in all the equivalent graphs the set of non-cut vertices remain unchanged.

\item {\bf Global Edges.} Let $K$ be the set of cut vertices who do not share an edge with a leaf vertex in $G$. For any vertex $k \in K$, let a family $F_k \in \mathcal{F}$ be such that there exists a vertex $f \in F_k$ such that $\{k,f\} \in E(G)$. Now, we will note two following condition for any neighbor $i \in N(k)$: (a) if $i \in K$, then $\{i,k\} \in E(G)$, and (b) otherwise, there exists a vertex $j \in F_k$ such that $\{j,k\} \in E(G)$. If condition (a) and condition (b) are met, global edge associated with $i$ are recovered correctly .

\end{enumerate} 

In Lemma~\ref{lemma:equiv-class-recovery}, we demonstrate that if an algorithm learns these sets correctly, and the conditions associated with the global edges are met, then an equivalent graph can be recovered. \\

 {\bf Experimental Setup}. We now describe our experimental setup. We generated precision matrices associated with $G_{\rm{syn}}$, and for IEEE-33 bus system, we added some loops. We then inverted the corresponding precision matrices to obtain the respective covariance matrices in the population setting. Next, we added a diagonal matrix of random positive values to the population covariance matrices to generate the corresponding noisy covariance matrices.\\ 

 We will now compare the performance of 0-1 loss of \mainalgo{} to \glasso{} under the influence of noise. In order to compare, our goal is to compare \mainalgo{} with the \textit{best} \glasso{}. Our protocol for selecting the \textit{best} \glasso{} is as follows: for a fixed maximum allowable diagonals, we selected the regularization parameter for which the output graph given the noise covariance matrix to \glasso{} is in equivalence class. Then, for that fixed regularization parameter we report the performance of \glasso{} for varying numbers of (increased) maximum allowable diagonals. Notice that the maximum allowed values of the diagonal elements of the matrix $D$ (which is being added to the generated covariance matrix) contain information about the potential influence of the noise in the setup. In order to study this influence we ran the experiment for a fixed sample size over 15 trials. In Fig.~\ref{fig:noiseInfluence}, we observe that up to the value of 2.5 and 1.5 for $G_{\rm{syn}}$ and IEEE-33 Bus, respectively, \glasso{} was able to recover an equivalent graph. On the contrary, \mainalgo{} was able to show a consistent performance in recovering an equivalent graph with various influence of noises.

 \begin{figure*}
    \begin{subfigure}[b]{0.5\textwidth}
    \centering
    \includegraphics[width=6cm, height=4cm]{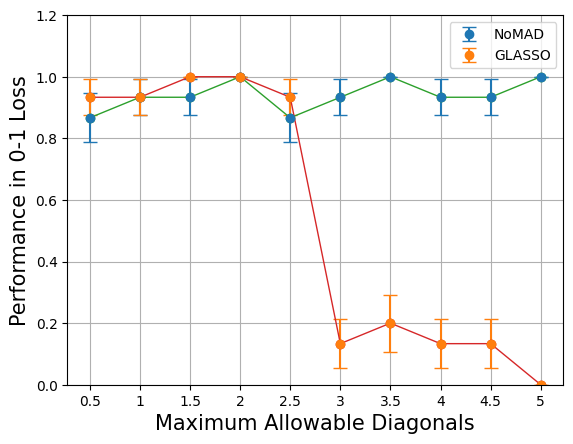}
    \caption{Performance for $G_{\rm{syn}}$}
    \label{fig:noiseInfluence_G1}
  \end{subfigure}
  \begin{subfigure}[b]{0.5\textwidth}
    \centering
    \includegraphics[width=6cm, height=4cm]{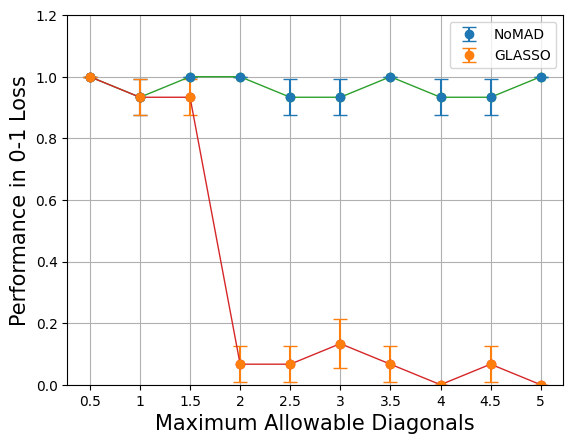}
    \caption{Performance for IEEE-33 Bus.}
    \label{fig:noiseInfluence_G2}
  \end{subfigure}

  \caption{ Performance of \glasso{} and \mainalgo{} for various degree of noises. Up to the values of 2.5 and 1.5 dollars for $G_{\rm{syn}}$ and IEEE-33 Bus, respectively, \glasso{} successfully reconstructed an equivalent graph. To investigate this influence, we conducted the experiment with a fixed sample size across 15 trials. In contrast, \mainalgo{} consistently demonstrated its ability to recover an equivalent graph under various noise influences. }
    \label{fig:noiseInfluence}
\end{figure*}

%% file: Conclusion.tex
\section{Conclusion and Future Directions}
\label{sec:discussion}

{\bf Conclusion.} In this paper, we consider model selection of non tree-structured Gaussian graphical models when the observations are corrupted by independent but non-identically distributed noise with unknown statistics. We first show that this ambiguity is unavoidable. Finally, we devise a novel algorithm, referred to as $\mainalgo{}$, which learns structure up to a small equivalence class.  \\

 {\bf Future Directions.} This paper opens up several exciting avenues for future research. First, our novel ancestor testing method can be used to identify the ancestors for other graphical models (beyond Gaussians) where information distance satisfies the factorization property in Fact~\ref{fact:sepration and pairwise correlation factorization (faithful_graph)}; e.g., Discrete graphical models, where the random vector $\mathbf{X}$ takes values in the product space $\mathcal{X}^p$, where $\left| \mathcal{X} \right| = k$. For any $i, j \in [p]$, let $\Upsilon_{ij} \in \mathbb{R}^{k\times k}$ denote the tabular representation of the marginal distribution of the pair $(X_i, X_j)$ and $\Upsilon_{ii}$ denote a diagonal matrix with the marginal distribution of $X_i$ on the diagonal. Then, it is known that the following quantity 
\begin{align*}
d_{ij} = \frac{\det\left( \upsilon_{ij} \right)}{\sqrt{\det\left( \upsilon_{ii}\upsilon_{jj} \right)}}
\end{align*}
can be taken to be the information distance. We refer the reader to \cite{semple2003phylogenetics} for more on this, including a proof of an equivalent version of Fact~\ref{fact:sepration and pairwise correlation factorization (faithful_graph)}. Note that the Ising model is a special case, and our results naturally extend to them.

Second, it is well known that $\rho_{\rm min}$ could scale exponentially in the diameter of the graph; this could imply that the sample complexity will scale polynomially in the number of vertices $p$ even for balanced binary trees, and as bad as exponential for more unbalanced graphs. Now, notice that in Subroutine~\ref{app:algo--Ident_Anc_Extnd_Dist}, \mainalgo{} identifies all the star triplets for any ancestor in $\gcombined{}$. Hence, this identification procedure in quite computationally expensive. As we reason this in theoretical section of the appendix that this computation is required in order to learn the pairwise distances $d_{ij}$ for each pair $(i,j)$ such that $i \in \vobs{}$ and $j \in \anchid{}$, where $\vobs{}$ and $\anchid{}$ is the set of observed vertices, and hidden ancestors, respectively. A promising future research work is to develop a TIA test which can obtain $\{d_{ij}\}_{i,j \in \vobs{} \cup \anchid{}}$ without iterating all the triplets in $\vobs{} \choose 3 $. Furthermore, the Subroutine~\ref{app:algo--Learn_clusters} can be redesigned to a computationally efficient one by learning the clusters in \textit{divide and conquer} manner. Another promising avenue for future research work is to obtain an upper bound on the diagonal entries $D_{ii}$ for which the underlying graph $G$ is identifiable. Finally, future research can be done to understand the robust identifiability and structure recoverability results for graphs that are locally tree-like, but globally loopy with a large girth \cite[see e.g.,][]{anandkumar2013learning}.

%% file: Appendix-Algo.tex
\begin{figure*}[t]
\center
{\Large \bf Appendix}\\
\vspace{3mm}
{\bf \Large Robust Model Selection of Gaussian Graphical Models}
\ \\ 
{Zahin, Anguluri, Sankar, Kosut, and Dasarathy}
\ \\ \ 
\hrule
\end{figure*}

\newcommand{\Balgo}{B_{\rm algo}}
\newcommand{\clustercut}{C_{\rm cut}} 
\newcommand{\clusternoncut}{C_{\rm non-cut}}
\newcommand{\distHid}{\mathcal{D}_{\rm hid}}
\newcommand{\distAll}{\widetilde{\mathcal{D}}}
\newcommand{\var}[1]{\text{\texttt{#1}}}
\newcommand{\func}[1]{\text{\textsl{#1}}}

\makeatletter
\newcounter{phase}[algorithm]
\newlength{\phaserulewidth}
\newcommand{\setphaserulewidth}{\setlength{\phaserulewidth}}
\newcommand{\phase}[1]{
   \vspace{-1.25ex}

  \Statex\leavevmode\llap{\rule{\dimexpr\labelwidth+\labelsep}{\phaserulewidth}}\rule{\linewidth}{\phaserulewidth}
  \Statex\strut\refstepcounter{phase}\textit{Phase~\thephase~--~#1}
  \vspace{-1.25ex}\Statex\leavevmode\llap{\rule{\dimexpr\labelwidth+\labelsep}{\phaserulewidth}}\rule{\linewidth}{\phaserulewidth}}
\makeatother

\setphaserulewidth{.7pt}

\setcounter{section}{0}
\renewcommand{\thesection}{\Alph{section}}

\section{Algorithmic Details}
\label{app:rev-algos}

 In the population setting, \mainalgo{} takes as input the pairwise distances $d_{ij}$, for all $i,j$ in the observed vertex set $\vobs{}$, and returns an articulated set tree $\ASTalgo{} \triangleq  (\Palgo{}, \Aalgo{}, \Ealgo{})$ (see Definition~\ref{def:art-set-graph}). Its operation is divided into  two main steps: 
(a) {learning $\Palgo{}$ and $\Aalgo{}$}; and (b) {learning $\Ealgo{}$}. These steps are summarized in the following.

 \subsection{Learning $\Palgo{}$ and $\Aalgo{}$ for $\ASTalgo{}$}

In \textit{Phase 1}, Subroutine~\ref{app:algo--Ident_Anc_Extnd_Dist} identifies the ancestors in $\gcombined{}$ using the pairwise distances $d_{ij}$ for all $i,j \in \vobs{}$. In this phase, it returns a collection  $\tripletCollection{}$ of vertex triplets such that each triplet collection $\mathcal{V} \in \tripletCollection{}$ contains (and only contains) all vertex triples that share an identical ancestor in $\gcombined{}$. The key component for this is to use \textsc{TIA} (Test Identical Ancestor). In the \textit{Phase 2}, Subroutine~\ref{app:algo--Ident_Anc_Extnd_Dist} enrolls each collection in $\tripletCollection{}$ to either $\tripletCollectionObs{}$ or $\tripletCollectionHid{}$, such that  $\tripletCollectionObs{}$  ($\tripletCollectionHid{}$) contains the collection of vertex triplets for which their corresponding ancestors are observed (hidden resp.), and observed ancestors are enrolled in the set $\ancobs{}$. For identifying the observed ancestors from $\tripletCollection{}$, Subroutine~\ref{app:algo--Ident_Anc_Extnd_Dist} does the following for each collection $\mathcal{V} \in \tripletCollection{}$: it checks for a vertex triplet $T$ in $\mathcal{V}$ for which one vertex in the triplet $T$ separates the other two.  In the final phase, Subroutine~\ref{app:algo--Ident_Anc_Extnd_Dist} accepts $d_{ij}$ for each pair $i,j \in \vobs{}$ and $\tripletCollectionHid{}$, and learns the pairwise distance $d_{ij}$ for each $i \in \vobs{}$ and $j \in \anchid{}$ by finding a vertex triplet $T$ in a collection $\mathcal{V}_j \in \tripletCollectionHid{}$ such that $T$ contains $i$. Then, in the next step, Subroutine~\ref{app:algo--Ident_Anc_Extnd_Dist} learns $d_{ij}$ for each $i, j \in \anchid{}$ by selecting the most frequent distance in $\Delta_{pq}$ as defined in Section~\ref{subsubsec:algo}.

\floatname{algorithm}{Procedure}

  \begin{wrapfigure}{R}{0.4\textwidth}
    \begin{minipage}{0.4\textwidth}
 \begin{algorithm}[H]
\caption{\scriptsize \textsc{TestIdenticalAncestor} (TIA)}
\label{app:algo--TIA}
\begin{algorithmic}[1]
\tiny
\Procedure{\textsc{TIA}}{$U,W$}    
 	\If{ for all $x \in U$, $\exists$ at least a pair $y,z \in W$ such that
 	
 	$d_x^U+d_y^W=d_{xy}$ and $d_x^U+d_z^W=d_{xz}$}
	\State \textbf{Return} \textsc{True}.
	\EndIf
	\State \textbf{Return} \textsc{False}.
\EndProcedure

\end{algorithmic}
\end{algorithm}
    \end{minipage}
  \end{wrapfigure}

 We next present Subroutine~\ref{app:algo--Learn_clusters} for clustering the vertices in the set $\vobs{} \setminus \ancobs{}$. It accepts $\ancobs{}$, $\anchid{}$, and $\{d_{ij}\}_{i \in \vobs{}, j \in \anchid{}}$, and enrolls each vertex in $\vobs{} \setminus \ancobs{}$ either in a \emph{leaf cluster} (see Phase 1 of Subroutine~\ref{app:algo--Learn_clusters}) or in an  \emph{internal cluster} ( see Phase 2 of Subroutine~\ref{app:algo--Learn_clusters}). Now, for the collection of leaf clusters $\mathcal{L}$, each cluster $L\in \mathcal{L}$ is associated to a unique element $a\in A$ such that $L_2$ is separated from $A \setminus a$ by $a$. Each cluster $I \in \mathcal{I}$ is associated with a subset of ancestors $I_1 \subset A$, such that $I_2$ is separated from all other ancestors in $A \setminus I_1$ by $I_1$.

  \floatname{algorithm}{Subroutine}
\begin{algorithm}
  \caption{ \scriptsize Identifying Ancestors and Extending the Pairwise Distance Set}
  \label{app:algo--Ident_Anc_Extnd_Dist}
  \begin{multicols}{2}
  \begin{algorithmic}[1]
  \tiny
	\State \textbf{Input:} Pairwise distances $\mathcal{D} = \{d_{ij}\}_{i, j \in \vobs{}}$, where $\vobs{}$ is the set of observed vertices.
	\State \textbf{Return:} A collection of vertex triplets $\tripletCollectionObs{}$ ($\tripletCollectionHid{}$ resp.) with observed (hidden resp.) ancestors, the set $\ancobs{}$ ($\anchid{}$ resp.) of observed (hidden resp.) ancestors, the set of pairwise distances $\{d_{ij}\}$ for each pair $i,j \in \vobs{} \cup \anchid{}$. 
	
	\State \textbf{Initialize:} $\tripletCollectionObs{}, \tripletCollectionHid{}, \distAll{}, \distHid{} \leftarrow \emptyset $, collection of vertex triplets $\mathcal{V} \triangleq {\vobs{} \choose 3}$, counter $n=1$
    \phase{Clustering Star Triplets}
	\For{each $ U \in \mathcal{V}$} 
	\State $\mathcal{V}_n\triangleq \{ W \subset \mathcal{V}:$ \textsc{TIA} $(U,W) \mbox{ is \textsc{True}}\} \cup U$

			\If{$ \vert \mathcal{V}_{n} \vert > 1$}   $n=n+1$
	\EndIf
			\State $\tripletCollection \leftarrow \tripletCollection \cup \mathcal{V}_n$ \Comment{enrolling the collection $\mathcal{V}_n$ to $\mathfrak{V}$} 
	\EndFor
\State \textbf{Return:} $\tripletCollection = \{\mathcal{V}\subset {V\choose 3}: \mbox{each $U\in \mathcal{V}$ has the same ancestor
}\}$, 

    \phase{Labeling Ancestors}
		\For{each collection $\mathcal{V} \in \tripletCollection{}$}
		\If{ $\exists$ a triplet $V \triangleq \{u,v,w\} \in \mathcal{V}$ s.t. $d_{uv}+d_{vw}=d_{uw}$}
		\State $\tripletCollectionObs{} \leftarrow \tripletCollectionObs{} \cup \mathcal{V}$
		\State $\ancobs{} \leftarrow \ancobs{} \cup v $
		\Else
			\State $\tripletCollectionHid{} \leftarrow \tripletCollectionHid{} \cup \mathcal{V}$
			\EndIf
			\EndFor
		\State Set	$\anchid{} \triangleq \{a_i \vert i \in [\left \vert \tripletCollectionHid{} \right\vert] \}$ \Comment{introduce one vertex for each element in $\tripletCollectionHid{}$}
    \phase{Learning the pairwise distance set $\{d_{ij}\}_{i,j \in \vobs{} \cup \anchid{}}$}
 		\For{each $\mathcal{V}_i \in \tripletCollectionHid{} $ }
 		\For{each $j \in \vobs{}$}
 		\State Find a triplet $U$  $\in \mathcal{V}_i$ s.t $U \ni j$ $\triangleright$ \textcolor{teal}{\footnotesize cf. Claim~\ref{claim:collect_all_vertex_land_dist}}

 		\State $\distAll{} \leftarrow \distAll{} \cup \left(\left( i,j\right), d_{j}^U \right) $ 
 		\EndFor

 		\EndFor
 		    \For{each $p\neq q \in \anchid{}$}
    \State Pick a pair of triplets $U_p \in \mathcal{V}_p, U_q \in \mathcal{V}_q$.
    \State $\Delta_{pq}= \left\{d_{xy} - (d_{x}^{U_p} + d_{y}^{U_r}):  x\in U_p, y \in U_q\right\}.$
    \State $\distHid{} \leftarrow \distHid{} \cup \left( \left( p,q \right), {\rm mode}(\Delta_{pq})  \right)$ $\triangleright$ \textcolor{teal}{\footnotesize most frequent element in $D_{pq}$}
    \EndFor
	\State \textbf{Return} $\tripletCollectionObs{}$, $\tripletCollectionHid{}$, $\ancobs{}$, $\anchid{}$, $\distAll{}$, and $\distHid{}$.
  \end{algorithmic}
  \end{multicols}
\end{algorithm}

 \begin{algorithm}
  \caption{\scriptsize \textsc{LearnClusters}}
  \label{app:algo--Learn_clusters}
  \begin{multicols}{2}
  \begin{algorithmic}[1]
  \tiny
\State \textbf{Input:} $\ancobs{}, \anchid{}, \mbox{ and } \mathcal{D}$, and $A \triangleq \ancobs{} \cup \anchid{} $.
	\State \textbf{Output:} A collection of leaf clusters $\mathcal{L}$ and internal clusters $\mathcal{I}$. 
	\State  \textbf{Initialize}: $\mathcal{L} \triangleq (L_1,L_2,L_3), \mathcal{I} \triangleq (I_1,I_2,I_3)$, and $L_1, L_2,L_3, I_1,I_2,I_3 \leftarrow \emptyset$. 
    \phase{Learning Leaf Clusters}
	\For{each $x\in \vobs{} \setminus \ancobs{}$}

	\If{ $\exists a\in A$ such that 
	$d_{xa}+d_{aa'}=d_{xa'}$ for all $a' \in A \setminus \{a\}$}
	\If{$\exists$ $L \in \mathcal{L}$ such that $L_1=a$}
	\State $L_2 \leftarrow L_2 \cup \{x\}$
    \Else
    \State $L \triangleq (a, \{x\}, \emptyset)$
    \State $\mathcal{L} \leftarrow \mathcal{L} \cup L$ 
	\EndIf
	\State $\vobs \leftarrow \vobs \setminus \{x\}$
	\EndIf
	\EndFor
	\State \textbf{Return} $\mathcal{L}= \{L : L_2 \in 2^{\vobs{} \setminus \ancobs{}} \newline \mbox{s.t. $L_2$ is separated from $A \setminus L_1$ by $L_1$ where $\vert L_1 \vert=1$}\}$.
    \phase{Learning Internal Clusters}
	\For{each $x \in \vobs{} \setminus \ancobs{}$}
	\For {each $\Tilde{A} \subset 2^A$ s.t. $|\Tilde{A}|>1$}
	\For{each pair $k, \ell \in { \Tilde{A} \choose 2}$} 
	\If{there exists a pair $(k, \ell)$ s.t. $d_{xk}+d_{k \ell} = d_{x \ell} \mbox{ or } d_{x\ell} +d_{\ell k} = d_{xk}$} 
 	\EndIf
 		\State \textbf{Break}
	\EndFor
	\EndFor
	 \If{$\exists$ a $I \in \mathcal{I}$ such that $I_1 = \Tilde{A}$}
 \State $I_2 \leftarrow I_2 \cup \{x\}$.
 \Else
 \State $I \triangleq ( \Tilde{A}, \{x\}, \emptyset)$
 \State $\mathcal{I} \leftarrow \mathcal{I} \cup I $
 \EndIf
	\EndFor
\State	\textbf{Return} $\mathcal{I}= \{I : I_2 \in 2^{\vobs{} \setminus \ancobs{}} \newline \mbox{s.t. $I_2$ is separated from $A \setminus I_1$ by $I_1$ where $\vert I_1 \vert>1$}\}$.
  \end{algorithmic}
  \end{multicols}
\end{algorithm}

 \floatname{algorithm}{Procedure}

   \begin{algorithm}
   \begin{multicols}{2}
	\begin{algorithmic}[1]
	\tiny
  	\caption{\scriptsize \textsc{NonCutTest}}
 	\label{app:algo--Cut_Test}
 		\State \textbf{Input:} A leaf cluster $L \in \mathcal{L}$ such that $\vert L_2 \vert \geq 2$.
 		\State \textbf{Output:} A set $\clustercut{}, \clusternoncut{} \triangleq L_2 \setminus \clustercut{}, L_3 \subseteq L_2$.
  		\State \textbf{Initialize:} $\clustercut{}$ with $L_2$.
 			\For{each $x \in L_2$}
 			\For{each pair $y,z \in L_2 \setminus \{x\}$}
	\State Pick any arbitrary pair $\alpha_1, \alpha_2 \in \vobs{} \setminus L_2$.
	\State $U_i \triangleq \{x,y,\alpha_1\}$ and $U_j \triangleq \{x,z,\alpha_2\}$.

	\If{\textsc{TIA}($U_i,U_j$) is \textsc{False}}
	\EndIf
	\State \textbf{Break}
	\State $\clustercut{} \leftarrow \clustercut{} \setminus \{x\}$ \Comment{$x$ is not a non-cut vertex.}
	\EndFor
	\EndFor
\If{$\left(\left | \clustercut{}\right|\right)> 1 \wedge \left (L_1 \notin \ancobs{} \right)$}
\State Pick an arbitrary vertex $a$ from $\clustercut{}$ and set $L_3 \triangleq a$.
\EndIf
	\State \textbf{Return} $\clustercut{}$, $\clusternoncut{}$, and $L_3$.
\end{algorithmic}
\end{multicols}
 \end{algorithm}

\floatname{algorithm}{Subroutine}

 \begin{algorithm}
 \begin{multicols}{2}
	\begin{algorithmic}[1]
  \tiny
  	\caption{\scriptsize \underline{P}artitioning \underline{a}nd learning \underline{l}ocal \underline{e}dges (\textsc{PaLE})}
 	\label{app:algo--PaLE}
 		\State \textbf{Input:} The observed vertex set $\vobs{}$, the collection of leaf clusters $\mathcal{L}$ and internal clusters $\mathcal{I}$. 
 		\State \textbf{Output:} The vertex set $\Palgo{}$, the articulation points $\Aalgo{}$, and a subset $\mathcal{E}_{\rm leaf}$ of the edge set $\Ealgo{}$ for $\ASTalgo{}$. Each element $E \in \mathcal{E}_{\rm leaf}$ is an ordered quadruple such that $E_1, E_2 \subseteq \vobs{}$, and $E_3 \in E_1, E_4 \in E_2$.
 		 	\State	\textbf{Initialize} $\Palgo{}, \Aalgo{}, \mathcal{E}_{\rm leaf}  \leftarrow \emptyset$.
\State $A_{\rm cluster} \triangleq \{ c \in L_1: c \in \ancobs{}\}$
	\State $\Palgo{} \leftarrow \Palgo{} \cup (\ancobs{} \setminus A_{\rm cluster})$, $\Aalgo{} \triangleq \ancobs{}$.
	
	\phase{Partitioning and Local Edge Learning w.r.t. the Leaf Clusters}

 	\For{each $L \in \mathcal{L}$ s.t. $(\vert L_2 \vert < 3) \wedge (L_1 \notin \ancobs{})$} \Comment{  ancestor in the leaf cluster is not observed.}

 	\State Pick an arbitrary vertex $a \in L_2$, $L_3 \leftarrow a$, $L_2 \leftarrow L_2 \setminus \{a\}$, $\Aalgo{} \leftarrow \Aalgo{} \cup \{a\}$.
 	\State $\Palgo{} \leftarrow \Palgo{} \cup L_2 \cup L_3$
\State $\mathcal{E}_{\rm leaf} \leftarrow \mathcal{E}_{\rm leaf} \cup \left( L_2,L_3,L_2,L_3\right)$
 	\EndFor
 			\For{each $L \in \mathcal{L}$ s.t. $ (\vert L_2 \vert \geq 3) \wedge (L_1 \notin \ancobs{})$} \Comment{ancestor in the leaf cluster is not observed.}
 			
 			\State Get $\clustercut, \clusternoncut{} \textrm{ and } L_3$ from $\textsc{NonCutTest}(L)$.
	\State $\Aalgo{} \leftarrow \Aalgo{} \cup L_3$
	\State Set $B \triangleq \clusternoncut{} \cup L_3$ \Comment{$\clusternoncut{}$ and $L_3$ contains the non-cut vertices and cut vertex, respectively.}
\State $\Palgo{} \leftarrow \Palgo{} \cup B \cup \underset{v\in \clustercut{} \setminus L_3}{\bigcup} \{v\}$ \Comment{$\clustercut{}$ can contain multiple cut vertices.}
	\State $\mathcal{E}_{\rm leaf} \leftarrow \mathcal{E}_{\rm leaf} \cup \underset{v \in \clustercut{} \setminus L_3}{\bigcup} \left(B, \{v\}, L_3, v\right) $
	\EndFor

 	\For{each $L \in \mathcal{L}$ s.t. $(\vert L_2 \vert =1) \wedge (L_1 \in \ancobs{})$} \Comment{ancestor in the leaf cluster is observed.}
 	\State $\Palgo{} \leftarrow \Palgo{} \cup L_2 \cup L_3$
\State $\mathcal{E}_{\rm leaf} \leftarrow \mathcal{E}_{\rm leaf} \cup \left( L_2,L_3,L_2,L_3\right)$
 	\EndFor
	
		 			\For{each $L \in \mathcal{L}$ s.t. $ (\vert L_2 \vert > 1) \wedge (L_1 \in \ancobs{})$} \Comment{ancestor in the leaf cluster is observed.}
 			
 			\State Get $\clustercut,\clusternoncut{} \textrm{ and } L_3$ from $\textsc{NonCutTest}(L)$
\State Set $B \triangleq \clusternoncut{} \cup L_1$.
\State $\Palgo{} \leftarrow \Palgo{} \cup B \cup \underset{v \in \clustercut{}}{\bigcup} \{v\} $. 
	\State $\mathcal{E}_{\rm leaf} \leftarrow \mathcal{E}_{\rm leaf} \cup \underset{v \in \clustercut{}}{\bigcup} \left(B, \{v\}, L_3, v\right) $
	\EndFor

 \phase{Partitioning w.r.t. the Internal Clusters} 
	\For{each $I \in \mathcal{I}$}
	\For{each $i \in I_1$}
	\If{$i \notin \ancobs{}$}
	 \State Find the $L\in \mathcal{L}$ s.t. $L_1=i$
	 \State $I_3 \leftarrow I_3 \cup i$, $\Aalgo{} \leftarrow \Aalgo{} \cup \{i\}$
	 \EndIf
	 	\EndFor
 	\State $\Palgo{} \leftarrow \Palgo{} \setminus I_3$ 
	 	\State $B \triangleq \{I_2 \cup I_3\} $ and $\Palgo{} \leftarrow \Palgo{} \cup B$. \Comment{$B$ is an internal non trivial block.}
	\EndFor

	\State \textbf{Return} $\Palgo{}$, $\Aalgo{}$, 
	and a subset $\mathcal{E}_{\rm leaf}$ of $\Ealgo{}$.
\end{algorithmic}
\end{multicols}
 \end{algorithm}

 \floatname{algorithm}{Procedure}

\begin{minipage}{0.46\textwidth}
\begin{algorithm}[H]
    \centering
  	\caption{\scriptsize \textsc{NonBlockNeighbors}}
   \begin{algorithmic}[1]
     \tiny
 	\label{app:algo--NonBlockNeighbors}
 		\State \textbf{Input:} An ancestor vertex $u$, $\mathcal{C}_u$, $\Aalgo{}$, and the extended distance set $\Dextend{}$.
 		\State \textbf{Output:} Neighbors $\delta(u)$ of $u$ such that they do not belong to the clusters that contains $u$.

  		\State \textbf{Initialize:} $\delta(u) \triangleq \Aalgo{} \setminus \underset{C \in \mathcal{C}_u}{\bigcup}C_3$.
 			\For{each $x \in \delta(u)$}
 			\If{$\exists$ a vertex $b \in \mathcal{C}_u \setminus $ s.t. $d_{ux}=d_{ub}+d_{bx}$}
 			\State $\delta(u) \leftarrow \delta(u) \setminus x$
 			\EndIf
	\EndFor
	\For{ each $k, \ell \in {{\delta(u)} \choose 2}$}
	\If{$d_{uk}+d_{k\ell}=d_{u \ell}$ }
	\State $\delta(u) \leftarrow \delta(u) \setminus \ell$.
	\EndIf
	\EndFor
    \end{algorithmic}
\end{algorithm}
\end{minipage}
\hfill
\begin{minipage}{0.46\textwidth}
\floatname{algorithm}{Subroutine}

\begin{algorithm}[H]
    \centering
  	\caption{\scriptsize Learning $\Ealgo{}$ for $\ASTalgo{}$}
   \label{app:algo--EAST}
\begin{algorithmic}[1]
  \tiny
 		\State \textbf{Input:} The collection of leaf clusters $\mathcal{L}$ and internal clusters $\mathcal{I}$, $\mathcal{C} \triangleq \mathcal{L} \cup \mathcal{I}$, a subset $\mathcal{E}_{\rm leaf}$ of $\Ealgo{}$.
 		\State \textbf{Output:} An edge set $\Ealgo{}$ for $\ASTalgo{}$.
 		\State \textbf{Initialize:} $\Ealgo{} \leftarrow \mathcal{E}_{\rm leaf}$.
 	
 		 		\For{each $u \in \Aalgo{}$}
 \State	Let $C \in \mathcal{C}$ be the cluster such that $ C_3 \ni u $
 		 \State Get $\delta(u)$ from $\textsc{NonBlockNeighbors}(u, C, \Aalgo{})$.
 		 
 		 	\For{each $P_u \in \Palgo{}$ s.t. $P_u \ni u$}
  \For{each $v \in \delta(u)$}
  \State $\Ealgo{} \leftarrow \Ealgo{} \cup  (P_u, \{v\}, u, v)$
  \EndFor
  \EndFor

\EndFor
 	
 \State \textbf{Return} The edge set $\Ealgo{}$ for $\ASTalgo{}$.
    \end{algorithmic}
\end{algorithm}
\end{minipage}

  We now discuss \textsc{NonCutTest} appears in Procedure~\ref{app:algo--Cut_Test}. The goal of \textsc{NonCutTest} is to learn (a) the non-cut vertices, and (b) potential cut vertices of a non-trivial block from a leaf cluster. \textsc{NonCutTest} accepts a set $W \subseteq \vobs{}$ s.t. $\vert W \vert \geq 3$, and partitions the vertex set $W$ into $\clustercut{}$ (the set of potential cut vertices) and $\clusternoncut{}$ (the set of vertices which \textit{can not} be a cut vertex). Then we use Subroutine~\ref{app:algo--PaLE} for learning (a) vertex set $\Palgo{}$, (b) articulation points $\Aalgo{}$, and a subset of the edge set $\Ealgo{}$ for $\ASTalgo{}$. The Subroutine~\ref{app:algo--PaLE} learns (a), (b), and (c) from both leaf clusters and internal clusters. In the following, we list all the possible cases of leaf clusters Subroutine~\ref{app:algo--PaLE} considered in learning $\Palgo{}$ and $\Aalgo{}$: Leaf clusters contains 1. At most two vertices with hidden ancestor, 2.  More than two vertices with hidden ancestor, 3. One vertex with observed ancestor, and 4. More than one vertex with observed ancestor. For each $I \in \mathcal{I}$, Subroutine~\ref{app:algo--PaLE} checks whether $i \in \ancobs{}$. If $i \notin \ancobs{}$, then the subroutine finds the leaf cluster $L$ s.t. $L_1 \ni i$.

\subsection{Learning $\Ealgo{}$ for $\ASTalgo{}$} 
The next goal of \mainalgo{} is to learn to learn the edge set $\Ealgo{}$ for $\Aalgo{}$. Precisely, \mainalgo{} learns the neighbors of each articulation point in $\Aalgo{}$. The learning of $\Ealgo{}$ is divided into two steps: (a) Learn the neighbors of each articulation points (appears in Procedure 3), and (b) use the information obtained from (a) to construct $\Ealgo{}$ (appears in Procedure 4).

%% file: Appendix-Theory.tex
\section{Theory: Guaranteeing the Correctness of the \mainalgo{}}
In this section we will prove that $\mainalgo{}$ correctly learns the equivalence class. We star this section by restating Theorem~\ref{thm--main_thm_pop} from Section~\ref{sec:theory}.
  \begin{theorem}
Consider a covariance matrix $\Sigma^*$ whose conditional independence structure is given by the graph $G$, and the model satisfies Assumption~\ref{assump:anc_faithfulness}. Suppose that according to the problem setup in Section~\ref{subsec: robust model selection}, we are given pairwise distance $d_{ij}$ of a vertex pair $(i,j)$ in the observed vertex set $\vobs{}$, that is, $d_{ij} \triangleq - {\rm log} |\rho_{ij}|$ where $\rho_{ij} \triangleq \Sigma^o_{ij} / \sqrt{\Sigma^o_{ii} \Sigma^o_{jj}}$. Then, given the pairwise distance set $\{d_{ij}\}_{i,j \in \vobs{}}$ as inputs, \mainalgo{} outputs the equivalence class $[\gtrue{}]$.
 \end{theorem}

\label{app-theory-population}
\noindent {\bf Proof Outline.} We show that $\mainalgo{}$ correctly learns the equivalence class by showing that it can correctly learn $\ASTalgo{}$. Given this, and using Definition~\ref{def:equiv_rel}, the entire equivalence class can be readily generated. We show that \mainalgo{} learns $\ASTalgo{}$ correctly by proving that (a) the vertex set $\Palgo{}$, (b) the articulation points $\Aalgo{}$, and (c) the edge set $\Ealgo{}$ are learnt correctly. Following is the outline for (a) and (b). From Section~\ref{sub-sec:algo}, it is clear that \mainalgo{} succeeds in finding the ancestors, which is the first step, provided the TIA tests succeed (established in Lemma~\ref{lemma:two_dstar_triplets_with_identical_ancestor}).Then, Proposition~\ref{prop:correctness of Pale} establishes that \mainalgo{} learns $\Palgo{}$ and $\Aalgo{}$ correctly. The proof correctness of this step crucially depends on identifying the non-cut vertices (c.f. Lemma~\ref{lemma:cut_test}). Then, for establishing the correctness of \mainalgo{} in learning $\Ealgo{}$, \mainalgo{} learns the neighbor articulation points of each articulation point. Proposition~\ref{prop:correctness of East} shows that \mainalgo{} correctly learns $\Ealgo{}$.

\setcounter{theorem}{0}
\begin{lemma}
\label{lemma:AST_is_a_Tree}
Let $\gtrue{}$ be a graph on vertex set $V$, and $\bctree{(G)}$ be the corresponding articulated set tree of $\gtrue{}$. Then, $\bctree{(G)}$ is a tree.
\end{lemma}

\begin{proof}
 We prove $\bctree{(G)}$ is a tree by showing that $\ASTalgo{}$ is connected and acyclic. Suppose on the contrary that $\bctree{(G)}$ contains a cycle $B'$. Then, $B'$ is a non-trivial block in $\gtrue{}$ with no cut vertex. This would contradict the maximality of the non-trivial blocks contained in the cycle $B'$. Hence, any cycle is contained in a unique non-trivial block in $\bctree{(G)}$. We now show that $\ASTalgo{}$ is connected. Recall that vertices in $\bctree{(G)}$ can either be a non-trivial block or a singleton vertices not part of any non-trivial block. Consider any vertex pair $(u,v)$ in $\bctree{(G)}$. We will find a path from $u$ to $v$. Suppose that $u$ and $v$ are non-singletons, and associated with non-trivial blocks $B_u$ and $B_v$ respectively. Since, $\gtrue{}$ is connected, $\exists$ a path between the articulation points of $B_u$ and $B_v$. Hence, $u$ and $v$ are connected in $\bctree{(G)}$. The other cases where one of them is a singleton vertex or both are singleton vertices follows similarly.

\end{proof}

\noindent We now show that \mainalgo{} correctly learns $[\gtrue{}]$. For the graph $G$ on a vertex set $V$, let $\gcombined{} = (\vcombined{}, \ecombined{})$ be defined as in Subsection~\ref{subsec: robust model selection}. Let $\anccombined{}$ be the set of ancestors in $\gcombined{}$. Recall that \mainalgo{} only observes samples from a subset $\vobs{} \subseteq \vcombined{}$ of vertices. \mainalgo{} uses $\{d_{ij}\}_{i,j \in \vobs{}}$ to learn $\ASTalgo{}$, which in turn will output $[\gtrue{}]$. Hence, each theoretical section first states a result of $\gcombined{}$ {\em assuming} that the pair $(\vcombined{}, \ecombined{})$ is known.

\noindent {\bf Correctness in Learning $\Palgo{}$ and $\Aalgo{}$}. 
We first establish that $\mainalgo{}$ correctly identifies ancestors in $\gcombined{}$. In the following, we first identify the vertices in $G$ which are ancestors in $\gcombined{}$. Then, in Lemma~\ref{lemma:atleast_two_triplets_for_an_ancestor}, we show the existence of at least two vertex triplets for each ancestor in $\gcombined{}$. Finally, in Proposition~\ref{prop:correctness_of_CollectStarTriplets}, we show that Subroutine~\ref{app:algo--Ident_Anc_Extnd_Dist} correctly identifies the star triplets in $\gcombined{}$. We start with introducing $uw$-separator.

\begin{definition}[$uw$-separator]
\label{def:uw_separator}
Consider an arbitrary pair $u,w \in \vtrue{}$ in the graph $\gtrue{}$. We say $v \in \vtrue{} \setminus \{u,w\}$ is a $uw$- separator in $\gtrue{}$ if and only if any path $\pi \in \mathcal{P}_{uw}$ contains $v$. 
\end{definition}

\begin{lemma}
\label{lemma:2_way_seprator}
A vertex $a \in \vcombined{}$ is an ancestor in $\gcombined{}$ if and only if $a$ is an $uw$- separator in $\gtrue$, for some $u,w \in \vtrue{} $.  
\end{lemma} 

\begin{proof}
($\Rightarrow$) Suppose $a \in \vcombined{}$ is an ancestor in $\gcombined$. Then, we show that $a$ is an $uw$- separator in $\gtrue{}$. Fix a triplet $T \triangleq \{a_1^e,a_2^e,a_3^e\} \in \mathcal{V}_a$, where $\mathcal{V}_a \triangleq $ collection of all triplets with ancestor $a$ in $\gcombined$. Then, any path $\pi \in \mathcal{P}_{a_i^ea_j^e}$ contains $a$ in $\gcombined$, for $i,j=1,2, \mbox{ and } 3$. Thus, $a_i^e \indep a_j^e \vert a$, and $a$ is an $uw$-separator with $u=a_i^e, \mbox{ and } w=a_j^e$. Furthermore, for a joint graph following is true for any vertex $u$ and its corresponding noisy samples $u^e$: $u^e \indep v \vert u$ for all $v \in \vcombined{} \setminus \{u,u^e\}$. Hence, we can conclude that $a_i \indep a_j \mid a$, and $a$ is an $uw$-separator in $\gtrue{}$.

($\Leftarrow$) Suppose that  $ \exists \; u,w \in \vtrue{} $ for which $a \in \vtrue{}$ is an $uw$- separator in $\gtrue{}$. Then, we show that $v$ is an ancestor in $\gcombined{}$ by constructing a triplet $T$ with ancestor $a$. This construction directly follows from Definition~\ref{def:uw_separator} and Definition~\ref{def:minimal-mut-sep}.

\end{proof}

\begin{lemma}
\label{lemma:nC_nBC_Leaf_is_not_an_ancestor}
Let $V_{\emph{cut}}$ be the set of all cut vertices in $\gtrue{}$. Then, there does not exist any pair $u,w \in \vtrue{}$ such that $b \in \vtrue{} \setminus V_{\emph{cut}}$ is an $uw$- separator in $\gtrue{}$.
\end{lemma}

\begin{proof}
Let $b \in \vtrue{} \setminus V_{\emph{cut}}$. Then, notice that $b$ can be either (1) a non-cut vertex of a non-trivial block or a (2) leaf vertex in $\gtrue$. For (1), by the definition of a block, any non-cut vertex ceases to be a $uw$-separator for any $u \neq b \mbox{ and } w \neq b$ in $\vtrue$. For (2), since $b$ is a leaf vertex, its degree is one, and hence, cannot be a $uw$-separator for any $u \neq b \mbox{ and } w \neq b$ in $\vtrue$. 
\end{proof}

\begin{lemma}
\label{lemma:atleast_two_triplets_for_an_ancestor}
Let $\anccombined{}$ be the set of all ancestors in $\gcombined{}$. Then, for each $a \in \anccombined{}$, there exists at least two triplets $U, W \in { \vobs \choose 3}$ for which $a$ is the ancestor in $\gcombined{}$. 
\end{lemma}

\begin{proof}
We construct two triplets for any ancestor in $\gcombined{}$. Lemma \ref{lemma:nC_nBC_Leaf_is_not_an_ancestor} states that only a cut-vertex in $\gtrue{}$ is an ancestor in $\gcombined{}$. First, let $c$ be a cut-vertex of a non-trivial block $B$ in $\gtrue{}$. Pick any two non-cut vertices $x,y \in B \setminus \{c\}$. Then, consider the following two triplets in $\vobs{}$ : $\{x^e, c^e, \alpha_1^e\}$ and $\{y^e, c^e, \alpha_2^e\}$, where $\alpha_1, \alpha_2 \in \vtrue{} \setminus B$. Then both  $\{x^e, c^e, \alpha_1^e\}$ and $\{y^e, c^e, \alpha_2^e\}$ share the ancestor $c$ in $\gcombined{}$. Now, let $c$ be a cut vertex which is not in any non-trivial block. Consider two blocks $B_i$ and $B_j$ such that $B_i \indep B_j \mid c$. Then, consider the following pair: $\{i_1,c,j_1\}$ and $\{i_2,c,j_2\}$  s.t. $i_1,i_2 \in B_i$ and $j_1,j_2 \in B_j$. Notice that $\{i_1^e,c^e,j_1^e\}$ and $\{i_2^e,c^e,j_2^e\}$ in ${\vobs{} \choose 3}$ share the ancestor $c$ in $\gcombined{}$. Finally, if $\gtrue{}$ is a tree on three vertices, then $\gtrue{}$ has an unique ancestor.

\end{proof}

\begin{claim}
\label{claim:transfer-betn-true-and-combined}
Let (a) $\{i,j,k\}$ be a vertex triple in $\gtrue{}$, and (b) $i^e$ be the corresponding noisy counterpart of $i$. Then, $j$ separates $i \mbox{ and }k$ if and only if $j$ separates $i^e \mbox{ and }k$ in combined graph $\gcombined{}$
\end{claim}
\begin{proof}
The forward implication directly follows from the construction of joint graph. For the reverse implication suppose that in $\gcombined{}$, $i^e \indep k \vert j$. We show that this implies $i \indep k \vert j$ and $k$ in $\gtrue{}$. Suppose on the contrary that $i \notindep k \vert j$. That means $\exists$ a path $\pi$ between $i \mbox{ and } k$ that does not contain $j$. Now, notice that $\pi \cup \{i,i^e\}$ is a valid path between $i^e$ and $k$ in $\gcombined{}$ that does not contain $j$, and it violates the hypothesis.

\end{proof}

\noindent The following lemma relates an observed ancestor in a triplet $T$ with the remaining pair.

\begin{lemma}
\label{lemma:identify-observed-ancestors}
Suppose that a triplet $T \in {\vobs{} \choose 3}$ is a star triplet in $\gcombined{}$. A vertex $v \in T$ is an $uw-$ separator for $u,w \in T \setminus v$ if and only if $v$ is the ancestor of $T$.
\end{lemma}

\begin{proof}
Suppose that a vertex $v \in T$ is an $uw-$ separator for $u,w \in T \setminus v$. We show that $v$ is an ancestor. As $v$ is an $uw-$ separator, i.e., $u \indep w \vert v$. Suppose on the contrary that $v' \neq v$ is the ancestor of $T$ in $\gcombined{}$. We show that $v$ is not an $uw-$ separator for $u,w \in T \setminus v$. As $v'$ is the ancestor of $T$, $u \indep w \vert v'$. (according to Definition~\ref{def:minimal-mut-sep}). This contradicts the hypothesis that $u \indep w \vert v$. Thus, $v \mbox{ and } v'$ are identical. Therefore, $v$ is the ancestor of  $\{u,v,w\}$. The reverse implication follows from Definition~\ref{def:minimal-mut-sep}. 
\end{proof}

\noindent We will now prove the correctness of the TIA test. We proceed with the following claim.
\begin{claim}
\label{claim:triplet_sepration_criteria}
Suppose that $U \mbox{ and } W \in {\vobs \choose 3}$ are star triplets with non-identical ancestors $r_u$ and $r_w$, resp. Then, there exists a vertex $u \in U$ and a pair, say $w_2, w_3\in W$, such that all paths $\pi \in \mathcal{P}_{uw_i}$ for $i=1,2$ contain both $r_u$ and $r_w$.
\end{claim}

\begin{proof}
Without loss of generality, let $W = \{w_1,w_2,w_3\}$. We prove this claim in two stages. In the first stage, we show that for each vertex $u \in U$ there exists at least a pair $w_2, w_3 \in W$ such that $u \indep \{w_2, w_3\} \vert r_w$. Then, in the next stage we show that there exists a vertex $u \in U$ such that $u \indep r_w \vert r_u$.  For the first part, suppose on the contrary that there exists a vertex $u \in U$ and a pair $w_2, w_3 \in W$ such that there exists a path $\pi_2 \in \mathcal{P}_{uw_2}$ and a path $\pi_3 \in \mathcal{P}_{uw_3}$ such that $r_w \notin \pi_2$ and $r_w \notin \pi_3$. Then, one can construct a path between $w_2$ and $w_3$ that does not contain $r_w$, which violates the hypothesis that $W$ is a star triplet. Now, in the next step of proving the claim, we show that there exists a vertex $u \in U$ such that $u \indep r_w \vert r_u$. Now, suppose that for all $u \in U$ there exists a path between $u$ and $r_w$, that does not contain $r_u$. We will next show that this implies there has to be a path between $u_1$ and $u_2$ ($u_1,u_2 \in U$) that does not include $r_u$. We will show this constructively. Let $s$ be the last vertex in the path $\pi_{u_1 r_w}$ that is also contained in $\pi_{u_2 r_w}$. Note that $\pi_{u_1 s}$ and $\pi_{u_2 s}$ are valid paths in the graph, and that their concatenation is a valid path between $u_1$ and $u_2$. This proves that $u_1$ and $u_2$ are connected by a path that is not separated by $r_u$, and hence contradicting the hypothesis that $U$ is a star triplet. Finally, let $u' \in U$ be the vertex for which $u' \indep r_w \mid r_u$. Then, there exists a triplet $\{u',w_2,w_3\}$ such that both $r_u$ and $r_w$ separates $u'$ and $w_2$, and both $r_u$ and $r_w$ separates $u'$ and $w_3$.  
\end{proof}

\noindent Using Claim~\ref{claim:triplet_sepration_criteria}, we now show the correctness of our TIA test.  Recall that the TIA $(U,W)$ accepts triplets $U,W \in {\vobs \choose 3}$, and returns \textsc{True} if and only if $U \mbox{ and } W$ share an ancestor in $\gcombined$. Also recall the following assumption: Let $U,W \in {\vtrue \choose 3} \setminus \mathcal{V}_{\rm star} \cup \mathcal{V}_{\rm sep}$. Then,
(i) there are no vertices $x \in U$ and $a \in W$ that satisfy $d_x^U + d_a^W =d_{xa}$, and (ii) there does not exist any vertex $r \in \vtrue$ and $x \in U$ for which the distance $d_{xr}$ satisfies relation in \eqref{eq:ancestor_distance_triplet}.

\begin{lemma}(\textbf{Correctness of TIA test})
\label{lemma:two_dstar_triplets_with_identical_ancestor}Fix any two vertex triplets $U \neq W \in {\vobs \choose 3}$. TIA$(U,W)$ returns \textsc{True} if and only if $U$ and $W$ are star triplets in $\combinedgraph$ with an identical ancestor $r \in \vtrue$.
\end{lemma}
\begin{proof} From Subroutine \ref{app:algo--TIA} returning TRUE is same as checking that for all $x \in U$, there exist at least two vertices $y,z \in W$ such that both of the following hold
  \begin{align}
  \label{eq:iden_anc_1}   d_{x}^{U}+d_{y}^{W}&=d_{xy},\\
    \label{eq:iden_anc_2}  d_{x}^{U}+d_{z}^{W}&=d_{xz}.
 \end{align}

\noindent Suppose that $U$ and $W$ are star triplets with an identical ancestor $r \in \vtrue$. We prove by contradiction. Let $a\in U$ and assume that there is {\em at most} one vertex $x\in V$ such that $d^U_a + d^W_x = d_{ax}$. Therefore, one can find two vertices $y_1,y_2 \in V$ such that 
\begin{equation}
    d_a^U + d^W_{y_i} \neq d_{a{y_i}}, \;\; i=1,2. \label{eq:violation}
\end{equation} 
However, from our hypothesis that $U$ and $W$ are star triplets with the common ancestor $r$, we know that $d_a^U = d_{ar}$ and $d^W_{y_i} = d_{r {y_{i}}}$, for $i=1,2$. This, along with \eqref{eq:violation}, implies that $r$ does not separate $a$ from $y_1$ or $y_2$. For $i=1,2$, let $\pi_{a y_i}$ be the path between $a$ and $y_i$ that does not include $r$. We will next show that this implies there has to be a path between $y_1$ and $y_2$ that does not include $r$. We will show this constructively. Let $s$ be the last vertex in the path $\pi_{a y_1}$ that is also contained in $\pi_{a y_2}$. Note that $\pi_{y_1 s}$ and $\pi_{s y_2}$ are valid paths in the graph, and that their concatenation is a valid path between $y_1$ and $y_2$. This proves that $y_1$ and $y_2$ are connected by a path that is not separated by $r$, and hence contradicting the first hypothesis.

\noindent For the reverse implication, we do a proof by contrapositive. Fix two triplets $U$ and $W$. Suppose that $U$ and $W$ are {\em not star triplets with an identical ancestor} in $\gcombined$. We will show that this implies that there exists at least one vertex in $U$ for which no pair in $W$ satisfies both Eq. \eqref{eq:iden_anc_1} and Eq. \eqref{eq:iden_anc_2}. To this end, we will consider all three possible configurations for a triplet pair $U \mbox{ and }W$ where they are not star triplets with an identical ancestor in $\gcombined$: 1. $U$ and $W$ are star triplets with a non-identical ancestor in $\gcombined$, 2. Both $U \mbox{ and } W$ are non-star triplets in $\gcombined$, 3. $U$ is a star triplet and $W$ is a non-star triplet in $\gcombined$. Then, for each configuration, we will show that there exists at least a vertex $x \in U$ for which no pair in $W$ satisfies both Eq. \eqref{eq:iden_anc_1} and Eq. \eqref{eq:iden_anc_2}

\noindent {\bf $U$ and $W$ are star triplets with non-identical ancestors.}
 Let $U$ and $W$ be two star triplets with two ancestor $r_{u}\mbox{ and }r_w$, respectively, such that $r_u \neq r_{w}$. As $U \mbox{ and } W$ are star triplets, $d_{x}^U$ and $d_{y}^W$ returns the distance from their corresponding ancestors $d_{xr_u}$ for all $x \in U$, and  $d_{yr_w}$ for all $y \in W$, respectively. Now, according to the Claim~\ref{claim:triplet_sepration_criteria}, there exists a vertex triplet, say $\{u,w_1,w_2\}$ w.l.o.g., where $u \in U$ and $w_1,w_2 \in W$ such that $u$ is separated from $w_i$ for $i=1,2$ by both $r_u \mbox{ and } r_w$. Furthermore, the same $u$ identified above is separated from $r_w$ by $r_u$. This implies that $d_{uw_i}=d_{ur_u}+d_{r_uw_i}=d_{ur_u}+d_{r_ur_w}+d_{r_ww_i}$ for $i=1,2$. As we know that $r_u$ and $r_w$ are not identical, $d_{r_ur_w} \neq 0$, which implies that $d_{uw_i} \neq d_{ur_u}+d_{r_ww_i} $, where $i=1,2$. Thus we conclude the proof for the first configuration by showing that there exists a vertex $u \in U$ and a pair $w_1, w_2 \in W$ such that the identities in \eqref{eq:iden_anc_1} and \eqref{eq:iden_anc_2} do not hold.

 \noindent {\bf $U$ is a star triplet and $W$ is a non-star triplet in $\gcombined$.}
We show that there exists a triplet $\{y,a,b\}$ where $y \in U$ and $a,b \in W$ such that identities in \eqref{eq:iden_anc_1} and \eqref{eq:iden_anc_2} do not hold.
Let $W$ be a non-star triplet, and $U$ be a star triplet with the ancestor $r \in \vtrue$ in $\gcombined$. Now, as $U$ is a star triplet, $d_{x}^U$ returns the distance from its ancestor $d_{xr}$ for all $x \in U$. Suppose that there exists a vertex pair $x \in U$ and $a \in W$ for which $d_{xr}+d_a^W=d_{ax}$. We know that for a non-star triplet $W$, the computed distance $d_a^W \neq d_{ar} $ for any $a \in W$ from Assumption~\ref{assump:anc_faithfulness}. Thus, for the pair $\{x,a\}$, $d_{xr}+d_{ar} \neq d_{xa}$. This implies from the Fact~\ref{fact:sepration and pairwise correlation factorization (faithful_graph)} that $x \notindep a \mid r$. Similarly, we can conclude that $x \notindep b \mid r$. Then, $y \indep a \mid r$ and $y \indep b \mid r$. Otherwise, one can construct a path between $y$ and $x$ that does not contain $r$ which violates the assumption that $U \ni x,y$ is a star triplet with ancestor $r$. As $y \indep a \mid r$ and $y \indep b \mid r$, using the  Fact~\ref{fact:sepration and pairwise correlation factorization (faithful_graph)} we have that $d_{yr}+d_{ra}=d_{ya}$ and $d_{yr}+d_{rb}=d_{yb}$. As $a \in W$, and $d_{ar} \neq d_a^W$, thus, $d_{yr}+d_a^W \neq d_{ya}$. Similarly, for the pair $\{y,b\}$, we have that $d_{yr}+d_b^W \neq d_{yb}$. Thus, for the triplet $\{y,a,b\}$, the identities in Eq. \eqref{eq:iden_anc_1} and \eqref{eq:iden_anc_2} do not hold.

\noindent  {\bf $U \mbox{ and } W$ are both non-star triplets in $\gcombined$.} The proof for this configuration follows from the Assumption~\ref{assump:anc_faithfulness}. 

\noindent Notice that these three cases combined proves that the \textsc{TIA} test returns TRUE if and only if the triplets considered are both start triplets that share a common ancestor.
\end{proof}

\noindent Now recall that the first phase of Subroutine~\ref{app:algo--Ident_Anc_Extnd_Dist} identifies the star triplets in $\gcombined{}$, the observed ancestors in $\gcombined{}$, and outputs a set $\anchid{}$ such that $\left| \anchid{}  \right|$ equals to the number of hidden ancestors in $\gcombined{}$. Formally, the result is as follows.

\begin{proposition}[Correctness of Subroutine~\ref{app:algo--Ident_Anc_Extnd_Dist} in identifying ancestors]
\label{prop:correctness_of_CollectStarTriplets}
Given the pairwise distances $\{d_{ij}\}_{i,j \in \vobs{} }$, Subroutine~\ref{app:algo--Ident_Anc_Extnd_Dist} correctly identifies (a) the star triplets in $\combinedgraph$, (b) the observed ancestors in $\gcombined{}$, and (b) introduces a set $\anchid{}$ such that $\left| \anchid{}  \right|$ equals to the number of hidden ancestors in $\gcombined{}$ 
\end{proposition}

\begin{proof}
Combining Lemma ~\ref{lemma:atleast_two_triplets_for_an_ancestor} and Lemma ~\ref{lemma:two_dstar_triplets_with_identical_ancestor}, we prove that the Subroutine~\ref{app:algo--Ident_Anc_Extnd_Dist} successfully cluster the star triplets in $\gcombined{}$. Then, it partitions $\tripletCollection{}$ into $\tripletCollectionObs{}$ and $\tripletCollectionHid{}$ s.t. following is true: for any triplet collection $\mathcal{V}_i \in \tripletCollectionObs{}$ ($\mathcal{V} \in \tripletCollectionHid{}$ resp.), the ancestor of the triplets in $\mathcal{V}_i$ is observed (hidden resp.)  Finally, Subroutine~\ref{app:algo--Ident_Anc_Extnd_Dist} outputs a set $\anchid{}$ s.t. $\left| \anchid{}  \right|= \left | \tripletCollectionHid{}\right|$.
\end{proof}

\noindent We show the correctness of Subroutine~\ref{app:algo--Ident_Anc_Extnd_Dist} in learning (a) $\{d_{ij}\}_{i^e \in \vobs{}, j \in \anchid{}}$, and (b) $\{d_{ij}\}_{i,j \in \anchid{}}$.

\begin{claim}
\label{claim:collect_all_vertex_land_dist}
Fix any $a \in \anccombined{}$, where $\anccombined{} =$ set of all ancestors. Let $\mathcal{V}_a$ be the collection of vertex triplets which shares common ancestor. Then, any $i \in \anccombined{}$ is s.t. at least one triplet in $\mathcal{V}_a$.
\end{claim}

\begin{proof}
Fix any ancestor $a \in \anccombined{}$. Construct a triplet $T_i$ for a fixed vertex $i \neq a \in \vcombined{}{}$ s.t. $a$ is the ancestor of $T_i$ in $\gcombined{}$. From Lemma ~\ref{lemma:nC_nBC_Leaf_is_not_an_ancestor}: $a$ is a cut vertex in $\gtrue{}$. Thus, fixing $i \mbox{ and } a$, find another vertex $w \in \vtrue{}$ such that $a$ separates $i$ and $w$ in $\gtrue{}$. Hence, from Lemma~\ref{lemma:2_way_seprator} we can conclude the following: $a$ is the ancestor for the triplet $T_i \triangleq \{i,a,w\}$ in $\gcombined{}$.
\end{proof}

\begin{claim}
\label{claim:simple-sep-cpaim-for-mod-proof}
Let $U_i \mbox{ and } U_j$ be both star triplets with ancestor $i \mbox{ and } j $ respectively, and $i \neq j$. Let $x \in U_i$ and $y\in U_j$ be a vertex pair such that $x \indep y \vert i$ and $x \notindep y \vert j$ Then, $x \notindep i \vert j$.
\end{claim}
\begin{proof}
$x \indep y \vert i$ implies  any path between $x \mbox{ and } y$ contains $i$. $x \notindep y \vert j $ implies $\exists$ a path $\pi$ between $x \mbox{ and } y$ that does not contain $j$. Notice that, the path $\pi$ contains $i$. As $\pi$ contains both $x \mbox{ and } i$, $\exists$ a path between $x \mbox{ and } i$ which does not contain $j$. Hence, $x \notindep i \vert j$.  
\end{proof}

\begin{lemma}
\label{lemma:learning_pairwise_dist_betn_landmarks}
 For any pair of distinct ancestors $i, j \in \anccombined{}$, pick arbitrary triplets $U_i \in \mathcal{V}_i$ and $U_j \in \mathcal{V}_j$.  Define the set $D(U_i, U_j)$ as follows:
\begin{align}
    \Delta(U_i, U_j)\triangleq \left\{d_{xy} - (d_{x}^{U_i} + d_{y}^{U_j}):  x\in U_i, y \in U_j\right\}
\end{align}
 The most frequent element in $\Delta(U_i, U_j)$, that is, ${\rm mode}(\Delta_{ij})$ is the true distance $d_{ij}$ with respect to $\gcombined{}$. 
\end{lemma}

\begin{proof} To aid exposition, we suppose that $U_i = \left\{ x_1, x_2, x_3\right\}$ and $U_j = \{y_1, y_2, y_3\}$. We also define for any $x\in U_i$ and $y\in U_j$: $\Delta(x,y) \triangleq d_{xy} - \left( d_{x}^{U_i} + d_y^{U_j} \right)$. Observe that according to the Claim \ref{claim:triplet_sepration_criteria}, for two start triplets $U_i,U_j \in {\vobs{} \choose 3}$ with non-identical ancestors, there exist a vertex, say $x_1\in U_i$ and a pair, say $y_1,y_2 \in U_j$ such that following is true: $x_1 \indep y_i \vert i$ and $x_1 \indep y_i \vert j$  for $i=1,2$ Furthermore, the same $x_1$ (identified above) is separated from $j$ by $i$, that is, $x_1 \indep j \vert i$. This similar characterization is also true for a vertex triplet where one vertex is from $U_j$ and a pair from $U_i$. Now observe that 
\begin{align*}
\Delta(x_1, y_1) &= d_{x_1 y_1} - d_{x_1 i} - d_{y_1 j}
= d_{x_1 j} + d_{y_1j} - d_{x_1i} - d_{y_1 j}
= d_{x_1 j} - d_{x_1 i}
= d_{ij}.
\end{align*}
Similarly, it can be checked that $\Delta(x_1, y_2) = d_{ij}$. The similar calculation can be shown for the other triplet (where one vertex is from $U_j$ and a pair from $U_i$).  In other words, we have demonstrated that 4 out of the 9 total distances in $D(U_i, U_j)$ are equal to $d_{ij}$. All that is left to be done is to show that no other value can have a multiplicity of four or greater.

\noindent Now, our main focus is to analyze the five remaining distances, i.e., $\Delta(x_3,y_3)$, $\Delta(x_3,y_1)$, $\Delta(x_3,y_2)$, $\Delta(x_1,y_3)$, and $\Delta(x_2,y_3)$, for two remaining configurations: (a) $x_3$ is separated from $y_3$ by only one vertex in $\{i,j\}$, and (b) $x_3 \notindep y_3 \vert i$ \emph{and} $x_3 \notindep y_3 \vert j$. For configuration (a), consider without loss of generality  that $x_3 \indep y_3 \vert i$ and $x_3 \notindep y_3 \vert j$. Then, according to Claim~\ref{claim:simple-sep-cpaim-for-mod-proof}, we have the following two possibilities:

\begin{enumerate}
    \item \underline{\textit{$x_3 \indep y_3 \vert i, x_3 \notindep y_3 \vert j,  \mbox{ and }  x_3 \indep j \vert i$}:} As $x_3 \notindep y_3 \vert j$, it must be the case that $x_3 \indep y_{\nu} \vert j$ for $\nu=1,2$. Otherwise, one can construct a path between $y_1 \mbox{ and } y_{\nu}$ which does not contain $j$, and that violates the hypothesis that $U_j = \{y_1, y_2, y_3\}$ is a star triplet. Next, notice that in this set up, $x_3 \indep j \vert i$. Now, notice the following:
    
    \begin{align*}
        \Delta(x_3,y_{\nu}) &=d_{x_3y_{\nu}}- d_{x_3i}-d_{y_{\nu}j}
         = d_{x_3j}+d_{y_{\nu}j}-d_{x_3i}-d_{y_{\nu}j}
     = d_{x_3i}+d_{ij}-d_{x_3i}=d_{ij}.
    \end{align*}
     Therefore, for this set up, six distances are equal to $d_{ij}$. 
 
    \item \underline{\textit{$x_3 \indep y_3 \vert i, x_3 \notindep y_3 \vert j, \mbox{ and } x_3 \notindep j \vert i$}:} As $x_3 \notindep y_3 \vert j$, it must be the case that $x_3 \indep y_{\nu} \vert j$ for $\nu=1,2$. Otherwise, one can construct a path between $y_1 \mbox{ and } y_{\nu}$ which does not contain $j$, and that violates the hypothesis that $U_j = \{y_1, y_2, y_3\}$ is a star triplet. $\Delta(x_3,y_{\nu})=d_{x_3y_{\nu}}- d_{x_3i}-d_{y_\nu j}= d_{x_3j}+d_{y_\nu j}-d_{x_3i}-d_{y_\nu j}=d_{x_3j}-d_{x_3i}$. Now, $d_{x_3j}-d_{x_3i}$ equals to $d_{ij}$ implies that $x_3 \indep j \vert i$ which contradicts the setup. Therefore, $\Delta(x_3,y_{\nu})$ not equals to $d_{ij}$. Therefore, for this set up, even if three remaining distances are equal, correct $d_{ij}$ will be chosen.
\end{enumerate}

\noindent In the following we will analyze the distance between $\Delta(x_3,y_3)$ and $\Delta(x_3,y_{\nu})$ using the following assumption common in graphical models literature: For any vertex triplet $i,j,k \in {\vobs{} \choose 3}$, if $i \notindep j \vert k$, then $\left \vert d_{ij}-d_{ik}-d_{jk} \right \vert > \gamma$.

{\small
\begin{align*}
    &\Delta(x_3,y_3)-\Delta(x_3,y_{\nu})
    =  d_{x_3y_3}-d_{x_3i}-d_{y_3j}-d_{x_3y_{\nu}} +d_{x_3i}+d_{y_{\nu}j},\\
    =& d_{x_3i}+d_{y_3i}-d_{x_3i}-d_{y_3j}-d_{x_3j}-d_{y_{\nu}j}+d_{x_3i}+d_{y_{\nu}j}
    =d_{y_3i}+d_{x_3i}-d_{y_3j}-d_{x_3j}
    =d_{x_3y_3}-d_{y_3j}-d_{x_3j}.
\end{align*}
}

\noindent Now, as $x_3 \notindep y_3 \vert j$ according to Assumption~\ref{assump:strong-faith}, $\left \vert \Delta(x_3,y_3)-\Delta(x_3,y_{\nu}) \right \vert > \gamma$ for $\nu=1,2$. For configuration (b), we analyze the five remaining distances, i.e., $\Delta(x_3,y_3)$, $\Delta(x_3,y_1)$, $\Delta(x_3,y_2)$, $\Delta(x_1,y_3)$,  and $\Delta(x_2,y_3)$, and show that these five distances can not be identical which in turn will prove the lemma.

\noindent  \underline { $x_3 \notindep y_3 \vert i$ \emph{and} $x_3 \notindep y_3 \vert j$}. For this configuration we note the following two observations:

\begin{enumerate}[label=O\arabic*]
    \item As $x_3 \notindep y_3 \vert j$, it must be the case that $x_3 \indep y_{\nu} \vert j$ for $\nu=1,2$. Otherwise, one can construct a path between $y_1 \mbox{ and } y_{\nu}$ which does not contain $j$, and that violates the hypothesis that $U_j = \{y_1, y_2, y_3\}$ is a star triplet.
    \item Similarly, as $x_3 \notindep y_3 \vert i$, it must be the case that $x_{\nu} \indep y_3 \vert i$ for $\nu=1,2$. Otherwise, one can construct a path between $x_1 \mbox{ and } x_{\nu}$ which does not contain $i$, and that violates the hypothesis that $U_i = \{x_1, x_2, x_3\}$ is a star triplet. 
\end{enumerate}

\noindent Recall that our goal for configuration (b) is to analyze the distances $\Delta(x_3,y_3)$, $\Delta(x_3,y_1)$, $\Delta(x_3,y_2)$, $\Delta(x_1,y_3)$, and $\Delta(x_2,y_3)$. We start with the distance pair $\Delta(x_3,y_{\nu})$ and $\Delta(x_3,y_3)$ for $\nu=1,2$.
\begin{align*}
    \Delta(x_3,y_{\nu})  \overset{(a)}{=}d_{{x_3j}}+d_{y_\nu j}-d_{x_3i}-d_{y_\nu j}
    = d_{x_3j}-d_{x_3i},
\end{align*}

\noindent where (a) follows from the O1. Furthermore, the distance $\Delta(x_3,y_3) =d_{x_3y_3}-d_{y_3j}-d_{x_3i}$. Now,  $\Delta(x_3,y_\nu )$ equals to $\Delta(x_3,y_3)$ implies that $d_{x_3j}-d_{x_3i}=d_{x_3y_3}-d_{y_3j}-d_{x_3i}$ which is equivalent to saying that $d_{x_3j}+d_{y_3j}$ equals to $d_{x_3y_3}$. Then, $d_{x_3j}+d_{y_3j} = d_{x_3y_3}$ will imply $x_3 \indep y_3 \vert j$ -- which contradicts the hypothesis of the configuration that $x_3 \notindep y_3 \vert j$.

\noindent Thus, $\Delta(x_3,y_3)$ is not equal to $\Delta(x_3,y_{\nu})$ for $\nu=1,2$. (based on O1). Similarly, (based on the O2) $\Delta(x_3,y_3)$ is not equal to $\Delta(x_{\nu},y_3)$ for $\nu=1,2$. Thus, the distance $\Delta(x_3,y_3)$ is not equal to any of the following distances: $\Delta(x_3,y_1), \Delta(x_3,y_2), \Delta(x_1,y_3)$, and $\Delta(x_2,y_3)$.

\noindent Now all that remains to prove the lemma is to show that the 4 (remaining) distances $\Delta(x_3,y_1)$, $\Delta(x_3,y_2)$, $\Delta(x_1,y_3)$, and $\Delta(x_2,y_3)$ are not identical. To this end, we analyze two distances: $\Delta(x_1,y_3)$ and $\Delta(x_3,y_1)$. First notice from the O2 that $\Delta(x_1,y_3) = d_{x_1i}+d_{x_3i}-d_{x_3i}-d_{y_3j}$ equals to $d_{y_3i}-d_{y_3j}$, and $\Delta(x_3,y_1)=d_{x_3j}+d_{y_3j}-d_{x_3i}-d_{y_3j}$ equals to $d_{x_3j}-d_{x_3i}$. As neither $i$ nor $j$ is separating $x_3$ from $y_3$, the event that $d_{y_3i}-d_{y_3j}$ equals to $d_{x_3j}-d_{x_3i}$ happens only on a set of measure zero. We end this proof by computing the distance between $\Delta(x_3,y_3)$ and $\Delta(x_3,y_{\nu}$.
\begin{align*}
     \Delta(x_3,y_3)- \Delta(x_3,y_{\nu})
    = d_{x_3y_3}-d_{y_3j}-d_{x_3i} - d_{x_3j}-d_{y_{\nu}j}+d_{x_3i} +d_{y_{\nu}j}=d_{x_3y_3}-d_{y_3j}-d_{x_3j}.
\end{align*}
Now, as $x_3 \notindep y_3 \vert j$, according to Assumption~\ref{assump:strong-faith}, $\left \vert \Delta(x_3,y_3)-\Delta(x_3,y_{\nu}) \right \vert > \gamma$ for $\nu=1,2$.
\end{proof}

\begin{lemma}[Correctness in {\em extending the distances}.] 
\label{lemma:correctness_in_extending_distances}
\label{prop:correctness_of_CollectStarTriplets} 
Given $\{d_{ij}\}_{i,j \in \vobs{}}$, Subroutine~\ref{app:algo--Ident_Anc_Extnd_Dist} correctly learns (a) $\{d_{ij}\}_{i,j \in \vobs{} \cup A}$ and (b) $\{d_{ij}\}_{i,j \in \anchid{}}$, where $A \triangleq \ancobs{} \cup \anchid{}$.
\end{lemma}

\begin{proof}
Follows directly from Claim~\ref{claim:collect_all_vertex_land_dist} and Lemma~\ref{lemma:learning_pairwise_dist_betn_landmarks}. 
\end{proof}

\begin{lemma}[Correctness in learning clusters]
\label{lemma:correctness-in-clustering}

Subroutine~\ref{app:algo--Learn_clusters} correctly learns leaf clusters and internal clusters.
\end{lemma}
\begin{proof}
As the distances $\{d_{ij}\}_{i,j \in \anchid{}}$ and $\{d_{ij}\}_{i,j \in \vobs{} \cup \anchid{}}$ are learned correctly by Subroutine~\ref{app:algo--Ident_Anc_Extnd_Dist}, where $\vobs{}$ and $\anchid{}$ is the set of observed vertices, and hidden ancestors, respectively, the correctness of learning the leaf clusters and internal clusters follows from Fact~\ref{fact:sepration and pairwise correlation factorization (faithful_graph)}.
\end{proof}

\begin{lemma}
\label{lemma:cut_test}
Let $L \subset 2^V$ be a subset of vertices in $\gtrue{}$ s.t. only noisy samples are observed from the vertices in $L$. Then,  $\exists$ a vertex $v \in L$, where $v \in V_{\rm cut}$, s.t. $v$ separates $L \setminus \{v\}$ from the remaining vertices $v' \in V_{\rm cut} \setminus \{v\}$ Let $L^e$ be the noisy counterpart of $L$. The noiseless counterpart of $x^e$ is a non-cut vertex if and only if $\exists$ at least a pair $y^e,z^e \in L^e \setminus \{x^e\}$ such that $TIA\left(\{x^e,y^e,\alpha_1^e\}, \{x^e,z^e,\alpha_2^e\}  \right)$ returns \textsc{False}, where $\alpha_1^e, \alpha_2^e \in \vcombined{} \setminus L^e $.
\end{lemma}

\begin{proof}
($\Rightarrow$)
Suppose that $x$ is a non-cut vertex of a non-trivial block in $\gtrue{}$. We show the existence of a pair $y^e,z^e \in L^e \setminus \{x^e\}$ in $\vcombined{}$ such that $TIA\left(\{x^e,y^e,\alpha_1\}, \{x^e,z^e,\alpha_2^e\}  \right)$ returns \textsc{False}, where $\alpha_1^e, \alpha_2^e \in \vcombined{} \setminus L^e $. From Section~\ref{sec:prelim} we have that any non-trivial block in $\gtrue{}$ has at least three vertices. Hence, there exists another vertex $y^e \in L^e$ for which the noiseless counterpart is a non-cut vertex. We will show that one of $\{x^e,y^e, \alpha_1^e\}$ and $\{x^e,z^e, \alpha_2^e\}$ is not a star triplet, where $z^e \in L^e \setminus x^e,y^e$, and $\alpha_1^e,\alpha_2^e \in \vcombined{} \setminus L^e$. Then, the $TIA\left(\{x^e,y^e,\alpha_1^e\}, \{x^e,z^e,\alpha_2^e\}  \right)$ being \textsc{False} will follow from Lemma~\ref{lemma:two_dstar_triplets_with_identical_ancestor}. As $x$ and $y$ both are non-cut vertices, there does not exist a cut vertex that separates $x \mbox{ and } y$ in $\gtrue{}$, which implies that there does not exist an ancestor $a$ in $\gcombined{}$ s.t. $x^e \indep y^e \vert a$. Hence, $\{x^e,y^e,\alpha_1^e\}$ is not a star triplet in $\gcombined{}$. Then, the proof follows from Lemma~\ref{lemma:two_dstar_triplets_with_identical_ancestor}.

\noindent ($\Leftarrow$ ) Notice that the pair $\left(\{x^e,y^e,\alpha_1^e\}, \{x^e,z^e,\alpha_2^e\}  \right)$ can not share an ancestor, as it would violate the claim that $\{x^e,y^e,z^e\}$ is in a leaf cluster. Then, from Lemma~\ref{lemma:two_dstar_triplets_with_identical_ancestor} we have that if $TIA \left(\{x^e,y^e,\alpha_1^e\}, \{x^e,z^e,\alpha_2^e\}  \right)$ returns \textsc{False}, then at least one of the triplets is a non-star triplet, which rules out the existence of star triplets with non-identical ancestor. Suppose that $\{x^e,y^e,\alpha_1^e\}$ is a non-star triplet. As $\alpha_1^e \notin L^e$, an ancestor separates $x^e \mbox{ and } \alpha_1^e$, and  $y^e \mbox{ and } \alpha_1^e$. Then, the ancestor identified above does not separate $x^e \mbox{ and } y^e$. Hence in $\gtrue{}$, there does not exist a cut vertex that separates $x$ and $y$.
\end{proof}

\noindent \underline{Unidentifiability of the articulation point from a leaf cluster.} According to Lemma ~\ref{lemma:cut_test} the $\textsc{NonCutTest}$ returns the non-cut vertices of a non-trivial block from a leaf cluster, and the next (immediate) step is to learn the cut vertices of the non-trivial blocks. We now present a claim which shows a case where identifying the articulation point from a leaf cluster is not possible. This ambiguity is exactly the ambiguity (in robust model selection problem) of the {\em label swapping of the leaf vertices with their neighboring internal vertices} of a tree-structured Gaussian graphical models \cite{katiyar2019robust}.

\begin{claim}
\label{claim:unidentifiability_of_cut_vertex}
Let (i) a vertex $v \in V_{\rm cut} $ separates a subset $L \subset 2^V$ of vertices from any $v' \in V_{\rm cut} \setminus v$ where $L$ contains at least one leaf vertex, and (ii) $L^e$ be the noisy counterpart of $L$. Then, there exist at least two vertices $x_1^e,x_2^e \in L^e \cup \{v^e\} $ such that $TIA(\{x^e,y^e,\alpha_1^e\}, \{x^e,z^e,\alpha_2^e\} )$ returns \textsc{True} for any pair $y^e,z^e \in L^e \cup \{v^e\}$ where $x^e \in \{x_1^e,x_2^e\}$, and $\alpha_1^e, \alpha_2^e \in \vcombined{} \setminus L^e \cup \{v^e\}$.

\end{claim}

\begin{proof}
As $v$ is a cut vertex, $v^e \indep y^e \vert v$, $v^e \indep \alpha_1^e \vert v$, and $y^e \indep \alpha_1^e \vert v$. Here, $v$ is a unique separator since no other cut vertex (or ancestor in $\gcombined{}$) separates $v^e$ and $y^e$. Hence, $v$ is the ancestor of $\{v^e,y^e,\alpha_1^e\}$ in $\gcombined{}$. Similarly, one can construct another triplet $\{v^e,x^e,\alpha_1^e\}$ which has an ancestor $v$ in $\gcombined{}$. Hence, $TIA(\{v^e,y^e,\alpha_1^e\}, \{v^e,x^e,\alpha_2^e\})$ will return \textsc{True}. Now, let us consider a leaf vertex $x_1$ in $L$. Now, $x_1^e \indep \alpha_1^e \vert v$, $x_1^e \indep y^e \vert v$, and $\alpha_1^e \indep y^e \vert v$. Hence, $v$ is the ancestor of $\{x_1^e, y^e, \alpha_1^e\}$. Similarly, one can construct another triplet such that $v$ is the ancestor of $\{x_1^e, x^e, \alpha_2^e\}$. Hence, $TIA(\{x_1^e,y^e,\alpha_1^e\}, \{x_1^e,x^e,\alpha_2^e\})$ will return \textsc{True}.
\end{proof}

\begin{proposition}
\label{prop:correctness of Pale} 
Suppose that Subroutine~\ref{app:algo--PaLE} is invoked with the correct leaf clusters and internal clusters. Further suppose that \textsc{NonCutTest} succeeds in identifying the non-cut vertices of a non-trivial block. Then, Subroutine~\ref{app:algo--PaLE} correctly learns $\Palgo{}$ and $\Aalgo{}$ for $\ASTalgo{}$.

\end{proposition}

\begin{proof}
 According to Lemma~\ref{lemma:cut_test}, Subroutine~\ref{app:algo--PaLE} correctly learns the non-cut vertices of any non-trivial block $I$ with more than one cut vertices. If the cut vertex is observed, then it is identified in Subroutine~\ref{app:algo--Ident_Anc_Extnd_Dist}, and declared as one the articulation points of the vertex $I$ in $\Palgo{}$. Otherwise, the noisy counterpart belongs to a leaf cluster associated with an hidden ancestor, and the cut vertex be identified by selecting the label of the leaf cluster which is associated with the hidden ancestor (unobserved cut vertex of non-trivial block.)

\end{proof}

\noindent We now establish the correctness of \mainalgo{} in the learning the edge set $\Ealgo{}$ for $\ASTalgo{}$. This goal is achieved correctly by Procedure \textsc{NonBlockNeighbors} of \mainalgo{}.

\begin{proposition}
\label{prop:correctness of East} 
Suppose that Procedure 4 is invoked with the  correct $\Palgo{}$ and $\Aalgo{}$. Then, Procedure 4 returns the edge set $\Ealgo{}$ correctly.
\end{proposition}

\begin{proof} Procedure \textsc{NonBlockNeighbors} correctly learns the neighbors of any fixed articulation point in $\Aalgo{}$ by ruling out the non-neighbor articulation points in $\ASTalgo{}$. First, the procedure gets rid of the articulation points which are separated from the articulation points of the same vertex in $\Aalgo{}$. Then, from the remaining articulation points it chooses the set of all those articulation points such that no pair in the set is separated from  each other by the fixed articulation point. Then, Procedure 4 creates edges between vertices which contains the neighboring articulation points.
\end{proof}

\noindent {\bf Constructing the Equivalence Class.} Finally, in order to show that we can construct the equivalence class $[\gtrue{}]$ from the articulated set tree $\ASTalgo{}$, we note some additional definitions in the following. For graph $G$, let $B_{\noncut}$ be the set of all non-cut vertices in a non-trivial block $B$. Define $\mathcal{B}_{\noncut} \triangleq \bigcup_{B \in \mathcal{B}^{\textsc{NT}}}B_{\noncut}$, where $\mathcal{B}^{\textsc{NT}}$ is the set of all non-trivial blocks. Let a set $F_i$ referred as a {\em family} be defined as $\{v: {\rm deg}(v)=1 \mbox{ and } \{v,i\}\in E(G)\} \cup \{i\}$ where $E(G)$ is the edge set of  $G$, and let $\mathcal{F}= \bigcup_{i \in \vtrue{}}F_i$. Let $K$ be the set of cut vertices whose neighbors do not contain a leaf vertex in $G$. For any vertex $k \in K$, let a family $F_k \in \mathcal{F}$ be such that there exists a vertex $f \in F_k$ such that $\{k,f\} \in E(G)$. For example, in Fig.~\ref{fig:true_graph (zero_corrupted)}, $\mathcal{F} = \left \{\{10,11,12,13\}, \{14,15,16\}, \{17,20,21\} \right \}$; two sets $ \{1,2,3\}, \{18,19\}$ in $ \mathcal{B}_{\noncut}$, and $K = \{4,6,7,8,9\}$. Also, for example, $F_4= \{10,11,12,13\}$. For any arbitrary graph $\widetilde{G}$, let $B_{\noncut} (\widetilde{G})$, $\mathcal{F}(\widetilde{G})$, and $K(\widetilde{G})$ be the corresponding sets from $\widetilde{G}$.

Now, notice that in $\ASTalgo{}$, each vertex $k \in K$ has at least an edge in $\ASTalgo{}$. Let $N_{\textrm{art}}(k)$ be the neighbors of $k \in K$ in the edge set $\Ealgo{}$ returned for $\ASTalgo{}$.  Now, notice that as as long as $\mathcal{B}_{\noncut}$, $\mathcal{F}$, and $K$ are identified correctly in $\ASTalgo{}$, and the following condition holds in $\Ealgo{}$ for any $i \in N_{\textrm{art}}(k)$ for each $k \in K$: (a) if $i \in K$, then $\{i,k\} \in E(G)$, and (b) otherwise, there exists a vertex $j \in F_k$ such that $\{j,k\} \in E(G)$. Informally, identifying $\mathcal{B}_{\noncut}$ and $\mathcal{F}$ correctly, makes sure that vertices that constructs the local neighborhoods of any graph in $[G]$ are identical; identifying $K$ correctly, and satisfying the above-mentioned condition makes sure that the correct articulation points are recovered. Notice that the sets $\mathcal{B}_{\noncut}$, $\mathcal{F}$, and $K$ are identical in all the graphs in Fig.~\ref{fig:three_graphs_from_same_EC}. Following proposition shows that the sets $\mathcal{B}_{\noncut}$, $\mathcal{F}$, and $K$ are identified correctly from $\ASTalgo{}$.

\begin{lemma}
\label{lemma:equiv-class-recovery}
    Let $\widetilde{G}$ be an arbitrary graph. Then, $\widetilde{G} \in [G]$ if and only if the following holds: \begin{enumerate}
        \item $\mathcal{B}_{\noncut}(\widetilde{G})= \mathcal{B}_{\noncut}(G) $, $\mathcal{F}(\widetilde{G})=\mathcal{F}(G)$, and $K(\widetilde{G})=K(G)$.
        \item For any vertex $k \in K$, let a family $F_k \in \mathcal{F}$ be such that there exists a vertex $f \in F_k$ such that $\{k,f\} \in E(\widetilde{G})$. Now, for any neighbor $i \in N(k)$: (a) if $i \in K$, then $\{i,k\} \in E(\widetilde{G})$, and (b) otherwise, there exists a vertex $j \in F_k$ such that $\{j,k\} \in E(\widetilde{G})$.
    \end{enumerate}
\end{lemma}

\begin{proof}
($\Rightarrow$) The forward implication follows from Definition~\ref{def:equiv_rel}.

($\Leftarrow$) For the reverse implication, notice that the first condition is associated with the equality between sets. $\mathcal{B}_{\noncut}(\widetilde{G})= \mathcal{B}_{\noncut}(G) $ implies non-cut vertices are identified correctly, and  $\mathcal{F}(\widetilde{G})=\mathcal{F}(G)$ implies families are identified correctly. The second condition implies that an edge associated with a vertex $k \in K$ will have an ambiguity when the other vertex is from a family. Recall that from Definition~\ref{def:equiv_rel}, the label of a cut vertex can be swapped with it's neighbor leaf vertices.
\end{proof}

The reverse implication of the above-mentioned proof can be understood as follows: Identifying $\mathcal{B}_{\noncut}$ and $\mathcal{F}$ ensures that essentially the \textit{local structures} are identical between $G$ and $\widetilde{G}$. Recovering $K$ correctly and satisfying the second condition ensure that these local structures are correctly attached at the appropriate points. 

\begin{proposition}[Correctness in Learning the Equivalence Class]
\label{prop:correct-equiv-class}
Suppose that $\Palgo{}$, $\Aalgo{}$, and $\Ealgo{}$ returned by $\ASTalgo{}$ is correct. Then, following is true: (a)The sets $\mathcal{B}_{\noncut}$, $\mathcal{F}$, and $K$ are identified correctly, and (b) the condition is true for $N_{\textrm{art}}(k)$ for each $k \in K$.

\end{proposition}

\begin{proof} We first show that \mainalgo{} correctly identifies the sets $\mathcal{B}_{\noncut}$, $\mathcal{F}$, and $K$. By Lemma~\ref{lemma:cut_test}, Subroutine~\ref{app:algo--PaLE} correctly identifies the set $\mathcal{B}_{\noncut}$. Now, recall that each $F \in \mathcal{F}$ is a set of vertices constructed with a cut vertex and its neighbor leaf vertices. Hence, each family $F \in \mathcal{F}$ is captured in one of the leaf clusters returned by Subroutine~\ref{app:algo--Learn_clusters}. As Subroutine~\ref{app:algo--PaLE} correctly identifies the non-cut vertices from each leaf cluster, $\mathcal{F}$ is identified correctly. Finally, by Claim~\ref{claim:unidentifiability_of_cut_vertex}, the ambiguity in learning an articulation point is present only when a cut vertex has leaf vertex as it's neighbor; but $K$ does not contain such cut vertices. Hence, Subroutine~\ref{app:algo--PaLE} correctly learns $K$.
 We now show that above-mentioned condition is satisfied for the neighbor articulation points in $N_{\textrm{art}}(k)$ for any $k \in K$. As $K$ are identified correctly by Subroutine~\ref{app:algo--PaLE}, and the Procedure 4 returns correct $N_{\textrm{art}}(k)$, it is clear that if any neighbor articulation point $i \in N_{\textrm{art}}(k) \cap K$, then $\{i,k\} \in E(G)$. Now, suppose that $i\notin N_{\textrm{art}}(k) \cap K$. Then, from Definition~\ref{def:equiv_rel}, the label of a cut vertex can be swapped with it's neighbor leaf vertices. As each family $F \in \mathcal{F}$ are identified correctly, there exists a vertex $j \in F_k$ (which is an {\em unidentified} cut vertex in $\gtrue{}$) such that $\{i,j\} \in E(G)$.
\end{proof}

%% file: Appendix_smpl_cmplxty.tex
\section{Sample Complexity Result}
\label{app:smp-complexity}
Recall that \mainalgo{} returns the equivalence class of a graph $G $ while having access only to the noisy samples according to the problem setup in Section~\ref{subsec: robust model selection}. But, in the finite sample regime, instead of the population quantities, we only have access to samples. We will use these to create natural estimates $\widehat{\rho}_{ij}$, for all $i,j \in \vobs{}$ of the correlation coefficients given by $\widehat{\rho}_{ij} \triangleq \frac{\widehat{\Sigma^o}_{ij}}{\sqrt{\widehat{\Sigma^o}_{ii} \widehat{\Sigma^o}_{jj}} }, \mbox{ where } \widehat{\Sigma^o}_{ij}= \frac{1}{n} \sum_{k=1}^n y_{i}^{(k)} y_j^{(k)}.$ Indeed, these are random quantities and therefore we need to make slight modifications to the algorithm as follows: 

\noindent {\bf Change in the TIA test.} We start with the following assumption: For any triplet pair $U,W \in {\vtrue \choose 3} \setminus \mathcal{V}_{\rm star} \cup \mathcal{V}_{\rm sep}$ and any vertex pair $(x,a) \in U \times W$, there exists a constant $\zeta>0$, such that $\left \vert d_x^U + d_a^W -d_{xa}\right \vert > \zeta$. As we showed in Lemma~\ref{lemma:two_dstar_triplets_with_identical_ancestor}, for any pair $U,W \in {\vtrue \choose 3} \setminus \mathcal{V}_{\rm star} \cup \mathcal{V}_{\rm sep}$, there exists at least one triplet $\{x,a,b\}$ where $x \in U \mbox{ and } a,b \in W$ such that $d_{xa}-d_x^U - d_a^W \neq 0$ and $d_{xb}-d_x^U - d_b^W \neq 0$. Hence, the observation in Lemma~\ref{lemma:two_dstar_triplets_with_identical_ancestor} motivates us to replace the exact equality testing in the \textsc{TIA} test in Definition~\ref{def:TIA-test} with the following hypothesis test against zero: $\max\left\{\left|\widehat{d}_{xa}-\widehat{d}_x^U - \widehat{d}_a^W \right|,  \left|\widehat{d}_{xb}-\widehat{d}_x^U - \widehat{d}_b^W \right|\right\}\leq \xi$, for some $\xi < \frac{\zeta}{2}$.

\noindent {\bf Change in the Mode test.} In order to compute the distance between the hidden ancestors in the finite sample regime, we first recall from (the proof of) Lemma~\ref{lemma:learning_pairwise_dist_betn_landmarks} that 
 there are at least $4$ instances (w.l.o.g.) $\Delta(x_1,y_1), \Delta(x_1,y_2), \Delta(x_2,y_1), \mbox{ and } \Delta (x_2,y_2)$ where $\Delta(x,y)$ where $x\in U_i$ and $y \in U_j$ such that equals to $d_{ij}$. We also showed that no set of identical but incorrect distance has cardnilaity more than two. Hence, In the finite sample regime, we replace the ${\rm mode}$ test in Subroutine~\ref{app:algo--Ident_Anc_Extnd_Dist} with a more robust version, which we call the $\epsilon_d-{\rm mode}$ test, where $\epsilon_d < {\rm min} (\frac{\xi}{14}, \gamma) $ based on the following definition.

\begin{definition}[$\epsilon_d-{\rm mode}$]
\label{def:eps-mod-test}
Given a set of real numbers $\{r_1, \ldots, r_n\}$, let $S_1, \ldots, S_k$ be a partition where each $r, r' \in S_i$ is such that $\left| r-r' \right| < \epsilon_d$ for each $i$. Then, the $\epsilon_d$-mode of the this set is defined as selecting an arbitrary number from the partition with the largest cardinality.
\end{definition}
In the finite sample regime, we run \mainalgo{} with the mode replaced by the $\epsilon_d$-mode defined above such that $\epsilon_d < {\rm min} (\frac{\xi}{14}, \gamma)$. We will call this modified mode test as the $\epsilon_d$-mode test.

\noindent {\bf Change in Separation test.} For any triplet $(i,j,k) \in {\vobs{} \choose 3}$, in order to check whether $i \indep j \vert k$, instead of the equality test in Fact~\ref{fact:sepration and pairwise correlation factorization (faithful_graph)}, we modified the test for the finite sample regime as follows: $\vert \widehat{d}_{ij}- \widehat{d}_{ik}-\widehat{d}_{jk} \vert < \frac{\epsilon_d}{6}$. We now introduce two new notations to state our main result. Let $\rho_{\rm min} (p) = {\rm min}_{i,j \in {p \choose 2}}|\rho_{ij}|$ and $\kappa (p)=\log((16+\left(\rho_\mathrm{min}(p)\right)^2\epsilon_d^2)/(16-\left(\rho_\mathrm{min}(p)\right)^2\epsilon_d^2))$, where $\epsilon_d= {\rm min} (\frac{\xi}{14}, \gamma)$, where $\gamma$ is from Assumption~\ref{assump:strong_faith_assump}.

\setcounter{theorem}{2}

\begin{theorem}Suppose the underlying graph $\gtrue{} $ of a faithful GGM satisfies Assumptions \ref{assump:strong_faith_assump}-\ref{assump:strong_anc_consistency}. Fix any $\tau \in (0,1]$. Then, there exists a constant $C>0$ such that if the number of samples $n$ satisfies $n > C \left(\frac{1}{\kappa (p)}\right) {\rm max} \left( {\rm log} \left(\frac{p^2}{\tau}\right), {\rm log } \left(\frac{1}{\kappa (p)}\right)\right)$, then with probability at least $1- \tau$, \mainalgo{} accepting $\widehat{d}_{ij}$ outputs the equivalence class $[G]$.
\end{theorem}

\newcommand{\OpEquivClass}{[\mathcal{T}_{\rm algo}]}

\begin{proof} 
   First, there are (at most) seven pairwise distances to be estimated in terms of $\max\left\{\left|\widehat{d}_{xa}-\widehat{d}_x^U - \widehat{d}_a^W \right|,  \left|\widehat{d}_{xb}-\widehat{d}_x^U - \widehat{d}_b^W \right|\right\}$. Therefore, the probability that our algorithm fails is bounded above by the probability that there exists a pairwise distance estimate that is $\xi/14$ away from its mean. To this end, let us denote a bad event $B_{i,j}$ for any pair $i, j \in \vobs{}$ as the following: 
{\small \begin{align}
    B_{i,j} \triangleq \{ \vert d_{ij} - \widehat{d}_{ij}\vert \geq \epsilon_d \}.
\end{align}}
Then, the error probability $\mathbb{P}[ \OpEquivClass \neq [G]  ]$ is upper bounded as

{\small \begin{align}
    \mathbb{P} \left( \OpEquivClass \neq [G] \right) & \leq \mathbb{P} \left(\bigcup_{i,j \in \vobs{}} B_{i,j}\right) \leq \underset{i,j \in \vobs{}}{\sum} \mathbb{P} \left( B_{i,j}\right),
\end{align}}

\noindent where $\OpEquivClass$ is the output equivalence class. We now consider two following events: $K_{i,j} \triangleq \{\left| \widehat{\rho}_{ij}\right| \leq \frac{\rho_{\rm min}}{2}\}$ \footnote{for notational clarity we write $\rho_{\rm min}$ instead of $\rho_{\rm min}(p)$ }, and $R_{i,j} \triangleq\{ \left| \rho_{ij} - \widehat{\rho}_{ij}\right| < \frac{\rho_{\rm min} \epsilon_d}{2}\}$. We will upper bound $\mathbb{P}(B_{i,j})$ for any pair $i,j$ using $\mathbb{P}(K_{i,j})$ and $\mathbb{P}(R_{i,j})$. Before that, notice the following chain of implications:

\noindent $\left(\left| \rho_{ij} - \widehat{\rho}_{ij}\right| < \frac{\rho_{\rm min} \times \epsilon_d}{2} \right) \Rightarrow  \left(\left|\left| \rho_{ij} \right| - \left|\widehat{\rho}_{ij}\right| \right| < \frac{\rho_{\rm min} \times \epsilon_d}{2} \right)  \Rightarrow  \left(\left| d_{ij}- \widehat{d}_{ij} \right| < \frac{\left| \left| \rho_{ij} \right| - \left|\widehat{\rho}_{ij}\right| \right|}{{\rm min} \left( \left | \widehat{\rho}_{ij}\right|, \left | \rho_{ij}\right| \right)} \right) \Rightarrow \\
 \left( \left| d_{ij}- \widehat{d}_{ij} \right| < \frac{ \left| \left| \rho_{ij} \right| - \left|\widehat{\rho}_{ij}\right| \right| }{{\rm min} \left( \frac{\rho_{\rm min}}{2}, \rho_{\rm min} \right)}\right) \Rightarrow  \left( \left| d_{ij}- \widehat{d}_{ij} \right| < \frac{ \frac{\rho_{\rm min}}{2} \times \epsilon_d  }{\frac{\rho_{\rm min}}{2}}\right) \Rightarrow  \left( \left| d_{ij} - \widehat{d}_{ij}\right| < \epsilon_d\right)$. These implications establish that $R_{i,j} \cap K_{i,j}^c  \subseteq B^c_{i,j}$. Notice that as $R_{i,j} \cap K_{i,j}^c \subseteq B^c_{i,j} \cap K_{i,j}^c$, it will imply that $\mathbb{P}(B^c_{i,j} \cap K_{i,j}^c) \geq \mathbb{P}(R_{i,j} \cap K_{i,j}^c)$. Now, we can write the following bound:
 
 {\small\begin{align}
\label{eq:conditonal-bound} \mathbb{P}(B_{i,j} \vert K_{i,j}^c) \leq \mathbb{P}(R^c_{i,j} \vert K_{i,j}^c). 
\end{align}}

\noindent Then, $\mathbb{P}(B_{ij})$ can be upper bounded as follows:

 {\small \begin{align}
    \label{eq:err-tot-prob}\mathbb{P} \left( B_{i,j}\right) &= \mathbb{P} \left(  B_{i,j} \vert K_{i,j} \right) \mathbb{P}\left( K_{i,j} \right) + \mathbb{P} \left( B_{i,j} \vert K^c_{i,j} \right) \mathbb{P}\left( K^c_{i,j} \right), \\
   \label{eq:err-up-bound} & \leq  \mathbb{P} \left(  B_{i,j} \vert K_{i,j} \right) \mathbb{P}\left( K_{i,j} \right) + \mathbb{P} \left( R^c_{i,j} \vert K^c_{i,j} \right) \mathbb{P}\left( K^c_{i,j} \right),\\
 & \leq \left( 1 \times \mathbb{P}\left( K_{i,j} \right)\right) + \left( \mathbb{P} \left( R^c_{i,j} \vert K^c_{i,j} \right) \times 1\right). 
\end{align}}

\noindent Then, $\mathbb{P} \left( \OpEquivClass \neq [G] \right)$ can be further bounded as

{\small \begin{equation*}
    \mathbb{P} \left( \OpEquivClass \neq [G] \right) \leq \underset{i,j \in \vobs{}}{\sum} \mathbb{P} \left( B_{i,j}\right) \leq \underset{i,j \in \vobs{}}{\sum} \mathbb{P}\left( K_{i,j} \right) + \underset{i,j \in \vobs{}}{\sum} \mathbb{P} \left( R^c_{i,j} \vert K^c_{i,j} \right).
\end{equation*}}

\noindent Because $\mathbb{P}\left(R^c_{i,j} \vert K_{i,j}^c\right) < {\mathbb{P}\left(R^c_{i,j}\right)}/{\mathbb{P}\left(K^c_{i,j}\right)}$, we note that
{\small \begin{align*}
    \mathbb{P} \left( \OpEquivClass \neq [G] \right) \leq \underset{i,j \in \vobs{}}{\sum} \mathbb{P}\left( K_{i,j} \right) + \underset{i,j \in \vobs{}}{\sum} \frac{\mathbb{P}\left(R^c_{i,j}\right)}{\mathbb{P}\left(K^c_{i,j}\right)}.
\end{align*}}

\noindent We now find the required number of samples $n$ in order for $ \mathbb{P} \left( \OpEquivClass \neq [G] \right)$ to be bounded by $\tau$. Before computing $n$ we note an important inequality from \cite{kalisch2007estimating} which we use in bounding all the following events. For any $0< \epsilon \leq 2$, and ${\rm sup}_{i \neq j} \left| \rho_{ij} \right| \leq M<1$, following is true.
{\small \begin{align}
\label{eq:pairwise_correl_bound}
\mathbb{P}\left(\left \vert \widehat{\rho}_{ij} - \rho_{ij} \right \vert >\epsilon \right) \leq  C_{\rho} \left(n-2\right)\textrm{exp}\left(-\left(n-4\right)\textrm{log} \left(\frac{4 + \epsilon^2}{4 - \epsilon^2}\right)\right),
\end{align}}
for some constant $0 < C_{\rho} < \infty$ depending on $M$ only.

\noindent We now note the following assumption on bounded correlation which is a common assumption in learning the graphical models: $0 < \rho_{\rm min} \leq \rho_{\rm max} < 1$. Now notice that, $\left( \left| \widehat{\rho}_{ij}\right| \leq \frac{\rho_{\rm min}}{2}\right)$ together with $\left|\rho_{ij}\right| \geq \rho_{\rm min}$ implies that $\left|\rho_{ij} \right| - \left|\widehat{\rho}_{ij}\right| \geq \rho_{\rm min} - \frac{\rho_{\rm min}}{2} = \frac{\rho_{\rm min}}{2} $, since $\rho_{\rm min } > \frac{\rho_{\rm min}}{2}$. Furthermore, $\left| \rho_{ij} - \widehat{\rho}_{ij} \right| \geq \left|\rho_{ij} \right| - \left|\widehat{\rho}_{ij}\right|$ implies that $\left| \rho_{ij} - \widehat{\rho}_{ij} \right| \geq \frac{\rho_{\rm min}}{2}$. Then, we have the following:
{\small \begin{align}
\label{eq:bound-on-bad-corrl-event}
    \mathbb{P} \left(K_{i,j}\right) & \leq \mathbb{P}\left(\left| \rho_{ij} - \widehat{\rho}_{ij} \right| \geq \frac{\rho_{\rm min}}{2}\right)  \leq  C_{\rho} \left(n-2\right)\textrm{exp}\left(- \left(n-4 \right)\textrm{log} \left(\frac{16 + \rho^2_{\rm min}}{16-\rho^2_{\rm min}}\right)\right).
\end{align}}
\noindent Eq. \eqref{eq:bound-on-bad-corrl-event} follows from Eq. \eqref{eq:pairwise_correl_bound}. Now, According to Claim \ref{claim:solving-n-in-exponential-eq},

{\small \begin{equation}
    \begin{split}
        n_1 > &{\rm max} \Bigg(C_1\frac{{\rm log} \left(\frac{2C_{\rho} {p \choose 2} }{\tau}\right)}{{\rm log} \left(\frac{16+\rho_{\rm min}^2}{16- \rho_{\rm min}^2}\right)} \times \frac{C_2C_1}{\left(C_1-1\right){\rm log} \left(\frac{16+\rho_{\rm min}^2}{16-\rho_{\rm min}^2}\right)},
        {\rm log \left(\frac{C_1}{\left(C_1-1\right){\rm log} \left(\frac{16+\rho_{\rm min}^2}{16-\rho_{\rm min}^2}\right)}\right)}\Bigg)+4
    \end{split}
\end{equation}}

\noindent implies $\underset{i,j \in \vobs{}}{\sum} \mathbb{P}(K_{i,j}) < \frac{\tau}{2}$, 

{\small \begin{equation}
    \begin{split}
           n_3 > &{\rm max} \Bigg(C_1\frac{{\rm log} \left(\frac{C_{\rho}}{1-\tau'}\right)}{{\rm log} \left(\frac{16+\rho_{\rm min}^2}{16- \rho_{\rm min}^2}\right)}, \frac{C_2C_1}{\left(C_1-1\right) {\rm log} \left(\frac{16+\rho_{\rm min}^2}{16-\rho_{\rm min}^2}\right)} \times  {\rm log \left(\frac{C_1}{\left(C_1-1\right){\rm log} \left(\frac{16+\rho_{\rm min}^2}{16-\rho_{\rm min}^2}\right)}\right)}\Bigg)+4
    \end{split}
\end{equation}}

\noindent  implies $\mathbb{P}(K_{i,j}^c) > \tau'$, where $\tau'> 1-C_{\rho}$, and 
 
{\small \begin{equation}
    \begin{split}
        n_4 >  {\rm max}\Bigg(C_1 \frac{{\rm log} \left(\frac{2C_{\rho} {p \choose 2}}{\tau \tau'}\right)}{{\rm log \left(\frac{16+ \rho^2_{\rm min}\epsilon_d^2}{16-\rho^2_{\rm min}\epsilon_d^2}\right)}},  \frac{C_2C_1}{\left(C_1-1\right)  {\rm log} \left(\frac{16+\rho_{\rm min}^2\epsilon_d^2}{16-\rho_{\rm min}^2\epsilon_d^2}\right)} \times 
      {\rm log \left(\frac{C_1}{\left(C_1-1\right){\rm log} \left(\frac{16+\rho_{\rm min}^2\epsilon_d^2}{16-\rho_{\rm min}^2\epsilon_d^2}\right)}\right)} \Bigg)+4
    \end{split}
\end{equation}}

\noindent implies $\mathbb{P}(R_{i,j}^c) < \frac{\tau \tau'}{2{p \choose 2}}$. Now, notice that $n_2 \triangleq {\rm max}\left(n_3,n_4\right)$ implies $\frac{\mathbb{P}(R_{i,j}^c)}{\mathbb{P}(K_{i,j}^c)} < \frac{\tau}{2 {p \choose 2}}$. Therefore, acquiring at least $n_2$ samples will imply $\underset{i,j \in \vobs{}}{\sum}\frac{\mathbb{P}(R_{i,j}^c)}{\mathbb{P}(K_{i,j}^c)} < \frac{\tau}{2}$. Finally, for $\mathbb{P}(\OpEquivClass \neq [G])$ to be upper bounded by $\tau$, it is sufficient for the number of samples $n$ to satisfy $n > {\rm max } \left( n_1,n_2\right)$.
 \end{proof}

\begin{claim}
\label{claim:solving-n-in-exponential-eq}
There exist positive constants $T,C$, and $\widetilde{\alpha}$ such that if $n > {\rm max} (T, C \times \widetilde{\alpha} \:{\rm log} \: \widetilde{\alpha})$, then $n -\widetilde{\alpha} {\rm log} \left(n\right) > T$. 

\end{claim}

\begin{proof}
   We start the proof with the following claim: Suppose that there exists a constant $C_1,C_2$ where $C_1<C_2$ such that $C_1 m {\rm log} m < n < C_2 m {\rm log} m$. Notice that for $m$ sufficiently large ($m>C_2$), we can show that $n > m {\rm log } n $. Therefore, for some constant $C_1,C_2$, $n > C_2 \times \frac{C_1}{(C_1-1)\alpha} {\rm log} \left(\frac{C_1}{\left(C_1-1\right)\alpha}\right)$ implies $n> \frac{C_1}{(C_1-1)\alpha} {\rm log} \left(n\right)$. Now, suppose that ${\rm max} \left( C_1T, \frac{C_2C_1}{(C_1-\alpha)} {\rm log} \left(\frac{C_1}{\left(C_1-1\right)\alpha}\right)\right) = C_1T$. Then, $n > C_1T$ implies $n > C_2 \times \frac{C_1}{(C_1-1)\alpha} {\rm log} \left(\frac{C_1}{\left(C_1-1\right)\alpha}\right) $. Then, from the initial claim we have that $n> \frac{C_1}{(C_1-1)\alpha} {\rm log} \left(n\right)$. Then, $n \frac{(C_1-1)} {C_1}> \frac{1}{\alpha} {\rm log} \left(n\right)$, and $n - \frac{1}{\alpha} {\rm log }(n) > \frac{n}{C_1}$. As $\frac{n}{C_1}>T$, we have that $n - \frac{1}{\alpha} {\rm log }(n) > T$. Further, suppose that ${\rm max} \left( C_1T, \frac{C_2C_1}{(C_1-\alpha)} {\rm log} \left(\frac{C_1}{\left(C_1-1\right)\alpha}\right)\right) = \frac{C_2C_1}{(C_1-\alpha)} {\rm log} \left(\frac{C_1}{\left(C_1-1\right)\alpha}\right)$. Then, from the initial claim we have that $ n > \frac{C_2C_1}{(C_1-\alpha)} {\rm log} \left(\frac{C_1}{\left(C_1-1\right)\alpha}\right)$ implies $n> \frac{C_1}{(C_1-1)\alpha} {\rm log} \left(n\right)$. Also,  $ n > \frac{C_2C_1}{(C_1-\alpha)} {\rm log} \left(\frac{C_1}{\left(C_1-1\right)\alpha}\right)$ implies  $n >C_1T$, which will imply $n - \frac{1}{\alpha} {\rm log }(n) > \frac{n}{C_1}>T$. Setting $\widetilde{\alpha}$ equals to $\frac{C_1}{(C_1-1) \alpha}$ proves the result. 
\end{proof} 
 

%% file: Appendix-Unidentifiablity.tex
\newcommand{\bigzero}{\mbox{\normalfont\Large\bfseries 0}}
\newcommand{\bigI}{\mbox{\normalfont\Large I}}
\section {Identifiability Result}
\label{sec:identifiblty--proof}

\begin{proof}
\newcommand{\transpose}{\mathsf{T}}
We first consider the case where there is only one non-trivial block $\mathcal{B}^{NT}$ inside $G$ and that the block cut vertices of $\mathcal{B}^{NT}$ do not have neighboring leaf nodes. 
As a result, $\mathcal{B}^{NT}$ contains exactly two block cut vertices $b_1$ and $b_2$ connected to the cut vertices $p_1$ and $p_2$, respectively. Thus, we express the vertex set $V$ of $G$ as a union of disjoint sets $V_1\cup \{p_1\}$, $V_2\cup \{p_2\}$, and $V_{NT}$---the vertex set of $\mathcal{B}^{NT}$. 

\noindent Without loss of generality, let $V_1\cup \{p_1\}=\{1,\ldots,p_1\}$, $V_{NT}=\{p_1+1,\ldots,p_2-1\}$, and $V_2\cup \{p_2\}=\{p_2,\ldots,p\}$. Also, let $b_1=p_1+1$ and $b_2=p_2-1$. Because $G $, it follows that $V_1\cup \{p_1\} \indep V_1\cup \{p_2\} \mid V_{NT}$. In words, $V_{NT}$ separates $V_1\cup \{p_1\}$ and $V_2\cup \{p_2\}$. Furthermore, $b_1$ shares an edge with $p_1$ and $b_2$ shares an edge with $p_2$. From these facts, $K^*=(\Sigma^*)^{-1}$ can be partitioned as in \eqref{eq: identifiability_proof_inv_covariance} (see below). 
 \begin{figure*}[b]
\begin{align}\label{eq: identifiability_proof_inv_covariance} 
    K^*=\left[\begin{array}{ccc|ccc|ccc}
         K_{11} & \ldots & K_{1,p_1} & 0 & \ldots & 0 &\\
         \vdots & \ddots & \vdots & \vdots & \ddots & \vdots &\\
         K_{p_1,1} & \ldots & K_{p_1,p_1} & K_{p_1+1,p_1}& \dots & 0 &\\
         \hline 
         0 & \ldots & K_{p_1,p_1+1} & K_{p_1+1,p_1+1} & \ldots & K_{p_1+1,p_2-1} & 0 & \ldots  & 0\\
         \vdots &\ddots & \vdots & \vdots & \ddots & \vdots & \vdots & \ddots & \vdots\\
         0 & \ldots & 0 & K_{p_2-1,p_1+1} & \ldots & K_{p_2-1,p_2-1} & K_{p_2-1,p_2} & \ldots & 0\\
         \hline 
         & & & 0 & \ldots & K_{p_2,p_2-1} & K_{p_2,p_2} & \ldots & K_{p_2,p}\\
         & & & \vdots & \ddots & \vdots & \vdots & \ddots & \vdots \\
         & & & 0 & \ldots & 0 & K_{p,p_2} & \ldots & K_{p,p} 
    \end{array}
    \right]
\end{align}
\end{figure*}Let $K_1$, $K_{NT}$, and $K_2$ be the first, second, and third diagonal blocks of $K^*$ in \eqref{eq: identifiability_proof_inv_covariance}. Let $e_j$ be the canonical basis vector in $\mathbb{R}^p$. Then, we can express $K^*$ in \eqref{eq: identifiability_proof_inv_covariance} as
{\small \begin{equation}
    \begin{split}
        \label{eq: identifiability_proof_Kstar}
    K^*&=\mathrm{Blkdiag}(K_1,K_{NT},K_2) +e_{p_1+1}e_{p_1}^\transpose  +e_{p_1}e_{p_1+1}^\transpose +e_{p_2-1}e_{p_2}^\transpose+e_{p_2}e_{p_2-1}^\transpose.
    \end{split}
\end{equation}}

Recall that $\Sigma^{0}=\Sigma^{*}+D$. Decompose the diagonal matrix $D$ as $D=D^{(1)}+D^{(2)}$, where 
{\small \begin{align}
D^{(1)}&=\mathrm{Blkdiag}(\mathbf{0},D_{NT}^{(1)},\mathbf{0}), \label{eq: identifiability_proof_D_super1}\\
D^{(2)}&=\mathrm{Blkdiag}(D_1,D_{NT}^{(2)},D_2),\label{eq: identifiability_proof_D_super2}
\end{align}}

\noindent and the dimensions of $D_1$, $D_{NT}$, and $D_2$ are same as those of $K_1$, $K_{NT}$, and $K_2$, resp. Furthermore, $D_{NT}^{(1)}\!=\!\text{diag}(0,\times, \ldots, \times, 0)$ and $D_{NT}^{(2)}\!=\!\text{diag}(\times,0,\ldots,0,\times)$. Here $\times$ can be a zero or a positive value.  Let $\Sigma^{q}=\Sigma^*+D^{(1)}$ and $D^q=D^{(2)}$. From the above notations, we have $\Sigma^{0}=\Sigma^{*}+D=\Sigma^{*}+D^{(1)}+D^{(2)}=\Sigma^{q}+D^q$. We show that there exists at least one $H\ne G \in [G]$ that encodes the conditional independence structure of $(\Sigma^{q})^{-1}$. It suffices to show that $(\Sigma^{q})^{-1}$ exactly equals the expression of $K^*$ in \eqref{eq: identifiability_proof_Kstar}, except for the second diagonal block $K_{NT}$ in $\mathrm{Blkdiag}(K_1,K_{NT},K_2)$. Recall that different values of $K_{NT}$ yield different graphs in $[G]$; see Definition \ref{def:art-set-graph}. Consider the following identity:
{\small\begin{equation}
\label{eq: identifiability_proof_Kq}
    \begin{split}
        (\Sigma^{q})^{-1}=  (\Sigma^*+D^{(1)})^{-1} =(I+(\Sigma^*)^{-1}D^{(1)})^{-1}(\Sigma^*)^{-1} 
        =(I+K^*D^{(1)})^{-1}K^*.
    \end{split}
\end{equation}}
\noindent We first evaluate $(I+K^*D^{(1)})^{-1}$. Note that $e_{p_1+1}$, $e_{p_1}$, 
    $e_{p_2-1}$, and $e_{p_2}$ lie in the nullspace of $D^{(1)}$ and $K^*D^{(1)}$. Using this fact and the formulas in \eqref{eq: identifiability_proof_Kstar} and \eqref{eq: identifiability_proof_D_super1}, we can simplify $(I+K^*D^{(1)})$ as 
\begin{align}
    (I+K^*D^{(1)})=\mathrm{Blkdiag}(I,I+K_{NT}D_{NT}^{(1)},I), 
\end{align}
where, $\widetilde{K}_{NT}\triangleq I+K_{NT}D_{NT}^{(1)}$ is a positive definite matrix,  and hence, invertible. This is because $K_{NT}D_{NT}^{(1)}$ and $(D_{NT}^{(1)})^{1/2}K_{NT}^{1/2}K_{NT}^{1/2}(D_{NT}^{(1)})^{1/2}$ are similar matrices, where we used the facts that $K_{NT}$ is positive definite and $D_{NT}^{(1)}$ is non-negative diagonal. Thus, 
\begin{align}\label{eq: identifiability_proof_KstarD_plus_eye}
    (I+K^*D^{(1)})^{-1}=\mathrm{Blkdiag}(I_{p_1},\widetilde{K}_{NT}^{-1},I_{p-p_2+1}). 
\end{align}
Also, note that the null space vectors $e_{p_1+1}$, $e_{p_1}$, 
    $e_{p_2-1}$, and $e_{p_2}$ of $K^*D^{(1)}$ are also the eigenvectors of $(I+K^*D^{(1)})^{-1}$, with eigenvalues all being equal to one. Putting together the pieces, from \eqref{eq: identifiability_proof_Kstar}, \eqref{eq: identifiability_proof_Kq}, and \eqref{eq: identifiability_proof_KstarD_plus_eye} we have 
{\small \begin{align}
\label{eq:identifiability_proof_Kqfinal}
    (\Sigma^{q})^{-1}&=(I+K^*D^{(1)})^{-1}K^*\nonumber 
    =\mathrm{Blkdiag}(K_1,\widetilde{K}_{NT}^{-1}K_{NT},K_2)+e_{p_1+1}e_{p_1}^\transpose \nonumber 
    +e_{p_1}e_{p_1+1}^\transpose+e_{p_2-1}e_{p_2}^\transpose+e_{p_2}e_{p2-1}^\transpose. 
\end{align} } 

\noindent Moreover, $\widetilde{K}_{NT}^{-1}K_{NT}=(I+K_{NT}D_{NT}^{(1)})^{-1}K_{NT}=(\Sigma_{NT}+D_{NT}^{(1)})^{-1}$, where $\Sigma_{NT}=K_{NT}^{-1}$ is the covariance of the random vector associated with $\mathcal{B}^{NT}$. Thus, $K^*$ in \eqref{eq: identifiability_proof_Kstar} and $(\Sigma^{q})^{-1}$ are identical, except in their second diagonal blocks, as required. Using similar arguments, we can handle multiple internal blocks with block cut vertices that are not adjacent to leaf nodes. In the case where blocks have leaf nodes, we can combine the construction above with the construction in \cite[Theorem 1]{katiyar2019robust} for tree structured graphical models. Combining these two, we can show that in this general case there exists a graph $H\ne G \in [G]$ such that (a) the structure is arbitrarily different inside blocks, and (b) the block cut vertices are preserved (i.e., same as the ones in $G$), except they may be swapped with a neighboring leaf.

\end{proof}

%% file: Main.bbl
\begin{thebibliography}{55}
\providecommand{\natexlab}[1]{#1}
\providecommand{\url}[1]{\texttt{#1}}
\expandafter\ifx\csname urlstyle\endcsname\relax
  \providecommand{\doi}[1]{doi: #1}\else
  \providecommand{\doi}{doi: \begingroup \urlstyle{rm}\Url}\fi

\bibitem[Anandkumar and Valluvan(2013)]{anandkumar2013learning}
A.~Anandkumar and R.~Valluvan.
\newblock Learning loopy graphical models with latent variables: Efficient methods and guarantees.
\newblock \emph{The Annals of Statistics}, pages 401--435, 2013.

\bibitem[Anguluri et~al.(2022)Anguluri, Dasarathy, Kosut, and Sankar]{Angu2022}
R.~Anguluri, G.~Dasarathy, O.~Kosut, and L.~Sankar.
\newblock Grid topology identification with hidden nodes via structured norm minimization.
\newblock \emph{IEEE Control Systems Letters}, 6:\penalty0 1244--1249, 2022.

\bibitem[Baran and Wu(1989)]{baran1989network}
M.~E. Baran and F.~F. Wu.
\newblock Network reconfiguration in distribution systems for loss reduction and load balancing.
\newblock \emph{IEEE Transactions on Power delivery}, 4\penalty0 (2):\penalty0 1401--1407, 1989.

\bibitem[Biggs et~al.(1986)Biggs, Lloyd, and Wilson]{biggs1986graph}
N.~Biggs, E.~K. Lloyd, and R.~J. Wilson.
\newblock \emph{Graph Theory}.
\newblock Oxford University Press, 1986.

\bibitem[Bullmore and Bassett(2011)]{bullmore2011brain}
E.~T. Bullmore and D.~S. Bassett.
\newblock Brain graphs: graphical models of the human brain connectome.
\newblock \emph{Annual review of clinical psychology}, 7:\penalty0 113--140, 2011.

\bibitem[Carroll et~al.(1995)Carroll, Ruppert, and Stefanski]{carroll1995measurement}
R.~J. Carroll, D.~Ruppert, and L.~A. Stefanski.
\newblock \emph{Measurement error in nonlinear models}, volume 105.
\newblock CRC press, 1995.

\bibitem[Casanellas et~al.(2021)Casanellas, Garrote-L{\'o}pez, and Zwiernik]{casanellas2021robust}
M.~Casanellas, M.~Garrote-L{\'o}pez, and P.~Zwiernik.
\newblock Robust estimation of tree structured models.
\newblock \emph{arXiv preprint arXiv:2102.05472}, 2021.

\bibitem[Chang et~al.(2023)Chang, Zheng, Dasarathy, and Allen]{chang2023nonparanormal}
A.~Chang, L.~Zheng, G.~Dasarathy, and G.~I. Allen.
\newblock Nonparanormal graph quilting with applications to calcium imaging.
\newblock \emph{Stat}, 12\penalty0 (1):\penalty0 e623, 2023.

\bibitem[Chen et~al.(2015)Chen, Gao, and Ren]{chen2015robust}
M.~Chen, C.~Gao, and Z.~Ren.
\newblock Robust covariance matrix estimation via matrix depth.
\newblock \emph{arXiv preprint arXiv:1506.00691}, 2015.

\bibitem[Chen et~al.(2013)Chen, Caramanis, and Mannor]{chen2013robust}
Y.~Chen, C.~Caramanis, and S.~Mannor.
\newblock Robust sparse regression under adversarial corruption.
\newblock In \emph{International Conference on Machine Learning}, pages 774--782. PMLR, 2013.

\bibitem[Choi et~al.(2011)Choi, Tan, Anandkumar, and Willsky]{choi2011learning}
M.~J. Choi, V.~Y. Tan, A.~Anandkumar, and A.~S. Willsky.
\newblock Learning latent tree graphical models.
\newblock \emph{J. of Machine Learning Research}, 12:\penalty0 1771--1812, 2011.

\bibitem[Dasarathy(2019)]{dasarathy2019gaussian}
G.~Dasarathy.
\newblock Gaussian graphical model selection from size constrained measurements.
\newblock In \emph{2019 IEEE International Symposium on Information Theory (ISIT)}, pages 1302--1306. IEEE, 2019.

\bibitem[Dasarathy et~al.(2014)Dasarathy, Nowak, and Roch]{dasarathy2014data}
G.~Dasarathy, R.~Nowak, and S.~Roch.
\newblock Data requirement for phylogenetic inference from multiple loci: a new distance method.
\newblock \emph{IEEE/ACM transactions on computational biology and bioinformatics}, 12\penalty0 (2):\penalty0 422--432, 2014.

\bibitem[Dasarathy et~al.(2022)Dasarathy, Mossel, Nowak, and Roch]{dasarathy2022stochastic}
G.~Dasarathy, E.~Mossel, R.~Nowak, and S.~Roch.
\newblock A stochastic farris transform for genetic data under the multispecies coalescent with applications to data requirements.
\newblock \emph{Journal of Mathematical Biology}, 84\penalty0 (5):\penalty0 1--37, 2022.

\bibitem[Deka et~al.(2015)Deka, Baldick, and Vishwanath]{deka2015one}
D.~Deka, R.~Baldick, and S.~Vishwanath.
\newblock One breaker is enough: Hidden topology attacks on power grids.
\newblock In \emph{2015 IEEE Power \& Energy Society General Meeting}, pages 1--5. IEEE, 2015.

\bibitem[Deka et~al.(2020)Deka, Talukdar, Chertkov, and Salapaka]{deka2020graphical}
D.~Deka, S.~Talukdar, M.~Chertkov, and M.~V. Salapaka.
\newblock Graphical models in meshed distribution grids: Topology estimation, change detection \& limitations.
\newblock \emph{IEEE Transactions on Smart Grid}, 11\penalty0 (5):\penalty0 4299--4310, 2020.

\bibitem[Drton and Maathuis(2017)]{drton2017structure}
M.~Drton and M.~H. Maathuis.
\newblock Structure learning in graphical modeling.
\newblock \emph{Annual Review of Statistics and Its Application}, 4:\penalty0 365--393, 2017.

\bibitem[Erd{\"o}s et~al.(1999)Erd{\"o}s, Steel, Sz{\'e}kely, and Warnow]{erdos1999few}
P.~L. Erd{\"o}s, M.~A. Steel, L.~Sz{\'e}kely, and T.~J. Warnow.
\newblock A few logs suffice to build (almost) all trees: Part ii.
\newblock \emph{Theoretical Computer Science}, 221\penalty0 (1-2):\penalty0 77--118, 1999.

\bibitem[Friedman et~al.(2008)Friedman, Hastie, and Tibshirani]{friedman2008sparse}
J.~Friedman, T.~Hastie, and R.~Tibshirani.
\newblock Sparse inverse covariance estimation with the graphical lasso.
\newblock \emph{Biostatistics}, 9\penalty0 (3):\penalty0 432--441, 2008.

\bibitem[Harary(1971)]{harary2018graph}
F.~Harary.
\newblock \emph{Graph Theory}.
\newblock Addison Wesley series in mathematics. Addison-Wesley, 1971.

\bibitem[Herbert et~al.(1962)]{herbert1962architecture}
S.~Herbert et~al.
\newblock The architecture of complexity.
\newblock \emph{Proceedings of the American Philosophical Society}, 106\penalty0 (6):\penalty0 467--482, 1962.

\bibitem[Horn and Johnson(2012)]{horn2012matrix}
R.~A. Horn and C.~R. Johnson.
\newblock \emph{Matrix analysis}.
\newblock Cambridge university press, 2012.

\bibitem[Hwang(1986)]{hwang1986multiplicative}
J.~T. Hwang.
\newblock Multiplicative errors-in-variables models with applications to recent data released by the us department of energy.
\newblock \emph{Journal of the American Statistical Association}, 81\penalty0 (395):\penalty0 680--688, 1986.

\bibitem[Iturria et~al.(1999)Iturria, Carroll, and Firth]{iturria1999polynomial}
S.~J. Iturria, R.~J. Carroll, and D.~Firth.
\newblock Polynomial regression and estimating functions in the presence of multiplicative measurement error.
\newblock \emph{Journal of the Royal Statistical Society: Series B (Statistical Methodology)}, 61\penalty0 (3):\penalty0 547--561, 1999.

\bibitem[Kalisch and B{\"u}hlman(2007)]{kalisch2007estimating}
M.~Kalisch and P.~B{\"u}hlman.
\newblock Estimating high-dimensional directed acyclic graphs with the pc-algorithm.
\newblock \emph{Journal of Machine Learning Research}, 8\penalty0 (3), 2007.

\bibitem[Katiyar et~al.(2019)Katiyar, Hoffmann, and Caramanis]{katiyar2019robust}
A.~Katiyar, J.~Hoffmann, and C.~Caramanis.
\newblock Robust estimation of tree structured gaussian graphical models.
\newblock In \emph{International Conference on Machine Learning}, pages 3292--3300. PMLR, 2019.

\bibitem[Katiyar et~al.(2020)Katiyar, Shah, and Caramanis]{katiyar2020robust}
A.~Katiyar, V.~Shah, and C.~Caramanis.
\newblock Robust estimation of tree structured ising models.
\newblock \emph{arXiv preprint arXiv:2006.05601}, 2020.

\bibitem[Kim and Smaragdis(2013)]{kim2013single}
M.~Kim and P.~Smaragdis.
\newblock Single channel source separation using smooth nonnegative matrix factorization with markov random fields.
\newblock In \emph{2013 IEEE International Workshop on Machine Learning for Signal Processing (MLSP)}, pages 1--6. IEEE, 2013.

\bibitem[Krishnamurthy and Singh(2012)]{krishnamurthy2012robust}
A.~Krishnamurthy and A.~Singh.
\newblock Robust multi-source network tomography using selective probes.
\newblock In \emph{2012 Proceedings IEEE INFOCOM}, pages 1629--1637. IEEE, 2012.

\bibitem[Lauritzen(1996)]{lauritzen1996graphical}
S.~L. Lauritzen.
\newblock \emph{Graphical models}, volume~17.
\newblock Clarendon Press, 1996.

\bibitem[Loh and Tan(2018)]{loh2018high}
P.-L. Loh and X.~L. Tan.
\newblock High-dimensional robust precision matrix estimation: Cellwise corruption under $epsilon $-contamination.
\newblock \emph{Electronic Journal of Statistics}, 12\penalty0 (1):\penalty0 1429--1467, 2018.

\bibitem[Loh and Wainwright(2011)]{loh2011high}
P.-L. Loh and M.~J. Wainwright.
\newblock High-dimensional regression with noisy and missing data: Provable guarantees with non-convexity.
\newblock \emph{Advances in Neural Information Processing Systems}, 24, 2011.

\bibitem[Lounici(2014)]{lounici2014high}
K.~Lounici.
\newblock High-dimensional covariance matrix estimation with missing observations.
\newblock \emph{Bernoulli}, 20\penalty0 (3):\penalty0 1029--1058, 2014.

\bibitem[Maathuis et~al.(2018)Maathuis, Drton, Lauritzen, and Wainwright]{maathuis2018handbook}
M.~Maathuis, M.~Drton, S.~Lauritzen, and M.~Wainwright.
\newblock \emph{Handbook of graphical models}.
\newblock CRC Press, 2018.

\bibitem[Murphy et~al.(2013)Murphy, Weiss, and Jordan]{murphy2013loopy}
K.~Murphy, Y.~Weiss, and M.~I. Jordan.
\newblock Loopy belief propagation for approximate inference: An empirical study.
\newblock \emph{arXiv preprint arXiv:1301.6725}, 2013.

\bibitem[Nguyen et~al.(2022)Nguyen, Kuhn, and Mohajerin~Esfahani]{nguyen2022distributionally}
V.~A. Nguyen, D.~Kuhn, and P.~Mohajerin~Esfahani.
\newblock Distributionally robust inverse covariance estimation: The wasserstein shrinkage estimator.
\newblock \emph{Operations Research}, 70\penalty0 (1):\penalty0 490--515, 2022.

\bibitem[Nikolakakis et~al.(2019)Nikolakakis, Kalogerias, and Sarwate]{nikolakakis2019learning}
K.~E. Nikolakakis, D.~S. Kalogerias, and A.~D. Sarwate.
\newblock Learning tree structures from noisy data.
\newblock In \emph{The 22nd International Conference on Artificial Intelligence and Statistics}, pages 1771--1782. PMLR, 2019.

\bibitem[{\"O}llerer and Croux(2015)]{ollerer2015robust}
V.~{\"O}llerer and C.~Croux.
\newblock Robust high-dimensional precision matrix estimation.
\newblock In \emph{Modern nonparametric, robust and multivariate methods}, pages 325--350. Springer, 2015.

\bibitem[Ott and Stoop(2006)]{ott2006neurodynamics}
T.~Ott and R.~Stoop.
\newblock The neurodynamics of belief propagation on binary markov random fields.
\newblock \emph{Advances in neural information processing systems}, 19, 2006.

\bibitem[RJa and Rubin(1987)]{rja1987statistical}
L.~RJa and D.~Rubin.
\newblock \emph{Statistical analysis with missing data}.
\newblock 1987.

\bibitem[Saitou and Nei(1987)]{saitou1987neighbor}
N.~Saitou and M.~Nei.
\newblock The neighbor-joining method: a new method for reconstructing phylogenetic trees.
\newblock \emph{Molecular biology and evolution}, 4\penalty0 (4):\penalty0 406--425, 1987.

\bibitem[Schneider(2001)]{schneider2001analysis}
T.~Schneider.
\newblock Analysis of incomplete climate data: Estimation of mean values and covariance matrices and imputation of missing values.
\newblock \emph{Journal of climate}, 14\penalty0 (5):\penalty0 853--871, 2001.

\bibitem[Semple et~al.(2003)Semple, Steel, et~al.]{semple2003phylogenetics}
C.~Semple, M.~Steel, et~al.
\newblock \emph{Phylogenetics}, volume~24.
\newblock Oxford University Press on Demand, 2003.

\bibitem[Soh and Tatikonda(2014)]{soh2014testing}
D.~W. Soh and S.~C. Tatikonda.
\newblock Testing unfaithful gaussian graphical models.
\newblock \emph{Advances in Neural Information Processing Systems}, 27:\penalty0 2681--2689, 2014.

\bibitem[Sporns(2018)]{sporns2018graph}
O.~Sporns.
\newblock Graph theory methods: applications in brain networks.
\newblock \emph{Dialogues in clinical neuroscience}, 2018.

\bibitem[Sun and Li(2012)]{sun2012robust}
H.~Sun and H.~Li.
\newblock Robust gaussian graphical modeling via l1 penalization.
\newblock \emph{Biometrics}, 68\penalty0 (4):\penalty0 1197--1206, 2012.

\bibitem[Tandon et~al.(2021)Tandon, Yuan, and Tan]{tandon2021sga}
A.~Tandon, A.~H. Yuan, and V.~Y. Tan.
\newblock {SGA}: A robust algorithm for partial recovery of tree-structured graphical models with noisy samples.
\newblock \emph{arXiv preprint arXiv:2101.08917}, 2021.

\bibitem[Uhler et~al.(2013)Uhler, Raskutti, B{\"u}hlmann, and Yu]{uhler2013geometry}
C.~Uhler, G.~Raskutti, P.~B{\"u}hlmann, and B.~Yu.
\newblock Geometry of the faithfulness assumption in causal inference.
\newblock \emph{The Annals of Statistics}, pages 436--463, 2013.

\bibitem[Vinci et~al.(2019)Vinci, Dasarathy, and Allen]{vinci2019graph}
G.~Vinci, G.~Dasarathy, and G.~I. Allen.
\newblock Graph quilting: graphical model selection from partially observed covariances.
\newblock \emph{arXiv preprint arXiv:1912.05573}, 2019.

\bibitem[Wang et~al.(2014)]{wang2014robust}
J.-K. Wang et~al.
\newblock Robust inverse covariance estimation under noisy measurements.
\newblock In \emph{International Conference on Machine Learning}, pages 928--936. PMLR, 2014.

\bibitem[Xu and You(2007)]{xu2007covariate}
Q.~Xu and J.~You.
\newblock Covariate selection for linear errors-in-variables regression models.
\newblock \emph{Communications in Statistics—Theory and Methods}, 36\penalty0 (2):\penalty0 375--386, 2007.

\bibitem[Yang and Lozano(2015)]{yang2015robust}
E.~Yang and A.~C. Lozano.
\newblock Robust gaussian graphical modeling with the trimmed graphical lasso.
\newblock \emph{Advances in Neural Information Processing Systems}, 28, 2015.

\bibitem[Zhang and Tan(2021)]{zhang2021robustifying}
F.~Zhang and V.~Tan.
\newblock Robustifying algorithms of learning latent trees with vector variables.
\newblock \emph{Advances in Neural Information Processing Systems}, 34, 2021.

\bibitem[Zheng and Allen(2022)]{zheng2022graphical}
L.~Zheng and G.~I. Allen.
\newblock Graphical model inference with erosely measured data.
\newblock \emph{arXiv preprint arXiv:2210.11625}, 2022.

\bibitem[Zuo et~al.(2017)Zuo, Cui, Yu, Li, and Ressom]{zuo2017incorporating}
Y.~Zuo, Y.~Cui, G.~Yu, R.~Li, and H.~W. Ressom.
\newblock Incorporating prior biological knowledge for network-based differential gene expression analysis using differentially weighted graphical lasso.
\newblock \emph{BMC bioinformatics}, 18\penalty0 (1):\penalty0 1--14, 2017.

\end{thebibliography}
